\documentclass{article} 
\usepackage{iclr2027_conference,times}

\usepackage{hyperref}
\usepackage{url}
\usepackage{graphicx}
\usepackage{booktabs}
\usepackage{amsmath}
\usepackage{array} 
\usepackage{wrapfig}
\newtheorem{theorem}{Theorem}
\newtheorem{proof}{Proof}
\usepackage{enumitem}
\newtheorem{lemma}{Lemma}
\newtheorem{proposition}[theorem]{Proposition}
\usepackage{algorithm}
\usepackage{algorithmic}
\usepackage{tabularx}

\usepackage{subcaption}
\usepackage{wrapfig}
\usepackage{booktabs}       
\usepackage{amsfonts}       
\usepackage{nicefrac}       
\usepackage{microtype}      
\usepackage{xcolor}         
\usepackage{booktabs}
\usepackage{multirow}
\usepackage{graphicx}
\usepackage[most]{tcolorbox}
\def\method{GeoF} 
\title{Learning Propagation Geometry from Message-Passing Feedback}

\author{Yingxu Wang\textsuperscript{1}, Kunyu Zhang\textsuperscript{2}, Xinwang Liu\textsuperscript{3}, Mengzhu Wang\textsuperscript{4},
Siyang Gao\textsuperscript{5}, \\
\textbf{Chang Tang\textsuperscript{6}, Nan Yin\textsuperscript{2}}
\\
\textsuperscript{1} The Chinese University of Hong Kong 
\textsuperscript{2} The Education University of Hong Kong \\
\textsuperscript{3} National University of Defense Technology \;\;
\textsuperscript{4} Hebei University of Technology \\
\textsuperscript{5} City University of Hong Kong \;\;
\textsuperscript{6} Huazhong University of Science and Technology \\
\texttt{\{yingxv.wang,dreamkily,yinnan8911\}@gmail.com} \\
\texttt{kuzh330@outlook.com, xinwangliu@nudt.edu.cn} \\
\texttt{siyangao@cityu.edu.hk, tangchang@hust.edu.cn} }
\begin{document}

\iclrfinalcopy
\maketitle

\begin{abstract}
Learning local geometry enables graph neural networks (GNNs) to adapt how they compare and integrate neighborhood information. However, estimating geometry from aggregated representations can overlook variation among individual messages and dependencies across feature dimensions. We propose \method{}, a recurrent framework that jointly evolves node features and propagation geometry through message-passing feedback. Each node maintains a local symmetric positive-definite geometry, initialized from a structure-aware prototype atlas and parameterized in block log-triangular coordinates. At each step, the geometry determines neighborhood weights, while triangular frame transport maps transformed source messages into the target node's local coordinates before aggregation. Weighted second-order statistics of residuals between aligned messages and the transformed target state capture directional variation and within-block dependencies, yielding a geometric update target. A shared controller learns complementary corrections through task supervision. A bounded log-triangular update combines these corrections, the target, and the previous geometric state while preserving positive definiteness. The geometry governs subsequent propagation, closing the feedback loop. With parameters shared across recurrent steps, task-specific readouts support node classification, link prediction, and graph classification. Experiments on benchmark datasets show that \method{} consistently outperforms state-of-the-art GNN baselines.
\end{abstract}

\section{Introduction}
Graph neural networks (GNNs) learn representations by repeatedly transforming and aggregating neighborhood information~\citep{gilmer2017neural,xu2019how}. Convolution, attention, and multi-hop aggregation organize this computation through different propagation operators~\citep{kipf2017semi,velickovic2018graph,abu2019mixhop}. Beyond selecting and weighting neighbors, message passing involves comparing features and transforming messages before they are combined. Geometry makes these operations explicit: local metrics determine how feature directions contribute to comparisons, while coordinate mappings specify how messages are expressed across local spaces~\citep{chami2019hyperbolic,bodnar2022neural,wang2025protomol}. Learning propagation geometry provides a means of adapting neighborhood interactions to graph structure and the prediction task.

Existing geometric GNNs adapt propagation at two levels: the choice of representation space and the local rules for comparing and transforming messages. At the first level, curvature learning and node-specific space selection address the mismatch between prescribed geometric assumptions and graph structure by adapting the curvature or geometric family to the data~\citep{chami2019hyperbolic,fu2021ace,lee2023node}. At the second level, feature-dependent approaches refine how neighboring information is compared and transformed by learning local maps for message transformation~\citep{bodnar2022neural} and node-wise metrics for feature comparison~\citep{wang2026adaptive}. Whereas the first level concerns which geometry is appropriate, the second concerns how local interactions should adapt to the representations being processed. This makes the information used to infer those local rules consequential. ARGNN estimates its metrics from target features and neighborhood means~\citep{wang2026adaptive}, compressing the neighborhood into a summary that can conceal relationships across feature dimensions. Two neighborhoods can have identical means but different patterns of joint variation. Given identical target features, the metric estimator assigns them the same geometry, even when these differences matter for prediction. Examining individual incoming messages relative to the target state under the current geometry can expose variation obscured by the mean, providing additional information for subsequent geometric updates. This motivates our central question: \textbf{\textit{Can feedback from message passing improve the learning of local propagation geometry for prediction?}}

Leveraging this feedback presents three interconnected challenges. \textit{First, message differences require a consistent geometric interpretation.} When nodes use different local coordinate frames, direct comparisons can conflate coordinate mismatch with variation in the messages themselves. \textit{Second, feedback needs to capture how individual messages vary, rather than only how their average changes.} Opposing contributions can cancel during aggregation, obscuring directional variation and dependencies across feature dimensions. \textit{Third, residual statistics describe how messages vary, but not which directions of variation matter for prediction.} Directions with large residual energy need not be those most relevant to the downstream task. Geometric updates therefore need to incorporate task supervision while remaining well-defined and bounded across repeated propagation steps.

To address these challenges, we propose \method{}, a recurrent framework that jointly evolves node features and local geometry through message-passing feedback. Each node maintains a feature state and a symmetric positive-definite geometry, initialized from a structure-aware prototype atlas and parameterized in block log-triangular coordinates~\citep{lin2019riemannian}. First, to resolve coordinate mismatch, \textit{triangular frame transport} maps transformed source messages into the target node's local coordinates, establishing a common reference for comparison and geometry-weighted aggregation. Second, to capture variation obscured by averaging, \textit{second-order residual feedback} accumulates weighted outer products of residuals between individual aligned messages and the transformed target state. These statistics capture directional variation and within-block dependencies, yielding a geometric update target. Third, \textit{task-guided geometry evolution} complements this target with corrections learned by a shared controller through end-to-end task supervision. A bounded log-triangular update combines the target, corrections, and previous geometric state, preserving positive definiteness across recurrent steps. The updated geometry determines neighborhood weights and frame transformations for subsequent propagation, closing the feedback loop. All recurrent parameters are shared across steps, and task-specific readouts support node classification, link prediction, and graph classification.

Our contributions are summarized as follows:
(1) We investigate message-passing feedback as an information source for local geometry learning, formulating propagation geometry as a persistent node-wise state that evolves jointly with node features.
(2) We develop a closed-loop mechanism in which local geometry governs neighborhood weighting and frame-aligned message propagation, while weighted second-order residual statistics and task-supervised corrections drive its evolution.
(3) We evaluate \method{} on node classification, link prediction, and graph classification, where it demonstrates consistent gains over general, manifold-based, and adaptive GNN baselines.
\section{Related Work}

\textbf{Operator and Structure Adaptation in GNNs.}
GNN adaptation has been explored through propagation operator design and graph structure learning. Higher-order neighborhood mixing integrates information across hop distances, while adaptive propagation adjusts the contributions of different propagation scales~\citep{abu2019mixhop,xu2018representation,gasteiger2018predict,liu2020towards}. To adapt connectivity to downstream tasks, graph structure learning constructs meta-path graphs or iteratively refines graph structure together with node representations~\citep{yun2019graph,chen2020iterative,wang2026sgac}. Related approaches refine noisy adjacency matrices for robust prediction~\citep{jin2020graph,yao2023improving} or learn latent edges that support message passing beyond observed connectivity~\citep{wu2022nodeformer,wang2026usbd}. These studies adapt how messages are aggregated and which nodes communicate. \method{} learns node-wise propagation geometry that explicitly governs message comparison and transport.

\textbf{Geometric Modeling and Adaptive Propagation Spaces.}
Geometric GNNs introduce inductive biases through representation spaces and local maps for message passing. Hyperbolic methods exploit negative curvature to represent hierarchical and scale-free structures~\citep{liu2019hyperbolic,chami2019hyperbolic}, while sheaf-based diffusion learns local linear maps that relate node spaces for feature comparison and propagation~\citep{bodnar2022neural}. Geometric adaptation further includes learning graph-dependent curvature~\citep{fu2021ace} and selecting between Euclidean and hyperbolic spaces at the node level~\citep{lee2023node}. Adaptive Riemannian models learn anisotropic node-wise metric tensors from node features and neighborhood means, adapting geometry to individual nodes~\citep{wang2026adaptive}. \method{} maintains local geometry as a persistent state across propagation steps and updates it through second-order statistics of aligned message residuals and task-guided corrections.
\section{Methodology}

\textbf{Problem Setup.}
Let $G=(V,E,X)$ be an undirected attributed graph with $n=|V|$ nodes, adjacency matrix $A$, and node features $X=[x_1,\ldots,x_n]^\top\in\mathbb{R}^{n\times F_0}$. Each node has a structural signature $u_i\in\mathbb{R}^{q}$, and $U=[u_1,\ldots,u_n]^\top$ collects these signatures. We construct $u_i$ from normalized degree and random-walk return probabilities and keep it fixed during propagation. Let $\mathcal{N}(i)$ denote the neighbors of node $i$ and $\widetilde{\mathcal{N}}(i)=\mathcal{N}(i)\cup\{i\}$ include the self-loop. We consider three prediction tasks: node classification predicts node labels $y_i\in\{1,\ldots,C\}$, link prediction estimates whether a candidate node pair $(u,v)$ is connected, and graph classification predicts a graph label $Y_G\in\{1,\ldots,C\}$. Here, $C$ denotes the number of classes for the corresponding classification task.


\textbf{Overview.}
\method{} couples node features and local geometry through three recurrent components. We initialize feature states and construct local geometries from a structure-aware prototype atlas in block log-triangular coordinates. The recurrent computation then comprises:
(i) \textit{Triangular Frame Transport}, which uses the current geometry to weight neighbors and align transformed messages with the target node's local frame for aggregation and feature updates;
(ii) \textit{Second-Order Residual Feedback}, which summarizes aligned message residuals to capture directional variation and within-block dependencies, producing a target metric on the node’s local coordinates;
and (iii) \textit{Task-Guided Geometry Evolution}, which combines this target with task-supervised corrections and the previous geometric state to update the geometry.
Together, these components form a closed feedback loop, with parameters shared across steps. Task-specific readouts produce node, link, and graph predictions, and the corresponding learning objectives train the framework end to end.

\subsection{Geometric States and Initialization}
\label{sec:geometric_states}

Learning propagation geometry through recurrent feedback requires a compact parameterization that preserves positive definiteness under repeated updates and an initialization informed by local structure. Inspired by log-Cholesky parameterizations~\citep{lin2019riemannian}, we use block log-triangular coordinates and a shared prototype atlas. The former models within-block interactions, while the latter uses structural signatures to initialize node-specific geometries.

At step $l$, node $i$ maintains a local feature state $\xi_i^l\in\mathbb{R}^{d}$ and geometric coordinates $Z_i^l=\{z_{i,b}^l\}_{b=1}^{B}$. We partition the feature space into $B$ blocks of size $m$, with $d=Bm$. Each block is parameterized as
$z_{i,b}^l=[a_{i,b}^l, \ell_{i,b}^l]^\top\in\mathbb{R}^{m(m+1)/2},$
where $a_{i,b}^l\in\mathbb{R}^{m(m-1)/2}$ stores strictly lower-triangular entries and $\ell_{i,b}^l\in\mathbb{R}^{m}$ stores log-diagonal scales. The corresponding triangular frame and local metric are:
\begin{equation}
    L_{i,b}^l
    =
    \operatorname{mat}_{\mathrm{sl}}(a_{i,b}^l)
    +
    \operatorname{Diag}\!\left(\exp(\ell_{i,b}^l)\right),
    \qquad
    g_{i,b}^l
    =
    (L_{i,b}^l)^\top L_{i,b}^l,
    \label{eq:block_geometry}
\end{equation}
where $\operatorname{mat}_{\mathrm{sl}}$ reconstructs a strictly lower-triangular matrix and the exponential acts elementwise.
The positive diagonal makes $L_{i,b}^l$ invertible and $g_{i,b}^l$ positive definite. $L_{i,b}^l$ maps local coordinates into a shared environment, where $g_{i,b}^l$ measures lengths with $\|L_{i,b}^l v\|_2^2=v^\top g_{i,b}^l v$.
Combining the blocks:
\begin{equation}
    L_i^l=\bigoplus_{b=1}^{B}L_{i,b}^l,
    \qquad
    g_i^l=\bigoplus_{b=1}^{B}g_{i,b}^l
    =(L_i^l)^\top L_i^l,
    \label{eq:node_geometry}
\end{equation}
where $\bigoplus$ denotes the block-diagonal direct sum. This representation uses $d(m+1)/2$ geometric coordinates per node while retaining off-diagonal interactions within each block.
To bound coordinate magnitudes during initialization and recurrent updates, we define the admissible set
    $\mathcal{Z}
    =
    [-a_{\max},a_{\max}]^{m(m-1)/2}
    \times
    [\ell_{\min},\ell_{\max}]^{m},$
where $a_{\max}>0$ and $\ell_{\min}<\ell_{\max}$ are finite bounds. We denote coordinatewise clipping to this set by $\Pi_{\mathcal{Z}}$. This constraint bounds the geometric coordinates, while positive definiteness follows from the log-triangular reconstruction.

Before propagation feedback becomes available, we initialize geometry from the structural signatures $u_i$. Let $\{z_{k,b}^{\mathrm{proto}}\}_{k=1}^{K}$ denote $K$ learnable prototype coordinates for block $b$. 
A shared gate assigns each node a distribution over the prototypes, yielding
\begin{equation}
        \xi_i^0
        =W_{\mathrm{in}}x_i,
        \qquad
        \alpha_i
        =\operatorname{softmax}(W_gu_i+b_g),\qquad
        z_{i,b}^0
        =
        \Pi_{\mathcal{Z}}
        \left(
            \sum\nolimits_{k=1}^{K}
            \alpha_{ik}z_{k,b}^{\mathrm{proto}}
        \right),
    \label{eq:prototype_initialization}
\end{equation}
where $W_{\mathrm{in}}$, $W_g$, and $b_g$ are learnable parameters. The atlas shares geometric templates across nodes, while the structural gate determines their initial mixtures. These initial states provide the starting point for frame-aligned propagation and subsequent feedback-driven geometric updates.

\tcolorboxenvironment{lemma}{
  breakable, colback=gray!5, colframe=gray!50!black, boxrule=0.4pt, arc=3pt,
  left=3pt, right=3pt, top=1pt, bottom=1pt,
  grow to left by=0.4pt, grow to right by=0.4pt, before skip=4pt, after skip=4pt
}
\tcolorboxenvironment{proposition}{
  breakable, colback=blue!5, colframe=blue!50!black, boxrule=0.4pt, arc=3pt,
  left=3pt, right=3pt, top=1pt, bottom=1pt,
  grow to left by=0.4pt, grow to right by=0.4pt, before skip=4pt, after skip=4pt
}
\begin{lemma}[Well-Posed Geometric Parameterization]\label{lem:wellposed}
Let $m\ge 1$ and $z,z'\in\mathcal Z$. The frame $L(z)$ is lower triangular with
$\det L(z)=\exp(\sum_{t}\ell_t)>0$, so it is invertible and $g(z)\in\mathrm{SPD}(m)$. There are
constants $\sigma_+\ge\sigma_->0$ and $\kappa_L>0$, depending only on $m$, $a_{\max}$,
$\ell_{\min}$, and $\ell_{\max}$, such that
\[
\|L(z)\|_2\le\sigma_+,\qquad \|L(z)^{-1}\|_2\le\sigma_-^{-1},\qquad
\sigma_-^{2}\,I_m\ \preceq\ g(z)\ \preceq\ \sigma_+^{2}\,I_m ,
\]
and the reconstruction is Lipschitz: $\|L(z)-L(z')\|_F\le\kappa_L\|z-z'\|_2$ and
$\|L(z)^{-1}-L(z')^{-1}\|_2\le\sigma_-^{-2}\kappa_L\|z-z'\|_2$. All statements carry over to
the block-diagonal frames $L_i=\bigoplus_{b}L(z_{i,b})$ and $g_i=L_i^{\top}L_i$, with
$\|z-z'\|_2$ replaced by $\|Z_i-Z_i'\|_F:=(\sum_{b}\|z_{i,b}-z'_{i,b}\|_2^{2})^{1/2}$.
\end{lemma}

Lemma~\ref{lem:wellposed} makes the recurrence well-posed: the triangular solves of Eq.~\eqref{eq:cholesky_frame_transport}
and the projected update of Eq.~\eqref{eq:task_guided_geometry_update} are defined at every admissible state, and the bounds
supply the constants for Propositions~\ref{prop:gauge} and~\ref{prop:stability}. The proof, with
explicit expressions for $\sigma_\pm$ and $\kappa_L$, is in Appendix~\ref{proof_1}.

\vspace{-0.1cm}
\subsection{Triangular Frame Transport}\label{sec:frame_transport}
\vspace{-0.1cm}

Message passing across node-specific frames requires a common coordinate reference for comparison and aggregation. GCN~\citep{kipf2017semi} and GAT~\citep{velickovic2018graph} transform neighboring features in shared coordinates and weight them through degree normalization or learned attention. However, scalar reweighting alone cannot align distinct frames. We introduce \textit{triangular frame transport} to map transformed source messages through a shared environment into the target node's local coordinates before aggregation. The same geometric state determines neighborhood weights, coupling weighting with alignment and providing a consistent reference for residual computation.

Specifically, at step $l$, we construct environment representations from the current feature and geometric states and compute neighborhood weights from their pairwise distances:
\begin{equation}
    h_i^l=L_i^l\xi_i^l,
    \qquad
    e_{ij}^l=\|h_i^l-h_j^l\|_2^2,
    \qquad
    \omega_{ij}^l
    =
    \frac{\exp(-e_{ij}^l/\tau)}
    {\sum_{k\in\widetilde{\mathcal{N}}(i)}
    \exp(-e_{ik}^l/\tau)},
    \label{eq:weights}
\end{equation}
where $j\in\widetilde{\mathcal{N}}(i)$, $\|\cdot\|_2$ denotes the Euclidean norm, and $\tau>0$ controls the concentration of the weights. The weights satisfy $\omega_{ij}^l\geq 0$ and $\sum_{j\in\widetilde{\mathcal{N}}(i)}\omega_{ij}^l=1$.
To align messages before aggregation, we apply a shared feature transformation $\phi(\xi)=W_\phi\xi$. For each block $b$, the transformed source message is mapped through the source frame and pulled back into the target frame:\vspace{-3pt}
\begin{equation}
    m_{j\rightarrow i,b}^l
    =
    \operatorname{TriSolve}
    \left(
        L_{i,b}^l,\,
        L_{j,b}^l\big(\phi(\xi_j^l)\big)_b
    \right),
    \label{eq:cholesky_frame_transport}
\end{equation}
where $(\cdot)_b$ extracts the $b$-th feature block and
$\operatorname{TriSolve}(L,y)$ solves $Lw=y$ without explicitly forming $L^{-1}$. 
Lemma~\ref{lem:wellposed} ensures this solution is unique for every admissible geometric state. The transported message satisfies $L_{i,b}^l m_{j\rightarrow i,b}^l=L_{j,b}^l(\phi(\xi_j^l))_b$: both sides represent the same transformed message in the shared environment, while $m_{j\rightarrow i,b}^l$ expresses it in the target node's coordinates.
We concatenate the blockwise messages and aggregate them using the geometry-conditioned weights:
\begin{equation}
    \begin{aligned}
        m_{j\rightarrow i}^l
        =
        \operatorname{concat}
        \left(
            m_{j\rightarrow i,1}^l,\ldots,
            m_{j\rightarrow i,B}^l
        \right),\qquad
        \bar{\xi}_i^l
        =
        \sum\nolimits_{j\in\widetilde{\mathcal{N}}(i)}
        \omega_{ij}^l m_{j\rightarrow i}^l.
    \end{aligned}
    \label{eq:aligned_aggregation}
\end{equation}
All terms in $\bar{\xi}_i^l$ are expressed in the target frame. A feature gate balances the state against the aggregated message:
\begin{equation}
    \begin{aligned}
        r_i^l
        =
        \sigma\left(
            W_r[\xi_i^l\|\bar{\xi}_i^l]+b_r
        \right),\qquad
        \xi_i^{l+1}
        =
        (\mathbf{1}-r_i^l)\odot\xi_i^l
        +
        r_i^l\odot\bar{\xi}_i^l,
    \end{aligned}
    \label{eq:feature_gate}
\end{equation}
where $W_r\in\mathbb{R}^{d\times 2d}$ and $b_r\in\mathbb{R}^{d}$ are learnable parameters, $\sigma$ is the elementwise sigmoid, $\|$ denotes concatenation, $\odot$ denotes elementwise multiplication, and $\mathbf{1}\in\mathbb{R}^{d}$ is the all-ones vector. 
The parameters $W_\phi$, $W_r$, and $b_r$ are shared across nodes and recurrent steps.

\begin{proposition}[Frame-Aligned Transport]\label{prop:gauge}
Fix a recurrent step $l$ with admissible geometric states, and let
$T^{l}_{j\to i,b}=(L^{l}_{i,b})^{-1}L^{l}_{j,b}$ and $M^{l}_j=L^{l}_jW_\phi(L^{l}_j)^{-1}$.
Then $m^{l}_{j\to i,b}=T^{l}_{j\to i,b}(\phi(\xi^{l}_j))_b$ with
$\|T^{l}_{j\to i,b}\|_2\le\sigma_+/\sigma_-$, and the frame maps satisfy
$T^{l}_{i\to i,b}=I_m$ and $T^{l}_{j\to i,b}T^{l}_{k\to j,b}=T^{l}_{k\to i,b}$, so their
product along any walk depends only on its endpoints. In environment coordinates, the
aggregate and the residuals $\delta^{l}_{ij}:=m^{l}_{j\to i}-\phi(\xi^{l}_i)$ read
\[
L^{l}_i\bar\xi^{l}_i=\sum\nolimits_{j\in\tilde{\mathcal N}(i)}\omega^{l}_{ij}\,M^{l}_jh^{l}_j,
\qquad
L^{l}_i\delta^{l}_{ij}=M^{l}_jh^{l}_j-M^{l}_ih^{l}_i,
\]
where each $M^{l}_j$ is similar to $W_\phi$ with
$\|M^{l}_j\|_2\le(\sigma_+/\sigma_-)\|W_\phi\|_2$, and $M^{l}_j=cI_d$ when $W_\phi=cI_d$.
\end{proposition}

In environment coordinates, one recurrent step is thus Gaussian-kernel attention with temperature $\tau$, where the frames enter through the energies $e^{l}_{ij}$, the transforms $M^{l}_j$, and the pull-back $(L^{l}_i)^{-1}$ that makes the residuals of Sec.~\ref{sec:residual_feedback} comparable across neighbors. The proof is in Appendix~\ref{app:proof_gauge}.

\vspace{-0.1cm}
\subsection{Second-Order Residual Feedback}
\label{sec:residual_feedback}

The transport step aligns messages and updates node features under the current geometry, and the aligned messages give a common reference for local comparison. ARGNN estimates local metrics from target features and neighborhood means~\citep{wang2026adaptive}, but identical summaries can conceal directional variation and cross-feature dependence. We therefore examine the residuals of individual aligned messages from the transformed target state, not their aggregate, and accumulate their weighted outer products into a geometric target that captures directional variation and within-block dependence. Interactions under the current geometry thereby become feedback for its evolution.

Specifically, for node $i$ and block $b$, we compute the residual between each aligned
message and the transformed target state,
$\delta_{ij,b}^l=m_{j\rightarrow i,b}^l-\big(\phi(\xi_i^l)\big)_b$, which by
Proposition~\ref{prop:gauge} are expressed in the same frame. To capture directional variation
beyond the mean, we form a second-moment matrix whose quadratic form measures residual
energy along any direction, reusing $\omega_{ij}^l$ to preserve aggregation weights:
\begin{equation}
    C_{i,b}^l
    =
    \sum\nolimits_{j\in\widetilde{\mathcal{N}}(i)}
    \omega_{ij}^l
    \delta_{ij,b}^l(\delta_{ij,b}^l)^\top
    +
    \epsilon_s I_m,
    \label{eq:second_moment}
\end{equation}
where $\epsilon_s>0$ provides a positive spectral lower bound. The diagonal entries combine weighted squared residuals with $\epsilon_s$, while the off-diagonal entries capture weighted cross-products within each block. Since residuals are measured relative to the transformed target state rather than centered at their neighborhood mean, $C_{i,b}^l$ is a regularized second-moment matrix. 
Writing $\bar{\delta}_{i,b}^l=\sum_{j}\omega_{ij}^l\delta_{ij,b}^l$ for the weighted mean residual, the classical mean--covariance decomposition gives, for every $v\in\mathbb{R}^m$,
\begin{equation}
    v^\top C_{i,b}^l v
    =
    \epsilon_s\|v\|_2^2
    +
    \big(v^\top\bar{\delta}_{i,b}^l\big)^2
    +
    \sum\nolimits_{j\in\widetilde{\mathcal{N}}(i)}
    \omega_{ij}^l
    \left[
        v^\top
        \big(\delta_{ij,b}^l-\bar{\delta}_{i,b}^l\big)
    \right]^2.
    \label{eq:directional_residual_decomposition}
\end{equation}
The last term is the weighted residual covariance: residuals of opposite sign cancel in the
mean yet still contribute to $C_{i,b}^l$, which is exactly what mean-based estimators
discard. Since the geometric state evolves in log-triangular coordinates, we convert this
statistic into a coordinate target for updates:
\begin{equation}
    R_{i,b}^l
    =
    \operatorname{SChol}(C_{i,b}^l),
    \qquad
    \widehat z_{i,b}^l
    =
    \begin{bmatrix}
        \operatorname{Svec}_{\mathrm{sl}}(R_{i,b}^l)^\top, &
        \log\!\left(\operatorname{diag}(R_{i,b}^l)\right)^\top
    \end{bmatrix}^{\top},
    \label{eq:coordinate_target}
\end{equation}
where $\operatorname{SChol}(C)$ returns the positive-diagonal lower-triangular factor $R$ with $RR^\top=C+\eta I_m$, $\operatorname{Svec}_{\mathrm{sl}}$ stacks the strictly lower-triangular entries in fixed order, and $\log$ is elementwise. Since Cholesky factorization is bijective on $\operatorname{SPD}(m)$, $\widehat z_{i,b}^l$ encodes $C_{i,b}^l+\eta I_m$ losslessly, and $R_{i,b}^l$ maps the local unit ball onto the residual second-moment ellipsoid, stretching high-energy residual directions.

\begin{proposition}[Second-Order Refinement]\label{prop:separation}
Fix a step $l$. Write $\mathcal R^{l}_{i,b}=\{(\omega^{l}_{ij},\delta^{l}_{ij,b})\}_{j\in\tilde{\mathcal N}(i)}$
for the residual configuration of node $i$ in block $b$, with weighted mean
$\bar\delta^{l}_{i,b}$, and call an update \emph{first-order} if it depends on
$\mathcal R^{l}_{i,b}$ only through the weights and $\bar\delta^{l}_{i,b}$. Let $i,i'$ be nodes
with $\xi^{l}_i=\xi^{l}_{i'}$, $Z^{l}_i=Z^{l}_{i'}$, $u_i=u_{i'}$, and
$\bar\delta^{l}_{i,b}=\bar\delta^{l}_{i',b}$ for all $b$.
\begin{enumerate}[label=(\roman*),leftmargin=1.6em,itemsep=1pt,topsep=2pt]
\item Every first-order update, including the mean-only target
$C^{l,\mathrm{mean}}_{i,b}=\bar\delta^{l}_{i,b}(\bar\delta^{l}_{i,b})^{\top}+\epsilon_sI_m$,
the controller of Eq.~\eqref{eq:correction}, and the feature update of
Eq.~\eqref{eq:feature_gate}, returns identical outputs for $i$ and $i'$; in particular
$\xi^{l+1}_i=\xi^{l+1}_{i'}$ and $\Delta z^{l}_{i,b}=\Delta z^{l}_{i',b}$.
\item The second-order target is injective in the statistic:
$C^{l}_{i,b}\ne C^{l}_{i',b}$ implies $\hat z^{l}_{i,b}\ne\hat z^{l}_{i',b}$. If
$\lambda\in(0,1]$ and Eq.~\eqref{eq:task_guided_geometry_update} is not clipped in block
$b$, then $Z^{l+1}_i\ne Z^{l+1}_{i'}$, and $h^{l+1}_i\ne h^{l+1}_{i'}$ unless
$\xi^{l+1}_i\in\ker(L^{l+1}_i-L^{l+1}_{i'})$, a proper subspace.
\item If $i$ has at least two neighbors and $\mathcal R^{l}_{i',b}$ ranges over all
configurations with the same weights and first moment, those with
$C^{l}_{i',b}=C^{l}_{i,b}$ form a Lebesgue-null set.
\item The target depends on $\mathcal R^{l}_{i,b}$ only through $C^{l}_{i,b}$ and is
therefore strictly weaker than an injective aggregator: for $m=1$ there are distinct
three-point residual sets with equal weights, first, and second moments. Conversely, for
$m\ge2$ the sets $\{0,\pm v\}$ and $\{0,\pm w\}$ with $v=(1,1,0,\dots,0)^{\top}$,
$w=(1,-1,0,\dots,0)^{\top}$ and equal weights on $\pm v$, $\pm w$ share the mean-only and
the diagonal target
$C^{l,\mathrm{diag}}_{i,b}=\mathrm{Diag}\big(\sum_j\omega^{l}_{ij}\,\delta^{l}_{ij,b}\odot\delta^{l}_{ij,b}\big)+\epsilon_sI_m$,
yet $C^{l}_{i,b}\ne C^{l}_{i',b}$.
\end{enumerate}
\end{proposition}

Proposition~\ref{prop:separation} places the feedback of \method{} in a moment hierarchy:
first-order updates, including the controller and mean-based metric
estimation~\citep{wang2026adaptive}, cannot register variation that cancels in the mean,
the second-order target registers it for almost every configuration, and the mean-only and
diagonal targets sit strictly below it. The proof is in Appendix~\ref{app:proof_separation}.

\subsection{Task-Guided Geometry Evolution}\label{sec:geometry_evolution}

The residual target captures local variation but does not identify which differences matter for prediction. Rather than directly estimating local metrics from feature summaries~\citep{wang2026adaptive}, 
we learn a complementary correction through task supervision. This correction, the residual target, and the previous geometric state are combined through a bounded log-triangular update.

The geometric correction is predicted from the current local features, aligned neighborhood aggregate, and structural signature. For node $i$ and block $b$, these quantities are concatenated into
$q_{i,b}^l=[(\xi_i^l)_b\|(\bar{\xi}_i^l)_b\|u_i]$.
A shared MLP $\Gamma_\theta$ processes this input using a single hidden layer:
   $ \Gamma_\theta(q_{i,b}^l)
    =
    W_{\Gamma,2}
    \operatorname{SiLU}
    \left(
        W_{\Gamma,1}q_{i,b}^l+b_{\Gamma,1}
    \right)
    +
    b_{\Gamma,2},$
where $\theta$ collects trainable parameters and $\operatorname{SiLU}(x)=x/(1+e^{-x})$.
The output is reshaped into $U_{i,b}^l$, $V_{i,b}^l$ and a log-diagonal correction
$d_{i,b}^l$. Following Sec.~\ref{sec:geometric_states}, the strictly lower-triangular entries
of $U_{i,b}^l(V_{i,b}^l)^\top$ correct $a_{i,b}^l$ and $d_{i,b}^l$ corrects
$\ell_{i,b}^l$. Concatenating them gives
\begin{equation}
    \Delta z_{i,b}^l
    =
    \begin{bmatrix}
        \operatorname{Svec}_{\mathrm{sl}}
        \left(U_{i,b}^l(V_{i,b}^l)^\top\right)^\top
        &
        (d_{i,b}^l)^\top
    \end{bmatrix}^{\top}
    \in\mathbb{R}^{m(m+1)/2}.
    \label{eq:correction}
\end{equation}
The controller is trained using task objectives, with parameters shared across nodes, blocks, 
and recurrent steps. 
The geometric update combines the previous state, residual target, and correction:
\begin{equation}
    z_{i,b}^{l+1}
    =
    \Pi_{\mathcal{Z}}
    \left(
        (1-\lambda)z_{i,b}^l
        +
        \lambda\widehat z_{i,b}^l
        +
        \gamma\Delta z_{i,b}^l
    \right),
    \label{eq:task_guided_geometry_update}
\end{equation}
where $\lambda\in(0,1]$ and $\gamma\geq0$. The target supplies a reference derived from local interactions, while the learned correction allows the update to depart from that reference under task supervision. 

\begin{proposition}[Stability and Convergence]\label{prop:stability}
Fix $\tau>0$, $\epsilon_s>0$, $\eta\ge0$ and a feature bound $\Xi$. Let
$\mathcal D=\{(\xi,Z):\|\xi_i\|_2\le\Xi,\ z_{i,b}\in\mathcal Z\}$, let
$F:(\xi^{l},Z^{l})\mapsto(\xi^{l+1},Z^{l+1})$ denote one recurrent step, and let
$d_\infty(S,S')=\max_i\big(\|\xi_i-\xi'_i\|_2+\|Z_i-Z'_i\|_F\big)$.
\begin{enumerate}[label=(\roman*),leftmargin=1.6em,itemsep=1pt,topsep=2pt]
\item Eq.~\eqref{eq:task_guided_geometry_update} is a projected proximal step:
$\Pi_{\mathcal Z}$ is the Euclidean projection onto $\mathcal Z$, and $z^{l+1}_{i,b}$ is the
unique minimizer over $\mathcal Z$ of
\[
\tfrac{1-\lambda}{2}\|z-z^{l}_{i,b}\|_2^{2}+\tfrac{\lambda}{2}\|z-\hat z^{l}_{i,b}\|_2^{2}
-\gamma\langle\Delta z^{l}_{i,b},\,z-z^{l}_{i,b}\rangle .
\]
It differs from the uncorrected update
$z^{l+1,0}_{i,b}=\Pi_{\mathcal Z}\big((1-\lambda)z^{l}_{i,b}+\lambda\hat z^{l}_{i,b}\big)$
by at most $\gamma\|\Delta z^{l}_{i,b}\|_2$.
\item One step is Lipschitz on $\mathcal D$:
$d_\infty\big(F(S),F(S')\big)\le K\,d_\infty(S,S')$ with
$K=c_\xi+(1-\lambda)+\lambda c_{\hat Z}+\gamma c_\Delta$, where the constants depend only on
$m$, $B$, $\sigma_\pm$, $\kappa_L$, $\epsilon_s$, $\eta$, $\tau$, $\Xi$, $\lambda$, $\gamma$,
and the norms of $W_\phi$, $W_r$, and $\Gamma_\theta$, but not on node degrees.
\item For fixed $\xi$, let
$\Phi_\xi(Z)=\Pi_{\mathcal Z}\big((1-\lambda)Z+\lambda\hat Z(Z;\xi)+\gamma\Delta Z(Z;\xi)\big)$.
If $Z^{\star}$ is an unclipped fixed point of $\Phi_\xi$ whose Jacobian
$J=\partial\Phi_\xi(Z^{\star})$ has spectral radius $\rho(J)<1$, then $Z^{\star}$ is locally
attracting with rate $\rho(J)$: for every $\varepsilon>0$ there are $c>0$ and a neighborhood
$\mathcal U$ of $Z^{\star}$ such that
$\|Z^{l}-Z^{\star}\|_F\le c\,(\rho(J)+\varepsilon)^{l}\,\|Z^{0}-Z^{\star}\|_F$ for all
$Z^{0}\in\mathcal U$.
\item If moreover $\hat Z(\cdot;\xi)$ and $\Delta Z(\cdot;\xi)$ are $L_{\hat z}$- and
$L_\Delta$-Lipschitz on $\mathcal Z^{nB}$ and
$\rho:=(1-\lambda)+\lambda L_{\hat z}+\gamma L_\Delta<1$, the fixed point is unique and
$\|Z^{l+1}-Z^{l}\|_F\le\rho^{l}\|Z^{1}-Z^{0}\|_F$ from every initialization.
\end{enumerate}
\end{proposition}

Proposition~\ref{prop:stability} characterizes the update as a projected proximal step,
bounds the single-step effect of the learned correction, gives a one-step Lipschitz
constant that does not depend on node degrees, and states conditions under which the
geometry converges. The proof is in Appendix~\ref{app:proof_stability}.

\subsection{Learning Objectives}\label{sec:objectives}

The prediction losses depend on the final geometric coordinates
$\{z_{i,b}^{L}\}_{b=1}^{B}$ through the environment representation
$h_i^L=L_i^L\xi_i^L$, where $L_i^L$ is assembled from the block
factors reconstructed using Eq.~\eqref{eq:block_geometry}.
The intermediate updates
$z_{i,b}^{l}\mapsto z_{i,b}^{l+1}$ also shape subsequent
message passing by determining the neighborhood weights and frame transformations.
The task objectives below train the geometric controller $\Gamma_\theta$, propagation parameters, and task-specific readouts jointly through these dependencies.

\textbf{Node Classification.}
A shared linear classifier maps $h_i^L$ to class logits,
followed by softmax normalization. Training minimizes
cross-entropy over the labeled nodes
$V_{\mathrm{tr}}\subseteq V$:
\begin{equation}
    s_i=W_oh_i^L+b_o,
    \quad
    p_i=\operatorname{softmax}(s_i),
    \quad
    \mathcal{L}_{\mathrm{node}}
    =
    -\frac{1}{|V_{\mathrm{tr}}|}
    \sum_{i\in V_{\mathrm{tr}}}
    \sum_{c=1}^{C}
    \mathbb{I}[y_i=c]\log p_{i,c},
    \label{eq:node_objective}
\end{equation}
where $W_o$,
$b_o$ are learnable classifier parameters, $C$ is the number of node classes.
$s_i,p_i$ denote the logits
and class probabilities.
$y_i$ is the ground-truth,
and $\mathbb{I}[y_i=c]=1$ when $y_i=c$
and 0 otherwise.

\textbf{Graph Classification.}
For a graph $G$ with node set $V_G$, we mean-pool the final environment representations in their shared coordinate space and apply a shared MLP to obtain class probabilities:
\begin{equation}
    h_G=\frac{1}{|V_G|}\sum\nolimits_{i\in V_G}h_i^L,
    \;
p_G=\operatorname{softmax}\!\left(\operatorname{MLP}(h_G)\right),
    \;
    \mathcal{L}_{\mathrm{graph}}
    =
    \mathbb{E}_{G\sim\mathcal{D}_{\mathrm{tr}}^{\mathrm{graph}}}
    \left[\ell_{\mathrm{CE}}(p_G,Y_G)\right],
    \label{eq:graph_objective}
\end{equation}
where $h_G$ is the graph representation.
The vector $p_G$ contains the class probabilities, with $p_{G,c}$ denoting the probability of class $c$, and $Y_G$ is the ground-truth label.
The cross-entropy loss is
$\ell_{\mathrm{CE}}(p_G,Y_G)
=-\sum_{c=1}^{C}\mathbb{I}[Y_G=c]\log p_{G,c}$,
where $\mathbb{I}[Y_G=c]=1$ when $Y_G=c$ and 0 otherwise.

\textbf{Link Prediction.}
For a candidate pair $(u,v)$, we construct a symmetric pair representation
$f_{uv}=[h_u^L\odot h_v^L\|\,|h_u^L-h_v^L|]\in\mathbb{R}^{2d}$
from the final environment representations. A shared MLP predicts whether the nodes are connected. Let
$\mathcal{E}_{\mathrm{tr}}=E^+\cup E^-$ contain positive training edges $E^+$ and sampled negative pairs $E^-$. The edge probability and training objective are
\begin{equation}
    p_{uv}
    =
    \sigma\!\left(\operatorname{MLP}_{\mathrm{link}}(f_{uv})\right),
    \qquad
    \mathcal{L}_{\mathrm{link}}
    =
    \mathbb{E}_{(u,v)\sim\mathcal{E}_{\mathrm{tr}}}
    \left[\ell_{\mathrm{BCE}}(p_{uv},y_{uv})\right],
    \label{eq:link_objective}
\end{equation}
where $\operatorname{MLP}_{\mathrm{link}}$
is a learnable decoder,
$\sigma(t)$ is the sigmoid function, and
$p_{uv}$ is the predicted edge probability.
The label $y_{uv}=1$ for positive edges and 0 for negative pairs.


\section{Experiments}

\begin{table*}[t]
\centering
\tiny
\caption{Performance comparisons (in \%) between baselines and \method{} for Node Classification (NC) and Link Prediction (LP) on different datasets. \textbf{Bold} indicates the best performance.}
\label{tab:link_and_node}
\vspace{-0.2cm}
\setlength{\tabcolsep}{2.2pt}
\begin{tabular}{ccccccccccccc}
\toprule
\multirow{2}{*}{Model}
& \multicolumn{2}{c}{CiteSeer}
& \multicolumn{2}{c}{PubMed}
& \multicolumn{2}{c}{CS}
& \multicolumn{2}{c}{Physics}
& \multicolumn{2}{c}{Photo}
& \multicolumn{2}{c}{Computers} \\
\cline{2-13} 
& \multicolumn{1}{c}{\raisebox{-0.5ex}[0pt][0pt]{NC}}
& \multicolumn{1}{c}{\raisebox{-0.5ex}[0pt][0pt]{LP}}
& \multicolumn{1}{c}{\raisebox{-0.5ex}[0pt][0pt]{NC}}
& \multicolumn{1}{c}{\raisebox{-0.5ex}[0pt][0pt]{LP}}
& \multicolumn{1}{c}{\raisebox{-0.5ex}[0pt][0pt]{NC}}
& \multicolumn{1}{c}{\raisebox{-0.5ex}[0pt][0pt]{LP}}
& \multicolumn{1}{c}{\raisebox{-0.5ex}[0pt][0pt]{NC}}
& \multicolumn{1}{c}{\raisebox{-0.5ex}[0pt][0pt]{LP}}
& \multicolumn{1}{c}{\raisebox{-0.5ex}[0pt][0pt]{NC}}
& \multicolumn{1}{c}{\raisebox{-0.5ex}[0pt][0pt]{LP}}
& \multicolumn{1}{c}{\raisebox{-0.5ex}[0pt][0pt]{NC}}
& \multicolumn{1}{c}{\raisebox{-0.5ex}[0pt][0pt]{LP}} \\
\midrule

GCN
& 71.5$_{\pm 1.7}$ & 92.3$_{\pm 0.9}$
& 87.6$_{\pm 0.5}$ & 92.9$_{\pm 0.6}$
& 93.8$_{\pm 0.4}$ & 92.7$_{\pm 0.6}$
& 93.5$_{\pm 0.2}$ & 92.6$_{\pm 1.3}$
& 92.1$_{\pm 0.5}$ & 86.1$_{\pm 0.7}$
& 87.8$_{\pm 0.7}$ & 86.9$_{\pm 0.8}$ \\
GIN
& 70.6$_{\pm 1.2}$ & 93.0$_{\pm 1.0}$
& 86.6$_{\pm 0.6}$ & 89.5$_{\pm 0.7}$
& 91.3$_{\pm 0.6}$ & 93.3$_{\pm 0.7}$
& 94.2$_{\pm 0.5}$ & 92.1$_{\pm 0.8}$
& 92.7$_{\pm 0.5}$ & 87.7$_{\pm 0.4}$
& 87.1$_{\pm 0.8}$ & 84.1$_{\pm 1.2}$ \\
ML$^2$-GCL
& 73.7$_{\pm 2.0}$ & 93.8$_{\pm 1.0}$
& 87.9$_{\pm 0.5}$ & 95.8$_{\pm 0.3}$
& 92.1$_{\pm 0.5}$ & 97.1$_{\pm 0.1}$
& 93.4$_{\pm 1.6}$ & 97.2$_{\pm 0.1}$
& 92.8$_{\pm 0.8}$ & 96.7$_{\pm 0.1}$
& 87.4$_{\pm 0.8}$ & 95.8$_{\pm 0.3}$ \\
AMPs
& 74.4$_{\pm 1.8}$ & 93.5$_{\pm 0.9}$
& 88.9$_{\pm 0.3}$ & 96.6$_{\pm 0.3}$
& 94.6$_{\pm 0.3}$ & 97.1$_{\pm 0.2}$
& 96.1$_{\pm 0.1}$ & 96.6$_{\pm 0.6}$
& 94.6$_{\pm 0.8}$ & 96.8$_{\pm 0.2}$
& 90.2$_{\pm 0.6}$ & 96.0$_{\pm 0.1}$ \\
WaveGC
& 75.4$_{\pm 1.9}$ & 92.8$_{\pm 0.7}$
& 87.6$_{\pm 0.5}$ & 97.5$_{\pm 0.3}$
& 94.4$_{\pm 0.3}$ & 97.1$_{\pm 0.2}$
& 96.2$_{\pm 0.1}$ & 97.2$_{\pm 0.2}$
& 94.6$_{\pm 0.8}$ & 97.7$_{\pm 0.4}$
& 90.2$_{\pm 0.8}$ & 97.5$_{\pm 0.4}$ \\
SPARROW
& 73.1$_{\pm 2.1}$ & 92.9$_{\pm 0.9}$
& 85.5$_{\pm 0.6}$ & 97.0$_{\pm 0.7}$
& 92.0$_{\pm 0.3}$ & 97.0$_{\pm 0.1}$
& 94.5$_{\pm 0.3}$ & 97.2$_{\pm 0.3}$
& 93.4$_{\pm 1.0}$ & 97.8$_{\pm 0.1}$
& 88.2$_{\pm 0.5}$ & 97.3$_{\pm 0.2}$ \\
G$^2$Former
& 72.7$_{\pm 1.7}$ & 93.0$_{\pm 0.6}$
& 88.6$_{\pm 0.5}$ & 97.1$_{\pm 0.2}$
& 93.7$_{\pm 0.4}$ & 96.4$_{\pm 0.9}$
& 95.5$_{\pm 0.2}$ & 97.4$_{\pm 0.4}$
& 94.1$_{\pm 0.7}$ & 97.9$_{\pm 0.3}$
& 90.4$_{\pm 0.4}$ & 97.1$_{\pm 0.1}$ \\
\midrule
HGCN
& 74.6$_{\pm 1.8}$ & 92.1$_{\pm 1.0}$
& 84.8$_{\pm 0.7}$ & 91.3$_{\pm 0.3}$
& 91.9$_{\pm 0.4}$ & 93.2$_{\pm 0.4}$
& 92.0$_{\pm 3.1}$ & 95.7$_{\pm 0.4}$
& 89.5$_{\pm 0.8}$ & 92.6$_{\pm 0.9}$
& 85.9$_{\pm 0.7}$ & 88.1$_{\pm 0.2}$ \\
D-GCN
& 74.9$_{\pm 1.1}$ & 92.0$_{\pm 0.6}$
& 86.9$_{\pm 0.9}$ & 93.1$_{\pm 0.4}$
& 93.8$_{\pm 0.2}$ & 96.2$_{\pm 0.1}$
& 93.6$_{\pm 0.2}$ & 97.1$_{\pm 0.3}$
& 93.7$_{\pm 0.6}$ & 97.1$_{\pm 0.1}$
& 87.1$_{\pm 0.5}$ & 94.1$_{\pm 0.6}$ \\
SPDGNN
& 74.9$_{\pm 1.7}$ & 92.7$_{\pm 0.8}$
& 87.2$_{\pm 0.6}$ & 93.8$_{\pm 0.7}$
& 92.6$_{\pm 0.4}$ & 94.9$_{\pm 0.2}$
& 94.1$_{\pm 0.2}$ & 94.5$_{\pm 0.6}$
& 93.0$_{\pm 0.8}$ & 95.7$_{\pm 0.5}$
& 88.7$_{\pm 1.2}$ & 93.1$_{\pm 0.6}$ \\
\midrule
ACE-HGNN
& 74.0$_{\pm 1.5}$ & 92.9$_{\pm 0.6}$
& 86.8$_{\pm 0.5}$ & 96.9$_{\pm 0.1}$
& 92.5$_{\pm 0.3}$ & 97.0$_{\pm 0.1}$
& 93.9$_{\pm 0.2}$ & 97.3$_{\pm 0.2}$
& 94.0$_{\pm 1.0}$ & 97.2$_{\pm 0.1}$
& 89.5$_{\pm 1.1}$ & 97.3$_{\pm 0.5}$ \\
BEC-GNN
& 74.0$_{\pm 1.6}$ & 94.0$_{\pm 1.3}$
& 87.3$_{\pm 0.3}$ & 97.4$_{\pm 0.2}$
& 93.3$_{\pm 0.4}$ & 97.2$_{\pm 0.2}$
& 96.3$_{\pm 0.2}$ & 97.2$_{\pm 0.2}$
& 94.1$_{\pm 0.5}$ & 97.1$_{\pm 0.1}$
& 89.8$_{\pm 0.8}$ & 96.7$_{\pm 0.1}$ \\
GNRF
& 74.2$_{\pm 1.7}$ & 93.6$_{\pm 0.5}$
& 87.4$_{\pm 0.6}$ & 95.8$_{\pm 0.2}$
& 93.7$_{\pm 0.6}$ & 97.0$_{\pm 0.1}$
& 95.4$_{\pm 1.0}$ & 97.3$_{\pm 0.3}$
& 93.9$_{\pm 0.9}$ & 96.8$_{\pm 0.4}$
& 89.0$_{\pm 0.9}$ & 96.3$_{\pm 0.3}$ \\
ARGNN
& 75.6$_{\pm 1.2}$ & 94.1$_{\pm 0.6}$
& 88.8$_{\pm 0.2}$ & 97.6$_{\pm 0.4}$
& 94.4$_{\pm 0.4}$ & 96.2$_{\pm 0.3}$
& 96.3$_{\pm 0.2}$ & 97.4$_{\pm 0.2}$
& 94.9$_{\pm 0.5}$ & 97.0$_{\pm 0.3}$
& 90.4$_{\pm 0.4}$ & 97.4$_{\pm 0.5}$ \\
\midrule
\method{}
& \textbf{77.7$_{\pm 1.2}$} & \textbf{94.9$_{\pm 0.8}$}
& \textbf{89.6$_{\pm 0.3}$} & \textbf{98.5$_{\pm 0.2}$}
& \textbf{95.8$_{\pm 0.3}$} & \textbf{97.7$_{\pm 0.3}$}
& \textbf{97.2$_{\pm 0.1}$} & \textbf{97.9$_{\pm 0.3}$}
& \textbf{96.0$_{\pm 0.4}$} & \textbf{98.4$_{\pm 0.2}$}
& \textbf{91.2$_{\pm 0.5}$} & \textbf{98.0$_{\pm 0.4}$} \\
\bottomrule
\end{tabular}
\vspace{-0.5cm}
\end{table*}

\subsection{Experimental Settings}\label{sec:experimental_settings}

\textbf{Datasets.} We evaluate \method{} on node classification, link prediction, and graph classification. Node classification uses 60\%/20\%/20\% splits for training, validation, and testing, and link prediction uses 80\%/5\%/15\%~\citep{wang2026adaptive, pei2020geom}, both on CiteSeer, PubMed~\citep{sen2008collective}, CS, Physics, Photo, and Computers~\citep{shchur2018pitfalls}. Graph classification uses 10-fold cross-validation on PROTEINS, Mutagenicity, NCI1, FRANKENSTEIN, BBBP, and ogbg-molhiv~\citep{hong2024label, wei2023neural}. We report accuracy for node and graph classification, and ROC-AUC for link prediction and for BBBP and ogbg-molhiv. More details are provided in Appendix~\ref{sec:dataset}. 

\noindent \textbf{Baselines.}  We compare \method{} with a comprehensive set of baselines, including: (1) general graph neural networks (GNNs), such as GCN~\citep{kipf2017semi}, GIN~\citep{xu2019how}, ML$^2$-GCL~\citep{liang2025ml}, AMPs~\citep{errica2025adaptive}, WaveGC~\citep{liu2025general}, SPARROW~\citep{lin2025simplified}, and G$^2$Former~\citep{zhangrestricted}; (2) manifold-based GNNs, including HGCN~\citep{chami2019hyperbolic}, D-GCN~\citep{sun2024motif}, and SPDGNN~\citep{wang2025enhancing}; and (3) adaptive GNNs, including ACE-HGNN~\citep{fu2021ace}, GNRF~\citep{chen2025graph}, BEC-GNN~\citep{hevapathige2025depth}, and ARGNN~\citep{wang2026adaptive}. More details of baselines can be found in Appendix~\ref{sec:baselines}.

\vspace{-0.1cm}
\subsection{Performance comparisons}
\vspace{-0.1cm}

\begin{wraptable}{r}{0.55\textwidth}
\vspace{-12pt}
\centering
\small
\caption{Performance comparisons (in \%) between baselines and \method{} for Graph Classification.}
\label{tab:graph}
\vspace{-0.2cm}
\setlength{\tabcolsep}{4pt}
\renewcommand{\arraystretch}{1.2}
\resizebox{0.55\textwidth}{!}{
\begin{tabular}{>{\centering\arraybackslash}m{1.7cm}|cccccc}
\toprule
Model
& PROTEINS & Mutag & NCI1 & FRANK & BBBP & molhiv \\
\midrule
GCN
& 75.3$_{\pm 1.9}$ & 79.8$_{\pm 1.8}$ & 76.0$_{\pm 1.0}$ & 63.3$_{\pm 2.2}$ & 87.4$_{\pm 2.0}$ & 75.8$_{\pm 2.0}$ \\
GIN
& 76.7$_{\pm 1.7}$ & 80.1$_{\pm 1.9}$ & 78.0$_{\pm 1.2}$ & 68.9$_{\pm 1.7}$ & 89.5$_{\pm 2.1}$ & 77.3$_{\pm 2.0}$ \\
ML$^2$-GCL
& 77.9$_{\pm 1.8}$ & 81.9$_{\pm 1.9}$ & 80.7$_{\pm 1.3}$ & 70.7$_{\pm 1.8}$ & 90.9$_{\pm 2.2}$ & 81.3$_{\pm 1.6}$ \\
AMPs
& 78.8$_{\pm 2.0}$ & 83.4$_{\pm 1.6}$ & 82.0$_{\pm 1.1}$ & 72.4$_{\pm 1.8}$ & 92.5$_{\pm 2.3}$ & 82.0$_{\pm 2.4}$ \\
WaveGC
& 79.0$_{\pm 1.8}$ & 82.3$_{\pm 2.3}$ & 81.6$_{\pm 1.8}$ & 72.6$_{\pm 1.9}$ & 92.6$_{\pm 1.8}$ & 82.0$_{\pm 1.9}$ \\
SPARROW
& 78.7$_{\pm 1.5}$ & 83.7$_{\pm 1.9}$ & 80.9$_{\pm 1.2}$ & 72.3$_{\pm 2.2}$ & 91.8$_{\pm 1.9}$ & 81.5$_{\pm 2.3}$ \\
G$^2$Former
& 78.9$_{\pm 2.0}$ & 82.7$_{\pm 2.4}$ & 82.1$_{\pm 1.5}$ & 71.9$_{\pm 1.9}$ & 92.4$_{\pm 2.4}$ & 81.8$_{\pm 1.6}$ \\
\midrule
HGCN
& 77.3$_{\pm 1.6}$ & 80.2$_{\pm 2.1}$ & 78.3$_{\pm 1.4}$ & 69.4$_{\pm 2.3}$ & 89.4$_{\pm 2.2}$ & 76.0$_{\pm 1.6}$ \\
D-GCN
& 78.4$_{\pm 2.2}$ & 82.2$_{\pm 1.8}$ & 79.1$_{\pm 1.3}$ & 69.8$_{\pm 2.2}$ & 90.7$_{\pm 1.7}$ & 79.0$_{\pm 2.2}$ \\
SPDGNN
& 78.5$_{\pm 1.9}$ & 81.6$_{\pm 1.7}$ & 79.6$_{\pm 1.6}$ & 71.4$_{\pm 2.1}$ & 91.5$_{\pm 1.7}$ & 81.3$_{\pm 2.2}$ \\
\midrule
ACE-HGNN
& 77.5$_{\pm 1.7}$ & 80.7$_{\pm 1.8}$ & 81.6$_{\pm 1.2}$ & 71.7$_{\pm 2.2}$ & 91.6$_{\pm 1.2}$ & 79.1$_{\pm 2.2}$ \\
BEC-GNN
& 77.5$_{\pm 1.9}$ & 81.0$_{\pm 1.5}$ & 79.4$_{\pm 2.2}$ & 70.0$_{\pm 2.3}$ & 91.9$_{\pm 1.8}$ & 80.3$_{\pm 1.3}$ \\
GNRF
& 78.5$_{\pm 1.7}$ & 82.4$_{\pm 2.3}$ & 81.9$_{\pm 2.1}$ & 72.1$_{\pm 1.7}$ & 92.2$_{\pm 1.8}$ & 81.5$_{\pm 1.7}$ \\
ARGNN
& 78.0$_{\pm 1.8}$ & 83.4$_{\pm 1.5}$ & 81.7$_{\pm 1.6}$ & 72.2$_{\pm 2.0}$ & 92.3$_{\pm 1.7}$ & 81.1$_{\pm 2.2}$ \\
\midrule
\method{}
& \textbf{79.4$_{\pm 1.6}$} & \textbf{84.9$_{\pm 1.7}$}
& \textbf{83.0$_{\pm 1.5}$} & \textbf{73.3$_{\pm 1.8}$}
& \textbf{93.4$_{\pm 1.8}$} & \textbf{82.9$_{\pm 1.1}$} \\
\bottomrule
\end{tabular}
}
\vspace{-8pt}
\end{wraptable}

\begin{figure*}[t]
    \centering
    \begin{subfigure}[t]{0.235\textwidth}
        \centering
        \raisebox{-0.09cm}{\includegraphics[width=0.95\linewidth,height=0.574\linewidth]{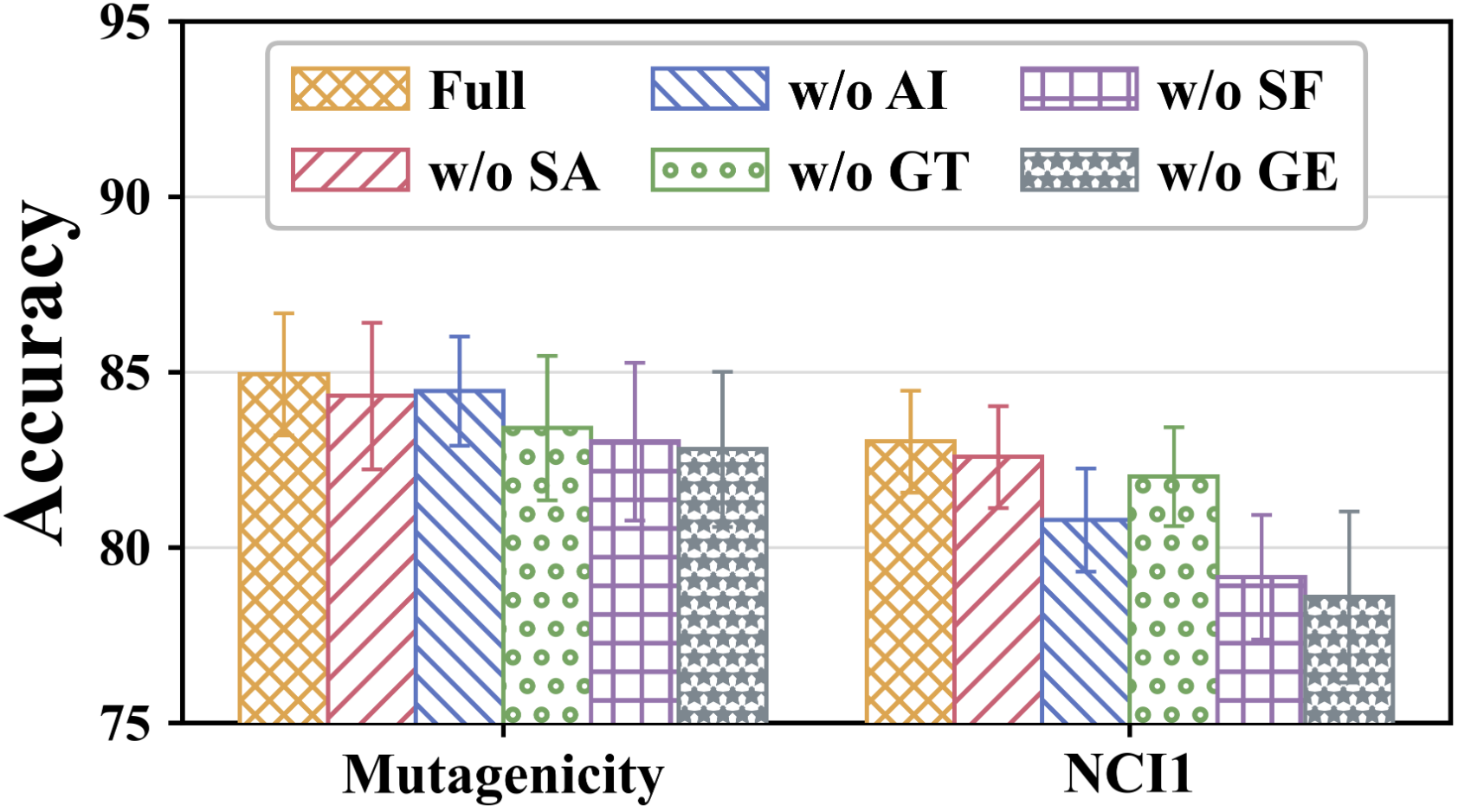}}
        \caption{Ablation Study}
        \label{fig:ablation_study}
    \end{subfigure}%
    \hspace{0.003\textwidth}%
    \begin{subfigure}[t]{0.235\textwidth}
        \centering
        \raisebox{-0.055cm}{\includegraphics[width=0.94\linewidth,height=0.5665\linewidth]{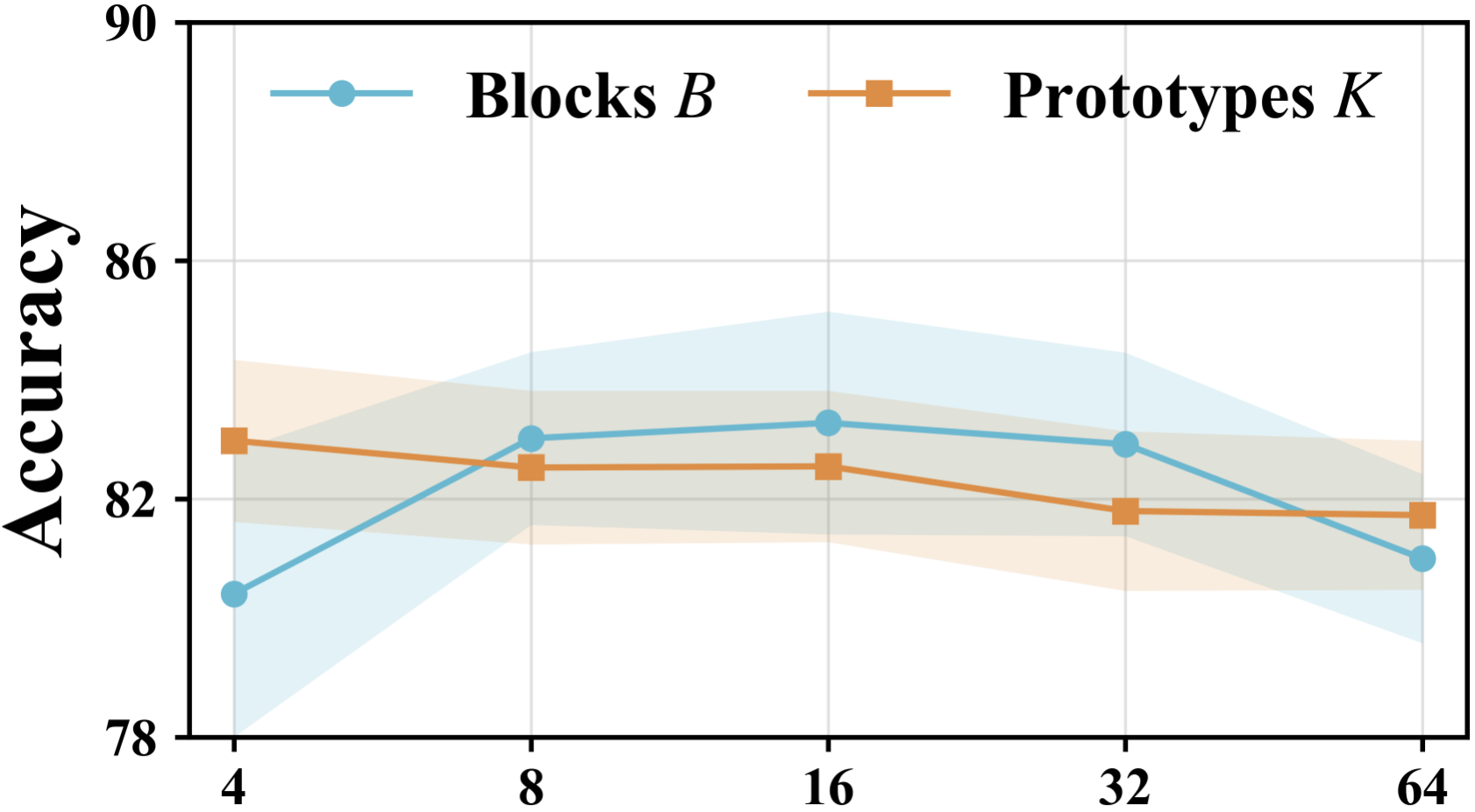}}
        \caption{Sensitivity Analysis}
        \label{fig:hyper_bk}
    \end{subfigure}%
    \hfill%
    \begin{subfigure}[t]{0.2585\textwidth}
        \centering
        \raisebox{-0.04cm}{\includegraphics[width=\linewidth,height=0.502\linewidth]{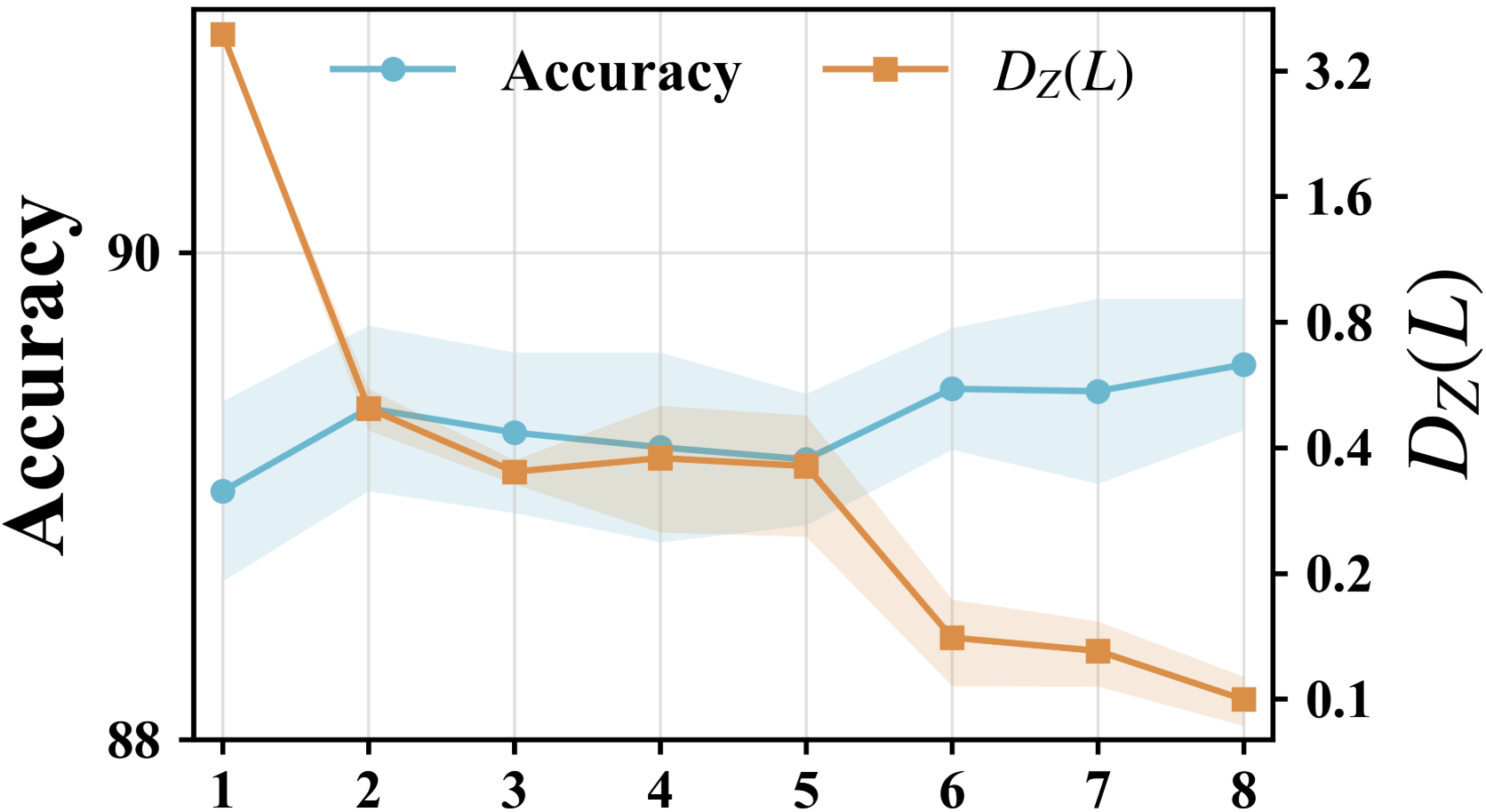}}
        \caption{Evolution Depth}
        \label{fig:evolution_depth}
    \end{subfigure}%
    \hfill%
    \begin{subfigure}[t]{0.2585\textwidth}
        \centering
        \raisebox{-0.039cm}{\includegraphics[width=0.99\linewidth,height=0.51\linewidth]{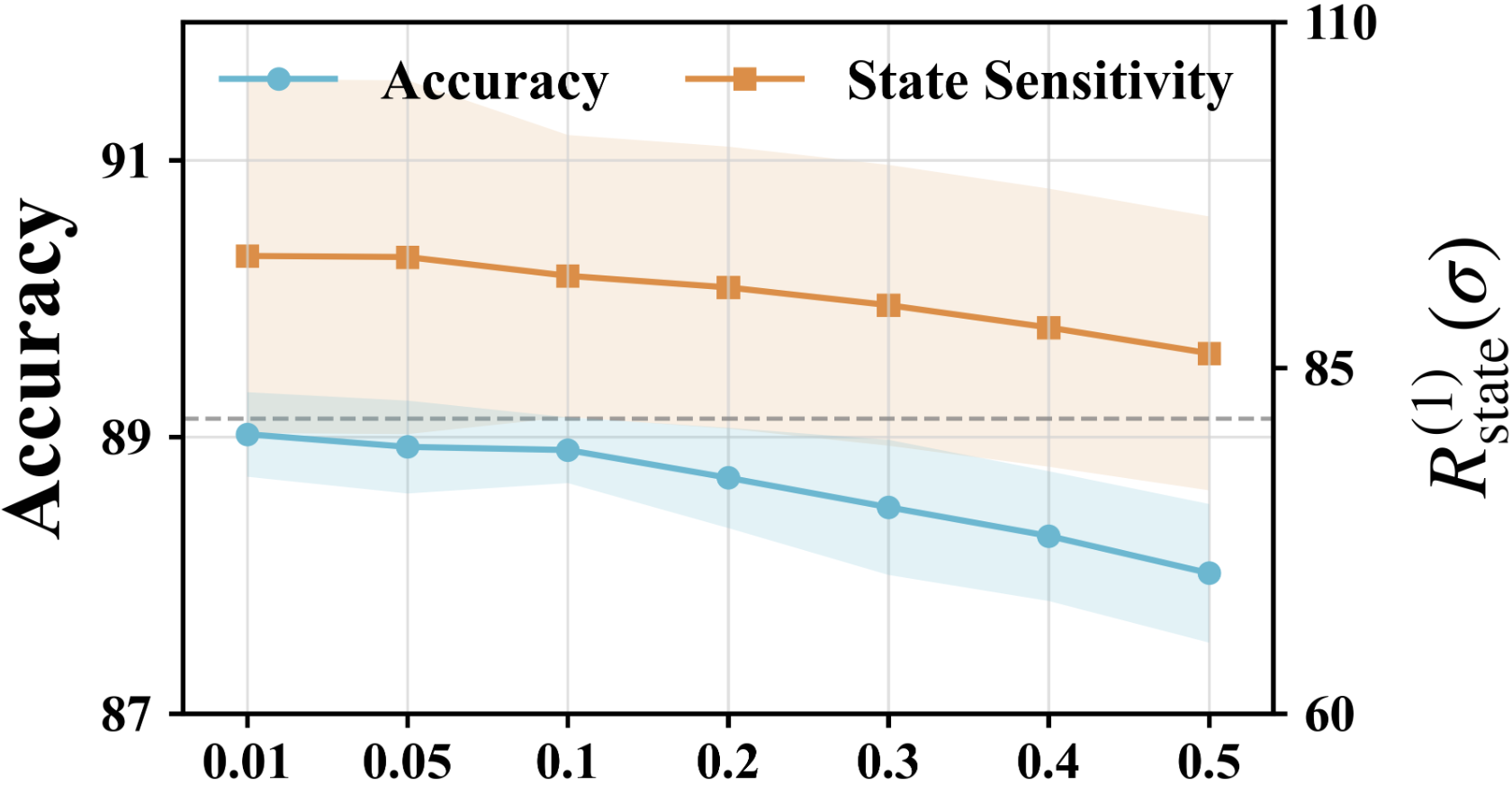}}
        \caption{Perturbation Stability}
        \label{fig:feature_stability}
    \end{subfigure}

    \vspace{-0.2cm}
    \caption{(a) Ablations on Mutagenicity and NCI1. (b) Sensitivity to block count $B$ and atlas size $K$ on NCI1. (c), (d) Evolution depth and perturbation stability on PubMed, respectively.}
    \label{fig:empirical_analysis}
    \vspace{-0.5cm}
\end{figure*}

Tables~\ref{tab:link_and_node} and~\ref{tab:graph} compare \method{} with the baselines on node classification (NC), link prediction (LP), and graph classification (GC). Three observations follow. (1) Advanced architectures such as AMPs and WaveGC improve clearly over GCN and GIN, which underlines the value of multi-scale propagation. (2) Manifold-based and adaptive GNNs are competitive through richer geometric priors and curvature-aware designs, but they treat geometry as a fixed manifold or a locally adaptive parameter rather than an evolving state, which limits how far the propagation space can follow task-specific requirements. (3) \method{} is best on every dataset and task. This matches its design. It keeps propagation geometry as a persistent node-wise state that governs neighborhood weighting and frame-aligned propagation, and it updates this state at every step from the previous geometry, a target built from second-order statistics of aligned residuals, and task-supervised corrections, so directional variation and within-block dependencies enter the geometry that governs subsequent propagation.

\vspace{-2pt}
\subsection{Ablation Study}
\label{sec:ablation}
\vspace{-2pt}

We evaluate five ablations on Mutagenicity and NCI1. Three remove a component: the
structural signatures in the prototype gate and the correction (w/o SA), the structure-aware
prototype atlas for initialization (w/o AI), and the triangular frame transport into the
target frame (w/o GT). Two alter the geometry update, either replacing the second-order
residual target by the current geometry while keeping the correction (w/o SF) or freezing
the initialized geometry across steps (w/o GE). Figure~\ref{fig:empirical_analysis}(a)
shows that the full model attains the highest accuracy on both datasets. Freezing the
geometry costs the most, confirming the benefit of adapting propagation geometry during
message passing, and removing second-order feedback costs the second most, so residual
statistics contribute beyond the correction. Disabling frame transport also lowers
accuracy, consistent with the role of coordinate alignment in aggregation and residual
computation. Removing the atlas hurts more on NCI1, and removing structural signatures
causes modest decreases. Additional results are in Appendix~\ref{sec:more_ablation}.

\vspace{-1pt}
\subsection{Sensitivity Analysis}
\label{sec:sensitivity}\label{sec:case_study}

\textbf{Hyperparameters Analysis.} 
We vary the number of geometric blocks $B$, atlas prototypes $K$, and evolution depth $L$,
which set the block partition, atlas size, and recurrent depth. On NCI1,
Figure~\ref{fig:empirical_analysis}(b) shows that accuracy improves as $B$ increases from
$4$ to $8$, stays stable for $B\in\{8,16,32\}$, and drops at $B=64$, while it changes little
for $K\in\{4,8,16\}$ and decreases slightly at $K=32$ and $64$; moderate values therefore
suffice. On PubMed, Figure~\ref{fig:empirical_analysis}(c) shows stable accuracy for $L$ from $1$ to $8$, while the terminal geometric change $D_Z(L)=\frac{1}{nB}\sum_{i,b}\|z_{i,b}^{L}-z_{i,b}^{L-1}\|_2$ decreases with depth, with a plateau between $L=2$ and $5$. Successive geometric updates thus shrink with depth without affecting accuracy, consistent with the stability of Proposition~\ref{prop:stability}. More results are provided in Appendix~\ref{sec:more_sensitivity}.

\textbf{Perturbation Stability.}
We perturb the initial features $\{\xi_i^0\}_{i=1}^{n}$ on PubMed along a random Gaussian direction with relative norm $\sigma$, keeping the graph, initial geometry, and trained parameters fixed, and report accuracy together with the one-step sensitivity $R_{\mathrm{state}}^{(1)}(\sigma)=d(S^{1}_\sigma,S^{1})/d(S^{0}_\sigma,S^{0})$, where $S^{l}=(\Xi^{l},Z^{l})$ stacks the feature states and geometric coordinates of all nodes and $d(S,S')=\|\Xi-\Xi'\|_F+\sum_{b}\|Z_b-Z'_b\|_F$. Figure~\ref{fig:empirical_analysis}(d) shows that $R_{\mathrm{state}}^{(1)}(\sigma)$ is nearly constant at small $\sigma$, indicating a proportional response to small perturbations, and decreases steadily by $7.6\%$ from $\sigma=0.01$ to $0.5$, while accuracy drops by only $1.1\%$ over the same range. Larger perturbations thus reduce accuracy mildly without amplifying the one-step sensitivity of the coupled update.

\subsection{Analysis of Second-Order Feedback}
\label{sec:second_order_control}

\begin{wraptable}{r}{0.45\textwidth}
\vspace{-12pt}
\centering
\small
\caption{Performance comparison (in \%) between baselines and their second-order (SO) variants. \textbf{Bold} indicates the best performance.}
\label{tab:second_order_control}
\vspace{-0.15cm}

\setlength{\tabcolsep}{2.3pt}
\renewcommand{\arraystretch}{1.15}

\resizebox{0.45\textwidth}{!}{
\begin{tabular}{
>{\centering\arraybackslash}m{2.15cm}
>{\centering\arraybackslash}m{1.35cm}
>{\centering\arraybackslash}m{1.35cm}
>{\centering\arraybackslash}m{1.45cm}
>{\centering\arraybackslash}m{1.35cm}
}
\toprule
Model
& CiteSeer
& Photo
& PROTEINS
& BBBP \\
\midrule

GCN
& 71.5$_{\pm 1.7}$
& 92.1$_{\pm 0.5}$
& 75.3$_{\pm 1.9}$
& 87.4$_{\pm 2.0}$ \\

GCN w/ SO
& 72.3$_{\pm 1.9}$
& 92.4$_{\pm 1.1}$
& 77.0$_{\pm 2.4}$
& 86.7$_{\pm 2.5}$ \\

\midrule

GIN
& 70.6$_{\pm 1.2}$
& 92.7$_{\pm 0.5}$
& 76.7$_{\pm 1.7}$
& 89.5$_{\pm 2.1}$ \\

GIN w/ SO
& 73.3$_{\pm 1.0}$
& 93.6$_{\pm 1.4}$
& 78.5$_{\pm 1.5}$
& 91.3$_{\pm 2.2}$ \\

\midrule

ARGNN
& 75.6$_{\pm 1.2}$
& 94.9$_{\pm 0.5}$
& 78.0$_{\pm 1.8}$
& 92.3$_{\pm 1.7}$ \\

ARGNN w/ SO
& 74.4$_{\pm 2.3}$
& 95.3$_{\pm 0.6}$
& 78.3$_{\pm 2.2}$
& 91.9$_{\pm 2.1}$ \\

\midrule

\method{}
& \textbf{77.7$_{\pm 1.2}$}
& \textbf{96.0$_{\pm 0.4}$}
& \textbf{79.4$_{\pm 1.6}$}
& \textbf{93.4$_{\pm 1.8}$} \\

\bottomrule
\end{tabular}
}
\vspace{-0.4cm}
\end{wraptable}
\textbf{Features versus Feedback.}
To test whether the gains come from second-order information alone, we give GCN, GIN, and ARGNN the same weighted second-order residual statistics as additional features, with parameter counts matched to \method{} within $0.5\%$ (Table~\ref{tab:second_order_control}). The augmentation helps GIN in all four settings and GCN in three, but hurts GCN on BBBP and ARGNN on CiteSeer and BBBP. \method{} exceeds the strongest augmented variant in every column, by $0.7$ to $3.3$ points. At comparable parameter budgets, the difference is how the statistics are used. The baselines consume them as features, whereas \method{} uses them as feedback that changes the geometry governing subsequent weighting and transport.

\begin{wraptable}{r}{0.45\textwidth}
\vspace{-12pt}
\centering
\small
\caption{Performance comparison (in \%) among different residual feedback statistics. \textbf{Bold} indicates the best performance.}
\label{tab:feedback_statistics}
\vspace{-0.15cm}

\setlength{\tabcolsep}{2.3pt}
\renewcommand{\arraystretch}{1.15}

\resizebox{0.45\textwidth}{!}{
\begin{tabular}{
>{\centering\arraybackslash}m{2.0cm}
>{\centering\arraybackslash}m{1.25cm}
>{\centering\arraybackslash}m{1.25cm}
>{\centering\arraybackslash}m{1.45cm}
>{\centering\arraybackslash}m{1.25cm}
}
\toprule
Feedback & Photo & CS & Mutag. & NCI1 \\
\midrule
Mean-only & 94.9$_{\pm 0.6}$ & 95.1$_{\pm 0.3}$ & 83.7$_{\pm 1.7}$ & 81.9$_{\pm 1.6}$ \\
Diagonal  & 95.1$_{\pm 0.3}$ & 95.2$_{\pm 0.4}$ & 83.2$_{\pm 2.0}$ & 80.4$_{\pm 1.0}$ \\
\midrule
One-shot  & 94.6$_{\pm 0.5}$ & 94.8$_{\pm 0.3}$ & 83.1$_{\pm 1.8}$ & 79.8$_{\pm 2.7}$ \\
\midrule
Full      & \textbf{96.0$_{\pm 0.4}$} & \textbf{95.8$_{\pm 0.3}$} & \textbf{84.9$_{\pm 1.7}$} & \textbf{83.0$_{\pm 1.5}$} \\
\bottomrule
\end{tabular}
}
\vspace{-0.5cm}
\end{wraptable}
\textbf{Feedback Structure and Schedule.}
Table~\ref{tab:feedback_statistics} varies the feedback with the rest of the model unchanged. Replacing the full statistic $C_{i,b}^l$ by the mean-only target $C_{i,b}^{l,\mathrm{mean}}$ or the diagonal target $C_{i,b}^{l,\mathrm{diag}}$ of Proposition~\ref{prop:separation} lowers accuracy on all four datasets, consistent with Proposition~\ref{prop:separation}; the diagonal target loses most on NCI1 and Mutag. The two restricted variants rank differently across datasets, so the mean and the marginal second moments capture different aspects of local variation. 
Updating the geometry only once, after the first step, also lowers accuracy on all four datasets, by $1.4$ and $1.0$ points on Photo and CS and by $1.8$ and $3.2$ points on Mutag and NCI1.
\section{Conclusion}
\label{sec:conclusion}


To capture message-level variation obscured by aggregation, we proposed \method{}, which
jointly evolves node features and propagation geometry through recurrent message-passing
feedback. The geometry governs neighborhood weighting and triangular frame transport for
consistent message comparison, second-order residual statistics capture directional
variation and within-block dependencies to define geometric targets, and bounded
log-triangular updates combine these targets with task-supervised corrections while
preserving positive definiteness. Experiments on node, link, and graph prediction show
consistent gains over competitive baselines, supporting learning propagation geometry from
the interactions it induces and motivating future work on adaptive-depth inference and
evolving graphs.

\bibliography{reference}
\bibliographystyle{iclr2027_conference}

\appendix
\onecolumn
\appendix

\section{Notation Summary}

As shown in the Table~\ref{tab:notation}, we summarize the key notations of this paper.

\begin{table}[h]
    \centering
    \caption{Summary of key notations.}
    \label{tab:notation}
    \small
    \setlength{\tabcolsep}{6pt}
    \renewcommand{\arraystretch}{1.10}
    \begin{tabularx}{0.98\linewidth}{@{}lX@{}}
        \toprule
        \textbf{Symbol} & \textbf{Description} \\
        \midrule

        $G=(V,E,X),A$
        & Attributed graph with node set $V$, edge set $E$, feature matrix $X$, and adjacency matrix $A$. \\

        $n,F_0,C$
        & Numbers of nodes, input features, and classes, respectively. \\

        $u_i,U,q$
        & Fixed structural signature of node $i$, signature matrix, and signature dimension. \\

        $\mathcal{N}(i),\widetilde{\mathcal{N}}(i)$
        & Neighborhood of node $i$ and its extension with a self-loop. \\

        $d,B,m$
        & Hidden dimension, number of geometric blocks, and block size, with $d=Bm$. \\

        $K,L$
        & Numbers of geometric prototypes and recurrent evolution steps. \\

        \midrule

        $\xi_i^l,Z_i^l$
        & Local feature state and node-wise geometric coordinates at step $l$. \\

        $z_{i,b}^l=[a_{i,b}^l,\ell_{i,b}^l]^\top$
        & Log-triangular coordinates of block $b$: strictly lower-triangular entries and log-diagonal scales. \\

        $L_{i,b}^l,g_{i,b}^l$
        & Block triangular frame and corresponding symmetric positive-definite geometry. \\

        $L_i^l,g_i^l$
        & Node-wise frame and local geometry assembled as block-diagonal matrices. \\

        $\mathcal{Z},\Pi_{\mathcal{Z}}$
        & Admissible coordinate domain with bounds $a_{\max},\ell_{\min},\ell_{\max}$, and coordinatewise clipping onto it. \\

        $\sigma_+,\sigma_-,\kappa_L$
        & Spectral bounds of admissible frames and Lipschitz constant of the frame reconstruction. \\

        $z_{k,b}^{\mathrm{proto}},\alpha_{ik}$
        & Coordinates of prototype $k$ in block $b$ and its structure-conditioned weight for node $i$. \\

        \midrule

        $h_i^l$
        & Shared-environment representation $h_i^l=L_i^l\xi_i^l$. \\

        $e_{ij}^l,\omega_{ij}^l,\tau$
        & Squared environment-space distance, normalized neighborhood weight, and temperature. \\

        $\phi,W_\phi$
        & Shared linear message transformation and its learnable weight matrix. \\

        $T_{j\rightarrow i,b}^l$
        & Blockwise source-to-target frame map $(L_{i,b}^l)^{-1}L_{j,b}^l$. \\

        $m_{j\rightarrow i,b}^l,m_{j\rightarrow i}^l$
        & Transported block message and the full message concatenated across blocks. \\

        $\bar{\xi}_i^l,r_i^l$
        & Aligned neighborhood aggregate and elementwise gate for feature updates. \\

        \midrule

        $\delta_{ij,b}^l,\bar{\delta}_{i,b}^l$
        & Residual between an aligned source message and the transformed target state, and its weighted mean. \\

        $C_{i,b}^l,\epsilon_s$
        & Regularized weighted second-moment matrix of residuals and its spectral regularization. \\

        $R_{i,b}^l,\widehat z_{i,b}^l,\eta$
        & Stabilized Cholesky target frame, its log-triangular coordinates, and the numerical jitter. \\

        \midrule

        $q_{i,b}^l,\Gamma_\theta$
        & Controller input and shared controller predicting task-supervised geometric corrections. \\

        $U_{i,b}^l,V_{i,b}^l,d_{i,b}^l$
        & Controller factors for lower-triangular corrections and the log-diagonal correction vector. \\

        $\Delta z_{i,b}^l$
        & Learned correction to the block log-triangular coordinates. \\

        $\lambda,\gamma$
        & Mixing weight for the residual target and strength of the learned correction. \\

        \bottomrule
    \end{tabularx}
\end{table}

\setcounter{lemma}{0}
\setcounter{theorem}{0}
\section{Proof of Lemma~\ref{lem:wellposed}}
\label{proof_1}

\begin{lemma}[Well-Posed Geometric Parameterization]
Let $m\ge 1$ and $z,z'\in\mathcal Z$. The frame $L(z)$ is lower triangular with
$\det L(z)=\exp(\sum_{t}\ell_t)>0$, so it is invertible and $g(z)\in\mathrm{SPD}(m)$. There are
constants $\sigma_+\ge\sigma_->0$ and $\kappa_L>0$, depending only on $m$, $a_{\max}$,
$\ell_{\min}$, and $\ell_{\max}$, such that
\[
\|L(z)\|_2\le\sigma_+,\qquad \|L(z)^{-1}\|_2\le\sigma_-^{-1},\qquad
\sigma_-^{2}\,I_m\ \preceq\ g(z)\ \preceq\ \sigma_+^{2}\,I_m ,
\]
and the reconstruction is Lipschitz: $\|L(z)-L(z')\|_F\le\kappa_L\|z-z'\|_2$ and
$\|L(z)^{-1}-L(z')^{-1}\|_2\le\sigma_-^{-2}\kappa_L\|z-z'\|_2$. All statements carry over to
the block-diagonal frames $L_i=\bigoplus_{b}L(z_{i,b})$ and $g_i=L_i^{\top}L_i$, with
$\|z-z'\|_2$ replaced by $\|Z_i-Z_i'\|_F:=(\sum_{b}\|z_{i,b}-z'_{i,b}\|_2^{2})^{1/2}$.
\end{lemma}

\begin{proof}
Throughout, $\|\cdot\|_2$ denotes the Euclidean norm of a vector and the spectral norm of a
matrix, $\|\cdot\|_F$ the Frobenius norm, and $\|\cdot\|_1$, $\|\cdot\|_\infty$ the induced
max-column-sum and max-row-sum norms. Recall from Eq.~\eqref{eq:block_geometry} that a block
coordinate $z=[a^{\top},\ell^{\top}]^{\top}\in\mathcal Z$, with $a\in\mathbb{R}^{m(m-1)/2}$ and
$\ell\in\mathbb{R}^{m}$, is reconstructed as
\begin{equation}\label{eq:frame}
L(z)=N+D,\qquad N:=\operatorname{mat}_{\mathrm{sl}}(a),\qquad D:=\mathrm{Diag}(e^{\ell}),
\end{equation}
so that $L(z)_{rr}=e^{\ell_r}$ and $L(z)_{rs}=a_{rs}$ for $r>s$, where $N$ is strictly lower
triangular with $|N_{rs}|\le a_{\max}$ and $\ell_t\in[\ell_{\min},\ell_{\max}]$ for all
$r>s$ and $t$. With
\begin{equation}\label{eq:aux_constants}
\nu:=a_{\max}\sqrt{\tfrac{m(m-1)}{2}},\qquad
\bar a:=a_{\max}e^{-\ell_{\min}},\qquad
\kappa_0:=\min\Big\{\sqrt m\,(1+\bar a)^{m-1},\ \sum_{k=0}^{m-1}\big(\nu e^{-\ell_{\min}}\big)^{k}\Big\},
\end{equation}
we set
\begin{equation}\label{eq:constants}
\sigma_+:=e^{\ell_{\max}}+\nu,\qquad
\sigma_-:=e^{\ell_{\min}}/\kappa_0,\qquad
\kappa_L:=1+e^{\ell_{\max}}.
\end{equation}
These constants depend only on $m$, $a_{\max}$, $\ell_{\min}$, and $\ell_{\max}$. Since
$\kappa_0\ge1$ and $e^{\ell_{\max}}\ge e^{\ell_{\min}}$, we have $\sigma_+\ge\sigma_->0$.

\smallskip\noindent
\emph{Invertibility.} $L(z)$ is lower triangular with diagonal entries
$e^{\ell_t}\in[e^{\ell_{\min}},e^{\ell_{\max}}]$, so
$\det L(z)=\prod_te^{\ell_t}=\exp(\sum_t\ell_t)>0$ and $L(z)$ is invertible. The matrix
$g(z)=L(z)^{\top}L(z)$ is symmetric and, for $v\ne0$, $v^{\top}g(z)v=\|L(z)v\|_2^{2}>0$;
hence $g(z)\in\mathrm{SPD}(m)$.

\smallskip\noindent
\emph{Upper bound.} Since $\|D\|_2=\max_te^{\ell_t}\le e^{\ell_{\max}}$ and $N$ has
$m(m-1)/2$ entries of magnitude at most $a_{\max}$,
\begin{equation}
\|L(z)\|_2\le\|D\|_2+\|N\|_2\le e^{\ell_{\max}}+\|N\|_F
\le e^{\ell_{\max}}+a_{\max}\sqrt{\tfrac{m(m-1)}{2}}=\sigma_+ .
\end{equation}

\smallskip\noindent
\emph{Lower bound.} Factor the diagonal out of Eq.~\eqref{eq:frame}: $L(z)=D(I_m+\tilde N)$
with $\tilde N:=D^{-1}N$, so that $L(z)^{-1}=(I_m+\tilde N)^{-1}D^{-1}$ and
$\|D^{-1}\|_2=\max_te^{-\ell_t}\le e^{-\ell_{\min}}$. Left multiplication by $D^{-1}$ scales
row $r$ by $e^{-\ell_r}$, so $\tilde N$ is strictly lower triangular with
$|\tilde N_{rs}|=e^{-\ell_r}|a_{rs}|\le\bar a$ and
$\|\tilde N\|_F\le e^{-\ell_{\min}}\|N\|_F\le\nu e^{-\ell_{\min}}$. It suffices to show
$\|(I_m+\tilde N)^{-1}\|_2\le\kappa_0$. We bound this norm in two ways and take the smaller.

(i) \emph{Forward substitution.} For $y\in\mathbb{R}^{m}$ let $w=(I_m+\tilde N)^{-1}y$, i.e.
$w_r=y_r-\sum_{s<r}\tilde N_{rs}w_s$ for $r=1,\dots,m$. We claim
$|w_r|\le\|y\|_\infty(1+\bar a)^{r-1}$. This holds for $r=1$ since $w_1=y_1$. If it
holds for all $s<r$, then
\[
|w_r|\le\|y\|_\infty+\bar a\sum_{s<r}|w_s|
\le\|y\|_\infty\Big(1+\bar a\sum_{s=1}^{r-1}(1+\bar a)^{s-1}\Big)
=\|y\|_\infty(1+\bar a)^{r-1},
\]
where the last equality uses the geometric sum
$\bar a\sum_{s=1}^{r-1}(1+\bar a)^{s-1}=(1+\bar a)^{r-1}-1$.
Hence $\|(I_m+\tilde N)^{-1}\|_\infty\le(1+\bar a)^{m-1}$. For any
$A\in\mathbb{R}^{m\times m}$, $\|A\|_1\le m\|A\|_\infty$ and
$\|A\|_2\le\sqrt{\|A\|_1\|A\|_\infty}$, so
$\|(I_m+\tilde N)^{-1}\|_2\le\sqrt{m}\,(1+\bar a)^{m-1}$.

(ii) \emph{Neumann series.} Since $\tilde N$ is strictly lower triangular, $\tilde N^{m}=0$
and $(I_m+\tilde N)^{-1}=\sum_{k=0}^{m-1}(-\tilde N)^{k}$. With
$\|\tilde N\|_2\le\|\tilde N\|_F\le\nu e^{-\ell_{\min}}$,
\[
\|(I_m+\tilde N)^{-1}\|_2\le\sum_{k=0}^{m-1}\|\tilde N\|_2^{k}
\le\sum_{k=0}^{m-1}\big(\nu e^{-\ell_{\min}}\big)^{k}.
\]
Combining (i) and (ii) gives $\|(I_m+\tilde N)^{-1}\|_2\le\kappa_0$, hence
$\|L(z)^{-1}\|_2\le\|(I_m+\tilde N)^{-1}\|_2\|D^{-1}\|_2\le\kappa_0e^{-\ell_{\min}}=\sigma_-^{-1}$.

\smallskip\noindent
\emph{Metric bounds.} The eigenvalues of $g(z)=L(z)^{\top}L(z)$ are the squared singular
values of $L(z)$. Since $\sigma_{\max}(L(z))=\|L(z)\|_2\le\sigma_+$ and
$\sigma_{\min}(L(z))=\|L(z)^{-1}\|_2^{-1}\ge\sigma_-$, we obtain
$\sigma_-^{2}I_m\preceq g(z)\preceq\sigma_+^{2}I_m$.

\smallskip\noindent
\emph{Lipschitz reconstruction.} Write $z'=[a'^{\top},\ell'^{\top}]^{\top}$,
$D'=\mathrm{Diag}(e^{\ell'})$ and $N'=\operatorname{mat}_{\mathrm{sl}}(a')$. By
Eq.~\eqref{eq:frame},
\[
L(z)-L(z')=(N-N')+(D-D').
\]
Since $\operatorname{mat}_{\mathrm{sl}}$ is an isometry onto the strictly lower entries,
$\|N-N'\|_F=\|a-a'\|_2$. By the mean value theorem on $[\ell_{\min},\ell_{\max}]$,
$|e^{\ell_t}-e^{\ell'_t}|\le e^{\ell_{\max}}|\ell_t-\ell'_t|$ for every $t$, so
$\|D-D'\|_F=\big(\sum_t(e^{\ell_t}-e^{\ell'_t})^{2}\big)^{1/2}\le e^{\ell_{\max}}\|\ell-\ell'\|_2$.
Using the triangle inequality and $\|a-a'\|_2,\ \|\ell-\ell'\|_2\le\|z-z'\|_2$,
\[
\|L(z)-L(z')\|_F
\le\|a-a'\|_2+e^{\ell_{\max}}\|\ell-\ell'\|_2
\le\big(1+e^{\ell_{\max}}\big)\|z-z'\|_2
=\kappa_L\|z-z'\|_2 .
\]
For the inverses, $L(z)^{-1}-L(z')^{-1}=L(z)^{-1}\big(L(z')-L(z)\big)L(z')^{-1}$, so by the
spectral bounds and $\|\cdot\|_2\le\|\cdot\|_F$,
\[
\|L(z)^{-1}-L(z')^{-1}\|_2
\le\|L(z)^{-1}\|_2\,\|L(z')-L(z)\|_F\,\|L(z')^{-1}\|_2
\le\sigma_-^{-2}\kappa_L\|z-z'\|_2 .
\]

\smallskip\noindent
\emph{Block-diagonal frames.} $L_i$ is block diagonal with lower-triangular blocks
$L(z_{i,b})$, hence lower triangular with $\det L_i=\prod_b\det L(z_{i,b})>0$, and
$g_i=\bigoplus_bg(z_{i,b})$ is SPD. The singular values of a block-diagonal matrix are the
union of those of its blocks, so the spectral bounds hold for $L_i$ and $g_i$ with the same
$\sigma_\pm$. For the Lipschitz bounds,
\[
\|L_i-L_i'\|_F^{2}=\sum_b\|L(z_{i,b})-L(z'_{i,b})\|_F^{2}
\le\kappa_L^{2}\sum_b\|z_{i,b}-z'_{i,b}\|_2^{2}=\kappa_L^{2}\|Z_i-Z_i'\|_F^{2},
\]
and, since the spectral norm of a block-diagonal matrix is the maximum over its blocks,
\[
\|L_i^{-1}-L_i'^{-1}\|_2=\max_b\|L(z_{i,b})^{-1}-L(z'_{i,b})^{-1}\|_2
\le\sigma_-^{-2}\kappa_L\max_b\|z_{i,b}-z'_{i,b}\|_2
\le\sigma_-^{-2}\kappa_L\|Z_i-Z_i'\|_F .
\]
\end{proof}

\section{Proof of Proposition~\ref{prop:gauge}}
\label{app:proof_gauge}

\begin{proposition}[Frame-Aligned Transport]
Fix a recurrent step $l$ with admissible geometric states, and let
$T^{l}_{j\to i,b}=(L^{l}_{i,b})^{-1}L^{l}_{j,b}$ and $M^{l}_j=L^{l}_jW_\phi(L^{l}_j)^{-1}$.
Then $m^{l}_{j\to i,b}=T^{l}_{j\to i,b}(\phi(\xi^{l}_j))_b$ with
$\|T^{l}_{j\to i,b}\|_2\le\sigma_+/\sigma_-$, and the frame maps satisfy
$T^{l}_{i\to i,b}=I_m$ and $T^{l}_{j\to i,b}T^{l}_{k\to j,b}=T^{l}_{k\to i,b}$, so their
product along any walk depends only on its endpoints. In environment coordinates, the
aggregate and the residuals $\delta^{l}_{ij}:=m^{l}_{j\to i}-\phi(\xi^{l}_i)$ read
\[
L^{l}_i\bar\xi^{l}_i=\sum\nolimits_{j\in\tilde{\mathcal N}(i)}\omega^{l}_{ij}\,M^{l}_jh^{l}_j,
\qquad
L^{l}_i\delta^{l}_{ij}=M^{l}_jh^{l}_j-M^{l}_ih^{l}_i,
\]
where each $M^{l}_j$ is similar to $W_\phi$ with
$\|M^{l}_j\|_2\le(\sigma_+/\sigma_-)\|W_\phi\|_2$, and $M^{l}_j=cI_d$ when $W_\phi=cI_d$.
\end{proposition}

\begin{proof}
Throughout, $\|\cdot\|_2$ denotes the Euclidean norm of a vector and the spectral norm of a
matrix. The geometric states are admissible, so Lemma~\ref{lem:wellposed} applies to every
block frame $L^{l}_{i,b}$ and to every block-diagonal frame $L^{l}_i$.

\smallskip\noindent
\emph{Weights.} For finite $\xi^{l}$ and admissible $Z^{l}$, $h^{l}_i=L^{l}_i\xi^{l}_i$ and
$e^{l}_{ij}=\|h^{l}_i-h^{l}_j\|_2^{2}$ are finite, so every term
$\exp(-e^{l}_{ij}/\tau)$ in Eq.~\eqref{eq:weights} is strictly positive and the normalizer
is finite. Hence $\omega^{l}_{ij}>0$ and $\sum_{j\in\tilde{\mathcal N}(i)}\omega^{l}_{ij}=1$.

\smallskip\noindent
\emph{Transport.} By Lemma~\ref{lem:wellposed}, $L^{l}_{i,b}$ is invertible, so the
triangular solve in Eq.~\eqref{eq:cholesky_frame_transport} returns the unique
$m^{l}_{j\to i,b}$ with $L^{l}_{i,b}m^{l}_{j\to i,b}=L^{l}_{j,b}(\phi(\xi^{l}_j))_b$, i.e.
$m^{l}_{j\to i,b}=(L^{l}_{i,b})^{-1}L^{l}_{j,b}(\phi(\xi^{l}_j))_b=T^{l}_{j\to i,b}(\phi(\xi^{l}_j))_b$.
Submultiplicativity and the spectral bounds of Lemma~\ref{lem:wellposed} give
$\|T^{l}_{j\to i,b}\|_2\le\|(L^{l}_{i,b})^{-1}\|_2\,\|L^{l}_{j,b}\|_2\le\sigma_-^{-1}\sigma_+$.

\smallskip\noindent
\emph{Composition.} Directly from the definition,
$T^{l}_{i\to i,b}=(L^{l}_{i,b})^{-1}L^{l}_{i,b}=I_m$ and
\[
T^{l}_{j\to i,b}\,T^{l}_{k\to j,b}
=(L^{l}_{i,b})^{-1}L^{l}_{j,b}(L^{l}_{j,b})^{-1}L^{l}_{k,b}
=(L^{l}_{i,b})^{-1}L^{l}_{k,b}=T^{l}_{k\to i,b}.
\]
For a walk $i_0\to i_1\to\cdots\to i_k$, induction on $k$ with this rule gives
$T^{l}_{i_{k-1}\to i_k,b}\cdots T^{l}_{i_0\to i_1,b}=T^{l}_{i_0\to i_k,b}=(L^{l}_{i_k,b})^{-1}L^{l}_{i_0,b}$,
which involves only the frames of the two endpoints; for a closed walk it equals $I_m$.

\smallskip\noindent
\emph{Environment-coordinate form.} Taking the $b$-th block of
Eq.~\eqref{eq:aligned_aggregation}, multiplying by $L^{l}_{i,b}$ and using linearity,
\[
L^{l}_{i,b}(\bar\xi^{l}_i)_b
=\sum_{j\in\tilde{\mathcal N}(i)}\omega^{l}_{ij}L^{l}_{i,b}m^{l}_{j\to i,b}
=\sum_{j\in\tilde{\mathcal N}(i)}\omega^{l}_{ij}L^{l}_{j,b}\big(\phi(\xi^{l}_j)\big)_b .
\]
All nodes share the block partition and $L^{l}_i$ is block diagonal, so stacking this
identity over $b=1,\dots,B$ yields
$L^{l}_i\bar\xi^{l}_i=\sum_j\omega^{l}_{ij}L^{l}_jW_\phi\xi^{l}_j$. Since
$\xi^{l}_j=(L^{l}_j)^{-1}h^{l}_j$, each summand equals
$\omega^{l}_{ij}L^{l}_jW_\phi(L^{l}_j)^{-1}h^{l}_j=\omega^{l}_{ij}M^{l}_jh^{l}_j$, which is
the stated form. By construction $M^{l}_j=L^{l}_jW_\phi(L^{l}_j)^{-1}$ is similar to
$W_\phi$, and
$\|M^{l}_j\|_2\le\|L^{l}_j\|_2\,\|W_\phi\|_2\,\|(L^{l}_j)^{-1}\|_2\le(\sigma_+/\sigma_-)\|W_\phi\|_2$
by Lemma~\ref{lem:wellposed}.

\smallskip\noindent
\emph{Residuals in environment coordinates.} For $j\in\tilde{\mathcal N}(i)$, stacking the
blocks of Eq.~\eqref{eq:cholesky_frame_transport} as above gives
$L^{l}_im^{l}_{j\to i}=L^{l}_jW_\phi\xi^{l}_j$, while
$L^{l}_i\phi(\xi^{l}_i)=L^{l}_iW_\phi(L^{l}_i)^{-1}h^{l}_i=M^{l}_ih^{l}_i$. Hence
\[
L^{l}_i\,\delta^{l}_{ij}
=L^{l}_im^{l}_{j\to i}-L^{l}_i\phi(\xi^{l}_i)
=L^{l}_jW_\phi(L^{l}_j)^{-1}h^{l}_j-M^{l}_ih^{l}_i
=M^{l}_jh^{l}_j-M^{l}_ih^{l}_i .
\]

\smallskip\noindent
\emph{Scalar transform.} If $W_\phi=cI_d$, then
$M^{l}_j=cL^{l}_j(L^{l}_j)^{-1}=cI_d$.
\end{proof}

\section{Proof of Proposition~\ref{prop:separation}}
\label{app:proof_separation}

\begin{proposition}[Second-Order Refinement]
Fix a step $l$. Write $\mathcal R^{l}_{i,b}=\{(\omega^{l}_{ij},\delta^{l}_{ij,b})\}_{j\in\tilde{\mathcal N}(i)}$
for the residual configuration of node $i$ in block $b$, with weighted mean
$\bar\delta^{l}_{i,b}$, and call an update \emph{first-order} if it depends on
$\mathcal R^{l}_{i,b}$ only through the weights and $\bar\delta^{l}_{i,b}$. Let $i,i'$ be nodes
with $\xi^{l}_i=\xi^{l}_{i'}$, $Z^{l}_i=Z^{l}_{i'}$, $u_i=u_{i'}$, and
$\bar\delta^{l}_{i,b}=\bar\delta^{l}_{i',b}$ for all $b$.
\begin{enumerate}[label=(\roman*),leftmargin=1.6em,itemsep=1pt,topsep=2pt]
\item Every first-order update, including the mean-only target
$C^{l,\mathrm{mean}}_{i,b}=\bar\delta^{l}_{i,b}(\bar\delta^{l}_{i,b})^{\top}+\epsilon_sI_m$,
the controller of Eq.~\eqref{eq:correction}, and the feature update of
Eq.~\eqref{eq:feature_gate}, returns identical outputs for $i$ and $i'$; in particular
$\xi^{l+1}_i=\xi^{l+1}_{i'}$ and $\Delta z^{l}_{i,b}=\Delta z^{l}_{i',b}$.
\item The second-order target is injective in the statistic:
$C^{l}_{i,b}\ne C^{l}_{i',b}$ implies $\hat z^{l}_{i,b}\ne\hat z^{l}_{i',b}$. If
$\lambda\in(0,1]$ and Eq.~\eqref{eq:task_guided_geometry_update} is not clipped in block
$b$, then $Z^{l+1}_i\ne Z^{l+1}_{i'}$, and $h^{l+1}_i\ne h^{l+1}_{i'}$ unless
$\xi^{l+1}_i\in\ker(L^{l+1}_i-L^{l+1}_{i'})$, a proper subspace.
\item If $i$ has at least two neighbors and $\mathcal R^{l}_{i',b}$ ranges over all
configurations with the same weights and first moment, those with
$C^{l}_{i',b}=C^{l}_{i,b}$ form a Lebesgue-null set.
\item The target depends on $\mathcal R^{l}_{i,b}$ only through $C^{l}_{i,b}$ and is
therefore strictly weaker than an injective aggregator: for $m=1$ there are distinct
three-point residual sets with equal weights, first, and second moments. Conversely, for
$m\ge2$ the sets $\{0,\pm v\}$ and $\{0,\pm w\}$ with $v=(1,1,0,\dots,0)^{\top}$,
$w=(1,-1,0,\dots,0)^{\top}$ and equal weights on $\pm v$, $\pm w$ share the mean-only and
the diagonal target
$C^{l,\mathrm{diag}}_{i,b}=\mathrm{Diag}\big(\sum_j\omega^{l}_{ij}\,\delta^{l}_{ij,b}\odot\delta^{l}_{ij,b}\big)+\epsilon_sI_m$,
yet $C^{l}_{i,b}\ne C^{l}_{i',b}$.
\end{enumerate}
\end{proposition}

\begin{proof}
\emph{Preliminaries.} Since the weights sum to one (Proposition~\ref{prop:gauge}),
\begin{equation}
\bar\delta^{l}_{i,b}
=\sum_{j\in\tilde{\mathcal N}(i)}\omega^{l}_{ij}\,m^{l}_{j\to i,b}-\big(\phi(\xi^{l}_i)\big)_b
=(\bar\xi^{l}_i)_b-\big(\phi(\xi^{l}_i)\big)_b .
\end{equation}
Hence, under $\xi^{l}_i=\xi^{l}_{i'}$, the conditions $\bar\delta^{l}_{i,b}=\bar\delta^{l}_{i',b}$
for all $b$ and $\bar\xi^{l}_i=\bar\xi^{l}_{i'}$ are equivalent. Moreover,
$Z^{l}_i=Z^{l}_{i'}$ implies $L^{l}_i=L^{l}_{i'}$, and the self residual vanishes,
$\delta^{l}_{ii,b}=0$, because $T^{l}_{i\to i,b}=I_m$. By Eq.~\eqref{eq:second_moment},
$C^{l}_{i,b}\succeq\epsilon_sI_m\succ0$, so the target of Eq.~\eqref{eq:coordinate_target}
exists for every fixed jitter $\eta\ge0$. Throughout, the statistics are regarded as
functions of the residual configuration $\mathcal R^{l}_{i,b}$.

\medskip\noindent
\emph{(i) First-order invariance.}
By definition, a first-order update depends on $\mathcal R^{l}_{i,b}$ only through the
weights and $\bar\delta^{l}_{i,b}$, which coincide for $i$ and $i'$; hence its outputs
coincide. This covers the three updates listed in the statement:
\begin{itemize}[leftmargin=1.6em,itemsep=1pt,topsep=2pt]
\item the mean-only target is a function of $\bar\delta^{l}_{i,b}$ alone;
\item the controller input $q^{l}_{i,b}=[(\xi^{l}_i)_b\|(\bar\xi^{l}_i)_b\|u_i]$ coincides
for the two nodes, so $\Gamma_\theta(q^{l}_{i,b})=\Gamma_\theta(q^{l}_{i',b})$ and, by
Eq.~\eqref{eq:correction}, $\Delta z^{l}_{i,b}=\Delta z^{l}_{i',b}$;
\item the feature update of Eq.~\eqref{eq:feature_gate} is a function of
$(\xi^{l}_i,\bar\xi^{l}_i)$ only, so $\xi^{l+1}_i=\xi^{l+1}_{i'}$.
\end{itemize}

\medskip\noindent
\emph{(ii) Injectivity and propagation.}
For a fixed jitter, $C\mapsto R=\mathrm{SChol}(C)$ is the Cholesky factor of $C+\eta I_m$,
a bijection between $\mathrm{SPD}(m)$ and lower-triangular matrices with positive
diagonal~\citep{lin2019riemannian}, hence injective. The map
\begin{equation}
R\ \mapsto\ \hat z=\big[\operatorname{Svec}_{\mathrm{sl}}(R)^{\top},\ \log(\operatorname{diag}R)^{\top}\big]^{\top}
\end{equation}
of Eq.~\eqref{eq:coordinate_target} is injective because $R$ is recovered from $\hat z$ as
$\operatorname{mat}_{\mathrm{sl}}(\hat a)+\mathrm{Diag}(e^{\hat\ell})$. Therefore
$C^{l}_{i,b}\ne C^{l}_{i',b}$ implies $\hat z^{l}_{i,b}\ne\hat z^{l}_{i',b}$.

When Eq.~\eqref{eq:task_guided_geometry_update} is not clipped in block $b$, the
projection acts as the identity. Since $z^{l}_{i,b}=z^{l}_{i',b}$ and
$\Delta z^{l}_{i,b}=\Delta z^{l}_{i',b}$ by (i),
\begin{equation}
z^{l+1}_{i,b}-z^{l+1}_{i',b}=\lambda\big(\hat z^{l}_{i,b}-\hat z^{l}_{i',b}\big)\ne0
\qquad\text{for }\lambda>0,
\end{equation}
so $Z^{l+1}_i\ne Z^{l+1}_{i'}$. The map $z\mapsto L(z)$ of Eq.~\eqref{eq:frame} is
injective, since $\ell_r=\log L_{rr}$ and $a_{rs}=L_{rs}$ recover $z$ from $L(z)$; hence
$L^{l+1}_{i,b}\ne L^{l+1}_{i',b}$ and, as block-diagonal matrices,
$L^{l+1}_i\ne L^{l+1}_{i'}$. With $\xi^{l+1}_i=\xi^{l+1}_{i'}$,
\begin{equation}
h^{l+1}_i-h^{l+1}_{i'}=\big(L^{l+1}_i-L^{l+1}_{i'}\big)\,\xi^{l+1}_i ,
\end{equation}
which vanishes only if $\xi^{l+1}_i$ lies in the kernel of the nonzero matrix
$L^{l+1}_i-L^{l+1}_{i'}$, a proper subspace of $\mathbb{R}^{d}$.

\medskip\noindent
\emph{(iii) Genericity.}
Let $j=1,\dots,k$ with $k\ge2$ index the neighbors of $i'$ other than itself, fix their
weights $(\omega_j)_{j=1}^{k}$, all positive by Proposition~\ref{prop:gauge}, and
parametrize the residuals of $i'$ in block $b$ by
$\delta=(\delta_1,\dots,\delta_k)\in\mathbb{R}^{mk}$; the self residual is zero and does not
enter. Equal first moment is the affine constraint
\begin{equation}
\sum_{j=1}^{k}\omega_j\delta_j=\bar\delta^{l}_{i,b},
\end{equation}
which defines an affine subspace $\mathcal A\subset\mathbb{R}^{mk}$ of dimension $m(k-1)$.
On $\mathcal A$ the map $C(\delta)=\sum_j\omega_j\delta_j\delta_j^{\top}+\epsilon_sI_m$ is
polynomial, and so is $p(\delta):=C(\delta)_{aa}-(C^{l}_{i,b})_{aa}$ for any fixed
coordinate $a$.

The polynomial $p$ is not identically zero on $\mathcal A$. For $\delta\in\mathcal A$ and
$t\in\mathbb{R}$, the perturbation $\delta_1\mapsto\delta_1+te_a$,
$\delta_2\mapsto\delta_2-(\omega_1/\omega_2)te_a$ keeps $\sum_j\omega_j\delta_j$ fixed, hence
stays in $\mathcal A$, and changes $C(\delta)_{aa}$ by
\begin{equation}
\omega_1\big[(\delta_{1,a}+t)^{2}-\delta_{1,a}^{2}\big]
+\omega_2\Big[\big(\delta_{2,a}-\tfrac{\omega_1}{\omega_2}t\big)^{2}-\delta_{2,a}^{2}\Big]
=2\omega_1t\,(\delta_{1,a}-\delta_{2,a})+\Big(\omega_1+\frac{\omega_1^{2}}{\omega_2}\Big)t^{2},
\end{equation}
a polynomial in $t$ with positive leading coefficient, hence nonconstant along the line.
The zero set of a nonzero polynomial on an affine space has Lebesgue measure zero, so
\begin{equation}
\{\delta\in\mathcal A:C(\delta)=C^{l}_{i,b}\}\subseteq\{\delta\in\mathcal A:p(\delta)=0\}
\end{equation}
is a Lebesgue-null subset of $\mathcal A$.

\medskip\noindent
\emph{(iv) Limits of second-order separation.}
The target $\hat z^{l}_{i,b}$ is a function of $C^{l}_{i,b}$ alone, so any two residual
configurations with equal weighted second moments yield equal targets.

\emph{The case $m=1$.} Take equal weights $1/3$, $q\in(0,1)$, and
$p,r=\tfrac{-q\pm\sqrt{4-3q^{2}}}{2}$, which are real since $4-3q^{2}>0$. Then $p$ and $r$
are the roots of $t^{2}+qt+(q^{2}-1)=0$, so $p+r=-q$, $pr=q^{2}-1$, and
$p^{2}+r^{2}=(p+r)^{2}-2pr=2-q^{2}$. The set $\{q,p,r\}$ therefore has first moment
$(q+p+r)/3=0$ and second moment $(q^{2}+p^{2}+r^{2})/3=2/3$, as does $\{-1,0,1\}$, while its
third moment is
\begin{equation}
\tfrac13\big(q^{3}+p^{3}+r^{3}\big)
=\tfrac13\big(q^{3}+(p+r)^{3}-3pr(p+r)\big)
=\tfrac13\big(q^{3}-q^{3}+3q(q^{2}-1)\big)
=q(q^{2}-1)\ne0 .
\end{equation}
The two sets are distinct multisets, so an injective aggregator separates them, whereas the
second-order target does not.

\emph{The case $m\ge2$.} Consider the configurations $\{0,\pm v\}$ and $\{0,\pm w\}$ with
weights $(\omega_0,\omega,\omega)$, $\omega>0$. In both, the first moment is
$\omega_0\cdot0+\omega v-\omega v=0$ (respectively with $w$), so the mean-only targets are
both $\epsilon_sI_m$. The full statistics are
\begin{equation}
C^{l}_{i,b}=2\omega\,vv^{\top}+\epsilon_sI_m,
\qquad
C^{l}_{i',b}=2\omega\,ww^{\top}+\epsilon_sI_m,
\end{equation}
whose $(1,2)$ entries are $+2\omega$ and $-2\omega$, so they differ. The diagonal
statistics use $v\odot v=w\odot w=(1,1,0,\dots,0)^{\top}$ and therefore coincide,
$\mathrm{Diag}\big(2\omega(1,1,0,\dots,0)^{\top}\big)+\epsilon_sI_m$. Such configurations
arise, for instance, from two neighbors whose transformed features differ from
$\phi(\xi^{l}_i)$ by $\pm v$ (respectively $\pm w$) and share the same frame and environment
distance to the target.
\end{proof}

\section{Proof of Proposition~\ref{prop:stability}}
\label{app:proof_stability}

\begin{proposition}[Stability and Convergence]
Fix $\tau>0$, $\epsilon_s>0$, $\eta\ge0$ and a feature bound $\Xi$. Let
$\mathcal D=\{(\xi,Z):\|\xi_i\|_2\le\Xi,\ z_{i,b}\in\mathcal Z\}$, let
$F:(\xi^{l},Z^{l})\mapsto(\xi^{l+1},Z^{l+1})$ denote one recurrent step, and let
$d_\infty(S,S')=\max_i\big(\|\xi_i-\xi'_i\|_2+\|Z_i-Z'_i\|_F\big)$.
\begin{enumerate}[label=(\roman*),leftmargin=1.6em,itemsep=1pt,topsep=2pt]
\item Eq.~\eqref{eq:task_guided_geometry_update} is a projected proximal step:
$\Pi_{\mathcal Z}$ is the Euclidean projection onto $\mathcal Z$, and $z^{l+1}_{i,b}$ is the
unique minimizer over $\mathcal Z$ of
\[
\tfrac{1-\lambda}{2}\|z-z^{l}_{i,b}\|_2^{2}+\tfrac{\lambda}{2}\|z-\hat z^{l}_{i,b}\|_2^{2}
-\gamma\langle\Delta z^{l}_{i,b},\,z-z^{l}_{i,b}\rangle .
\]
It differs from the uncorrected update
$z^{l+1,0}_{i,b}=\Pi_{\mathcal Z}\big((1-\lambda)z^{l}_{i,b}+\lambda\hat z^{l}_{i,b}\big)$
by at most $\gamma\|\Delta z^{l}_{i,b}\|_2$.
\item One step is Lipschitz on $\mathcal D$:
$d_\infty\big(F(S),F(S')\big)\le K\,d_\infty(S,S')$ with
$K=c_\xi+(1-\lambda)+\lambda c_{\hat Z}+\gamma c_\Delta$, where the constants depend only on
$m$, $B$, $\sigma_\pm$, $\kappa_L$, $\epsilon_s$, $\eta$, $\tau$, $\Xi$, $\lambda$, $\gamma$,
and the norms of $W_\phi$, $W_r$, and $\Gamma_\theta$, but not on node degrees.
\item For fixed $\xi$, let
$\Phi_\xi(Z)=\Pi_{\mathcal Z}\big((1-\lambda)Z+\lambda\hat Z(Z;\xi)+\gamma\Delta Z(Z;\xi)\big)$.
If $Z^{\star}$ is an unclipped fixed point of $\Phi_\xi$ whose Jacobian
$J=\partial\Phi_\xi(Z^{\star})$ has spectral radius $\rho(J)<1$, then $Z^{\star}$ is locally
attracting with rate $\rho(J)$: for every $\varepsilon>0$ there are $c>0$ and a neighborhood
$\mathcal U$ of $Z^{\star}$ such that
$\|Z^{l}-Z^{\star}\|_F\le c\,(\rho(J)+\varepsilon)^{l}\,\|Z^{0}-Z^{\star}\|_F$ for all
$Z^{0}\in\mathcal U$.
\item If moreover $\hat Z(\cdot;\xi)$ and $\Delta Z(\cdot;\xi)$ are $L_{\hat z}$- and
$L_\Delta$-Lipschitz on $\mathcal Z^{nB}$ and
$\rho:=(1-\lambda)+\lambda L_{\hat z}+\gamma L_\Delta<1$, the fixed point is unique and
$\|Z^{l+1}-Z^{l}\|_F\le\rho^{l}\|Z^{1}-Z^{0}\|_F$ from every initialization.
\end{enumerate}
\end{proposition}

\begin{proof}
Throughout, $\|\cdot\|$ denotes the Euclidean norm of a vector and the spectral norm of a
matrix, $\|\cdot\|_F$ the Frobenius norm, and $\|\cdot\|_1$, $\|\cdot\|_\infty$ the vector
$\ell_1$ and $\ell_\infty$ norms. The constants $\sigma_\pm$ and $\kappa_L$ are those of
Lemma~\ref{lem:wellposed}.

\medskip\noindent
\emph{(i) Projected proximal step.}
$\mathcal Z$ is a product of closed intervals, and the objective $\tfrac12\|z-y\|^{2}$
separates over coordinates, so its minimizer over $\mathcal Z$ is obtained by clipping each
coordinate of $y$ to its interval; hence $\Pi_{\mathcal Z}$ is the Euclidean projection. Let
$u=(1-\lambda)z^{l}_{i,b}+\lambda\hat z^{l}_{i,b}$ and $c=\Delta z^{l}_{i,b}$. Expanding
the squares and using $(1-\lambda)+\lambda=1$,
\begin{equation}
\tfrac{1-\lambda}{2}\|z-z^{l}_{i,b}\|^{2}+\tfrac{\lambda}{2}\|z-\hat z^{l}_{i,b}\|^{2}
-\gamma\langle c,z-z^{l}_{i,b}\rangle
=\tfrac12\|z-(u+\gamma c)\|^{2}+\mathrm{const},
\end{equation}
so the unique minimizer over $\mathcal Z$ is $\Pi_{\mathcal Z}(u+\gamma c)=z^{l+1}_{i,b}$ by
Eq.~\eqref{eq:task_guided_geometry_update}. The projection onto a closed convex set is
non-expansive, whence
\begin{equation}
\|z^{l+1}_{i,b}-z^{l+1,0}_{i,b}\|
=\|\Pi_{\mathcal Z}(u+\gamma c)-\Pi_{\mathcal Z}(u)\|\le\gamma\|c\| .
\end{equation}

\medskip\noindent
\emph{(ii) One-step Lipschitz bound.}
Let $S=(\xi,Z)$ and $S'=(\xi',Z')$ lie in $\mathcal D$, write $d:=d_\infty(S,S')$, and let
primes denote quantities computed from $S'$. Set $w_\phi:=\|W_\phi\|$, $w_r:=\|W_r\|$,
$\bar U:=\max_i\|u_i\|$, $\ell_\Gamma:=\mathrm{Lip}(\Gamma_\theta)\le1.1\|W_{\Gamma,2}\|\|W_{\Gamma,1}\|$
(since $\sup|\mathrm{SiLU}'|<1.1$), and
\begin{gather*}
H:=\sigma_+\Xi,\qquad c_h:=\max(\kappa_L\Xi,\sigma_+),\qquad c_\omega:=16Hc_h/\tau,\qquad
\bar M:=\sigma_+\sigma_-^{-1}w_\phi\Xi,\\
c_m:=w_\phi\big(\sigma_-^{-2}\sigma_+\kappa_L\Xi+\sigma_-^{-1}\kappa_L\Xi+\sigma_-^{-1}\sigma_+\big),\qquad
c_{\bar\xi}:=\bar Mc_\omega+c_m,\\
c_\xi:=2+c_{\bar\xi}+\tfrac{w_r}{4}(1+c_{\bar\xi})(\bar M+\Xi),\qquad
\bar D:=\bar M+w_\phi\Xi,\qquad \bar C:=\bar D^{2}+\epsilon_s,\\
c_C:=\bar D^{2}c_\omega+2\bar D(c_m+w_\phi),\qquad
L_{\mathrm{chol}}:=\frac{(\bar C+\eta)^{1/2}}{\epsilon_s+\eta},\qquad
c_z:=\max\big\{1,(\epsilon_s+\eta)^{-1/2}\big\},\\
c_{\hat Z}:=\sqrt{B}\,c_zL_{\mathrm{chol}}c_C,\qquad
\bar G:=\sup_{\|q\|\le\Xi+\bar M+\bar U}\|\Gamma_\theta(q)\|,\qquad
c_\Delta:=\sqrt{3B}\,\max(\bar G,1)\,\ell_\Gamma(1+c_{\bar\xi}).
\end{gather*}
Each bound below is derived for an arbitrary node $i$; since $d$ majorizes the perturbation
of $i$ and of every neighbor, the same $d$ appears throughout.

\smallskip
\emph{Frames.} By Lemma~\ref{lem:wellposed}, $\|L_i\|\le\sigma_+$, $\|L_i^{-1}\|\le\sigma_-^{-1}$,
$\|L_i-L'_i\|_F\le\kappa_L\|Z_i-Z'_i\|_F\le\kappa_Ld$ and
$\|L_i^{-1}-L_i'^{-1}\|\le\sigma_-^{-2}\kappa_Ld$.

\smallskip
\emph{Environment representations and weights.} $\|h_i\|=\|L_i\xi_i\|\le\sigma_+\Xi=H$ and
\begin{equation}
\|h_i-h'_i\|\le\|L_i-L'_i\|\,\|\xi_i\|+\|L'_i\|\,\|\xi_i-\xi'_i\|
\le\kappa_L\Xi\|Z_i-Z'_i\|_F+\sigma_+\|\xi_i-\xi'_i\|\le c_hd .
\end{equation}
Since $\|h_i-h_j\|\le2H$ and $|\,\|x\|^{2}-\|y\|^{2}|\le(\|x\|+\|y\|)\|x-y\|$,
$|e_{ij}-e'_{ij}|\le4H\big(\|h_i-h'_i\|+\|h_j-h'_j\|\big)\le8Hc_hd$. The softmax map is
$2$-Lipschitz from $\ell_\infty$ to $\ell_1$: its Jacobian at $\omega$ is
$\mathrm{Diag}(\omega)-\omega\omega^{\top}$, and for any $v$,
\begin{equation}
\big\|(\mathrm{Diag}(\omega)-\omega\omega^{\top})v\big\|_1
=\sum_k\omega_k\,|v_k-\langle\omega,v\rangle|\le\max_kv_k-\min_kv_k\le2\|v\|_\infty,
\end{equation}
because $\langle\omega,v\rangle$ is a convex combination of the $v_k$; the bound extends to
finite differences by the integral form of the mean value theorem. Applying it to
Eq.~\eqref{eq:weights},
$\|\omega_{i\cdot}-\omega'_{i\cdot}\|_1\le\tfrac{2}{\tau}\max_j|e_{ij}-e'_{ij}|\le c_\omega d$.

\smallskip
\emph{Messages and aggregate.} By Proposition~\ref{prop:gauge}, $m_{j\to i}=L_i^{-1}L_jW_\phi\xi_j$,
so $\|m_{j\to i}\|\le\bar M$. Writing
\begin{equation}
m_{j\to i}-m'_{j\to i}
=(L_i^{-1}-L_i'^{-1})L_jW_\phi\xi_j+L_i'^{-1}(L_j-L'_j)W_\phi\xi_j+L_i'^{-1}L'_jW_\phi(\xi_j-\xi'_j)
\end{equation}
and bounding the three terms by $\sigma_-^{-2}\kappa_Ld\cdot\sigma_+w_\phi\Xi$,
$\sigma_-^{-1}\cdot\kappa_Ld\cdot w_\phi\Xi$ and $\sigma_-^{-1}\sigma_+w_\phi d$ gives
$\|m_{j\to i}-m'_{j\to i}\|\le c_md$. Since the weights are nonnegative and sum to one,
\begin{equation}
    \begin{aligned}
        \|\bar\xi_i-\bar\xi'_i\|
&\le\sum_j|\omega_{ij}-\omega'_{ij}|\,\|m_{j\to i}\|+\sum_j\omega'_{ij}\|m_{j\to i}-m'_{j\to i}\|\\
&\le\bar M\|\omega_{i\cdot}-\omega'_{i\cdot}\|_1+\max_j\|m_{j\to i}-m'_{j\to i}\|
\le c_{\bar\xi}d,
    \end{aligned}
\end{equation}
and $\|\bar\xi_i\|\le\bar M$.

\smallskip
\emph{Feature update.} The sigmoid is $\tfrac14$-Lipschitz with values in $[0,1]$, so the
gate of Eq.~\eqref{eq:feature_gate} satisfies
$\|r_i-r'_i\|\le\tfrac{w_r}{4}\big(\|\xi_i-\xi'_i\|+\|\bar\xi_i-\bar\xi'_i\|\big)\le\tfrac{w_r}{4}(1+c_{\bar\xi})d$.
Writing $\xi^{l+1}_i=\xi_i+r_i\odot(\bar\xi_i-\xi_i)$,
\begin{equation}
\xi^{l+1}_i-\xi'^{\,l+1}_i
=(\xi_i-\xi'_i)+(r_i-r'_i)\odot(\bar\xi_i-\xi_i)+r'_i\odot\big[(\bar\xi_i-\bar\xi'_i)-(\xi_i-\xi'_i)\big],
\end{equation}
and using $\|a\odot b\|\le\|a\|\,\|b\|$, $\|r'_i\odot b\|\le\|b\|$ and
$\|\bar\xi_i-\xi_i\|\le\bar M+\Xi$,
\begin{equation}
\|\xi^{l+1}_i-\xi'^{\,l+1}_i\|
\le d+\tfrac{w_r}{4}(1+c_{\bar\xi})(\bar M+\Xi)\,d+(c_{\bar\xi}+1)\,d=c_\xi d .
\end{equation}

\smallskip
\emph{Residual statistic.} $\delta_{ij}=m_{j\to i}-W_\phi\xi_i$ satisfies
$\|\delta_{ij}\|\le\bar D$ and $\|\delta_{ij}-\delta'_{ij}\|\le(c_m+w_\phi)d$, and the same
bounds hold for each block $\delta_{ij,b}$. Since
$\|\delta\delta^{\top}-\delta'\delta'^{\top}\|_F\le2\bar D\|\delta-\delta'\|$ and
$\|\delta\delta^{\top}\|_F=\|\delta\|^{2}\le\bar D^{2}$, the statistic of
Eq.~\eqref{eq:second_moment} satisfies
\begin{equation}
\|C_{i,b}-C'_{i,b}\|_F
\le\bar D^{2}\|\omega_{i\cdot}-\omega'_{i\cdot}\|_1+2\bar D\max_j\|\delta_{ij,b}-\delta'_{ij,b}\|
\le c_Cd,
\end{equation}
and $\epsilon_sI_m\preceq C_{i,b}\preceq\bar CI_m$ because $\delta\delta^{\top}\preceq\|\delta\|^{2}I_m$.

\smallskip
\emph{Cholesky target.} Let $\mathcal K=\{C:\epsilon_sI_m\preceq C\preceq\bar CI_m\}$, a
compact convex subset of $\mathrm{SPD}(m)$, and let $R(C)$ denote the lower Cholesky factor
of $C+\eta I_m$. We bound the two maps $C\mapsto R$ and $R\mapsto\hat z$ in turn.

The map $C\mapsto R(C)$ is real-analytic on $\mathrm{SPD}(m)$~\citep{lin2019riemannian}.
Differentiating $RR^{\top}=C+\eta I_m$ gives $R^{-1}\,dC\,R^{-\top}=X+X^{\top}$ with
$X=R^{-1}dR$ lower triangular, hence
\begin{equation}
dR=R\,\Phi\big(R^{-1}\,dC\,R^{-\top}\big),
\end{equation}
where $\Phi$ keeps the strictly lower part of its argument and halves the diagonal, so that
$\|\Phi(Y)\|_F\le\|Y\|_F$. On $\mathcal K$, $\|R\|^{2}=\lambda_{\max}(C+\eta I_m)\le\bar C+\eta$
and $\|R^{-1}\|^{2}=\lambda_{\min}(C+\eta I_m)^{-1}\le(\epsilon_s+\eta)^{-1}$, so
\begin{equation}
\|dR\|_F\le\|R\|\,\|R^{-1}\|^{2}\,\|dC\|_F\le L_{\mathrm{chol}}\|dC\|_F .
\end{equation}
Since $\mathcal K$ is convex, the mean value theorem gives
$\|R(C)-R(C')\|_F\le L_{\mathrm{chol}}\|C-C'\|_F$.

For the coordinates of Eq.~\eqref{eq:coordinate_target}, the strictly lower coordinates are
the entries $R_{rs}$ themselves. The diagonal entries satisfy
$R_{tt}^{2}=C_{tt}-C_{t,<t}\,C_{<t,<t}^{-1}\,C_{<t,t}$, a Schur complement of $C+\eta I_m$,
hence $R_{tt}^{2}\ge\epsilon_s+\eta$ and
\begin{equation}
|\log R_{tt}-\log R'_{tt}|\le(\epsilon_s+\eta)^{-1/2}\,|R_{tt}-R'_{tt}|.
\end{equation}
Collecting all entries in $\ell_2$,
\begin{equation}
\|\hat z_{i,b}-\hat z'_{i,b}\|\le c_z\|R-R'\|_F\le c_zL_{\mathrm{chol}}\,c_C\,d,
\end{equation}
and stacking the $B$ blocks gives
$\|\hat Z_i-\hat Z'_i\|_F\le\sqrt{B}\,c_zL_{\mathrm{chol}}c_Cd=c_{\hat Z}d$.

\smallskip
\emph{Correction.} The controller input $q_{i,b}=[(\xi_i)_b\|(\bar\xi_i)_b\|u_i]$ satisfies
$\|q_{i,b}\|\le\Xi+\bar M+\bar U$ and, the signature being fixed,
$\|q_{i,b}-q'_{i,b}\|\le\|\xi_i-\xi'_i\|+\|\bar\xi_i-\bar\xi'_i\|\le(1+c_{\bar\xi})d$.
$\Gamma_\theta$ is $\ell_\Gamma$-Lipschitz, and its output on this ball is bounded by
$\bar G$; the reshaped factors therefore satisfy $\|U\|_F,\|V\|_F,\|d_\ell\|\le\bar G$. Using
$\|UV^{\top}-U'V'^{\top}\|_F\le\bar G(\|U-U'\|_F+\|V-V'\|_F)$, the correction of
Eq.~\eqref{eq:correction} satisfies
\begin{equation}
    \begin{aligned}
        \|\Delta z_{i,b}-\Delta z'_{i,b}\|
&\le\|UV^{\top}-U'V'^{\top}\|_F+\|d_\ell-d'_\ell\|\\
&\le\max(\bar G,1)\big(\|U-U'\|_F+\|V-V'\|_F+\|d_\ell-d'_\ell\|\big)\\
&\le\sqrt3\max(\bar G,1)\,\|\Gamma_\theta(q_{i,b})-\Gamma_\theta(q'_{i,b})\|
\le\sqrt3\max(\bar G,1)\,\ell_\Gamma(1+c_{\bar\xi})\,d ,
    \end{aligned}
\end{equation}
and stacking the $B$ blocks gives $\|\Delta Z_i-\Delta Z'_i\|_F\le c_\Delta d$.

\smallskip
\emph{Composition.} $\Pi_{\mathcal Z}$ acts coordinatewise and is non-expansive in the
Frobenius norm, so by Eq.~\eqref{eq:task_guided_geometry_update}
\begin{equation}
\|Z^{l+1}_i-Z'^{\,l+1}_i\|_F
\le(1-\lambda)\|Z_i-Z'_i\|_F+\lambda\|\hat Z_i-\hat Z'_i\|_F+\gamma\|\Delta Z_i-\Delta Z'_i\|_F
\le\big((1-\lambda)+\lambda c_{\hat Z}+\gamma c_\Delta\big)d .
\end{equation}
Adding the feature bound and taking the maximum over $i$,
$d_\infty(F(S),F(S'))\le\big(c_\xi+(1-\lambda)+\lambda c_{\hat Z}+\gamma c_\Delta\big)d=Kd$.
None of the bounds involves the number of neighbors, so $K$ is independent of node degrees.

\medskip\noindent
\emph{(iii) Local convergence.}
Write $\Phi_\xi=\Pi_{\mathcal Z}\circ G$ with
\begin{equation}
G(Z)=(1-\lambda)Z+\lambda\hat Z(Z;\xi)+\gamma\Delta Z(Z;\xi).
\end{equation}
The argument has three steps.

\smallskip
\emph{Reduction to $G$.} Since $Z^{\star}$ is unclipped, $G(Z^{\star})=Z^{\star}$ lies in
the interior of $\mathcal Z^{nB}$. As $G$ is continuous, $G(Z)$ stays in the interior for
all $Z$ in a neighborhood $\mathcal V$ of $Z^{\star}$, so $\Phi_\xi=G$ on $\mathcal V$.

\smallskip
\emph{Differentiability.} On $\mathcal V$ the map $G$ is continuously differentiable: it is
a composition of the frame reconstruction of Eq.~\eqref{eq:frame}, the softmax, the
triangular solves, which are rational in the entries of invertible frames, the residual
second moment, which is polynomial, the Cholesky factorization and elementwise logarithm on
$\mathcal K$, and the SiLU network. Let $J=\partial G(Z^{\star})$.

\smallskip
\emph{Contraction near $Z^{\star}$.} Fix $\varepsilon>0$ with $\rho(J)+\varepsilon<1$.
There is a vector norm $\|\cdot\|_\varepsilon$ whose induced matrix norm satisfies
$\|J\|_\varepsilon\le\rho(J)+\varepsilon/2$~\citep[Lemma~5.6.10]{horn2013matrix}. By
differentiability, there is a neighborhood $\mathcal U\subseteq\mathcal V$ of $Z^{\star}$ on
which
\begin{equation}
\|G(Z)-Z^{\star}\|_\varepsilon\le\big(\rho(J)+\varepsilon\big)\|Z-Z^{\star}\|_\varepsilon .
\end{equation}
Hence $G$ maps $\mathcal U$ into itself, and iterating gives
\begin{equation}
\|Z^{l}-Z^{\star}\|_\varepsilon\le\big(\rho(J)+\varepsilon\big)^{l}\|Z^{0}-Z^{\star}\|_\varepsilon
\qquad\text{for all }Z^{0}\in\mathcal U .
\end{equation}
Equivalence of $\|\cdot\|_\varepsilon$ and $\|\cdot\|_F$ on the finite-dimensional space
$\mathbb{R}^{nB\cdot m(m+1)/2}$ yields the constant $c$~\citep{ortega1970iterative}.

\medskip\noindent
\emph{(iv) Global convergence.}
With $\xi$ fixed, $\hat Z(\cdot;\xi)$ and $\Delta Z(\cdot;\xi)$ are Lipschitz on
$\mathcal Z^{nB}$ by part (ii); let $L_{\hat z}$ and $L_\Delta$ be their constants. Then
\begin{equation}
\|G(Z)-G(Z')\|_F\le\big((1-\lambda)+\lambda L_{\hat z}+\gamma L_\Delta\big)\|Z-Z'\|_F=\rho\|Z-Z'\|_F,
\end{equation}
and non-expansiveness of $\Pi_{\mathcal Z}$ gives $\mathrm{Lip}(\Phi_\xi)\le\rho$. The set
$\mathcal Z^{nB}$ is a nonempty closed subset of a Euclidean space, hence complete, and
$\Phi_\xi$ maps it into itself. If $\rho<1$, the Banach fixed-point theorem yields a unique
$Z^{\star}$ with $\Phi_\xi(Z^{\star})=Z^{\star}$, and
\begin{equation}
\|Z^{l+1}-Z^{l}\|_F=\|\Phi_\xi(Z^{l})-\Phi_\xi(Z^{l-1})\|_F\le\rho\|Z^{l}-Z^{l-1}\|_F\le\rho^{l}\|Z^{1}-Z^{0}\|_F.
\end{equation}
\end{proof}

\section{Numerical Operators for Log-Triangular Targets}
\label{proof_4}

This section defines the operators used to construct the geometric update target in
Eq.~\eqref{eq:coordinate_target}.

\subsection{Strictly Lower-Triangular Vectorization}
For $A\in\mathbb{R}^{m\times m}$, the operator $\operatorname{Svec}_{\mathrm{sl}}$ extracts
the strictly lower-triangular entries in row-major order, so that
$[\operatorname{Svec}_{\mathrm{sl}}(A)]_{\kappa(r,s)}=A_{rs}$ with
$\kappa(r,s)=(r-1)(r-2)/2+s$ for $1\leq s<r\leq m$:
\begin{equation}
    \operatorname{Svec}_{\mathrm{sl}}(A)
    =
    \begin{bmatrix}
        A_{21},\,A_{31},\,A_{32},\,\ldots,\,A_{m1},\,\ldots,\,A_{m,m-1}
    \end{bmatrix}^{\top}
    \in\mathbb{R}^{m(m-1)/2}.
    \label{eq:appendix_svec_definition}
\end{equation}
The reconstruction operator $\operatorname{mat}_{\mathrm{sl}}$ of Eq.~\eqref{eq:block_geometry}
inverts this ordering: for $a\in\mathbb{R}^{m(m-1)/2}$,
\begin{equation}
    \big[\operatorname{mat}_{\mathrm{sl}}(a)\big]_{rs}
    =
    \begin{cases}
        a_{\kappa(r,s)}, & r>s,\\
        0, & r\leq s .
    \end{cases}
    \label{eq:appendix_mat_definition}
\end{equation}
The two operators satisfy
$\operatorname{Svec}_{\mathrm{sl}}\big(\operatorname{mat}_{\mathrm{sl}}(a)\big)=a$ and
$\operatorname{mat}_{\mathrm{sl}}\big(\operatorname{Svec}_{\mathrm{sl}}(A)\big)=\operatorname{tril}(A,-1)$,
where $\operatorname{tril}(A,-1)$ retains only the strictly lower-triangular entries of $A$.
In particular, $\operatorname{mat}_{\mathrm{sl}}$ is an isometry from
$\mathbb{R}^{m(m-1)/2}$ onto the strictly lower-triangular matrices, which is used in the
proof of Lemma~\ref{lem:wellposed}.

\subsection{Stabilized Cholesky Factorization}
For a symmetric matrix $C\in\mathbb{R}^{m\times m}$, let $\epsilon_0>0$ be the initial
diagonal jitter, $\beta>1$ its growth factor, and $T_{\max}\in\{0,1,\ldots\}$ the maximum
retry index. Define the jitter schedule
\begin{equation}
    \epsilon_t=\beta^t\epsilon_0,
    \qquad
    C_t=C+\epsilon_t I_m,
    \qquad
    t=0,\ldots,T_{\max},
    \label{eq:appendix_jitter_schedule}
\end{equation}
and, whenever the set is nonempty, the first admissible index
\begin{equation}
    t^*
    =
    \min\left\{
        t\in\{0,\ldots,T_{\max}\}:
        C_t\in\operatorname{SPD}(m)
    \right\}.
    \label{eq:appendix_jitter_index}
\end{equation}
The stabilized factorization returns the Cholesky factor of the first admissible candidate,
\begin{equation}
    \operatorname{SChol}(C;\epsilon_0,\beta,T_{\max})
    =
    \operatorname{chol}(C_{t^*}),
    \qquad
    RR^\top=C+\epsilon_{t^*}I_m,
    \label{eq:appendix_schol_definition}
\end{equation}
where $\operatorname{chol}$ returns the unique lower-triangular factor $R$ with positive
diagonal. The main text abbreviates this operator as $\operatorname{SChol}(C)$ and denotes
the selected jitter $\epsilon_{t^*}$ by $\eta$. The schedule contains $T_{\max}+1$ candidate
factorizations, including the initial attempt.

For the residual second-moment matrix, Eq.~\eqref{eq:second_moment} gives
$C_{i,b}^l\succeq\epsilon_s I_m$ with $\epsilon_s>0$, hence
$C_{i,b}^l+\epsilon_t I_m\succeq(\epsilon_s+\epsilon_t)I_m\succ0$ for every $t$. The
factorization is therefore well-defined, with $t^*=0$ in exact arithmetic; in floating-point
arithmetic, retries are triggered by numerical factorization failure. The jitter
$\epsilon_t$ is a numerical safeguard and is separate from the regularization $\epsilon_s$
that enters the statistic itself.

\subsection{Construction of the Coordinate Target}
Given $R_{i,b}^l=\operatorname{SChol}(C_{i,b}^l)$, the target coordinates are
\begin{equation}
    \widehat z_{i,b}^l
    =
    \begin{bmatrix}
        \widehat a_{i,b}^l\\
        \widehat\ell_{i,b}^l
    \end{bmatrix}
    =
    \begin{bmatrix}
        \operatorname{Svec}_{\mathrm{sl}}(R_{i,b}^l)\\
        \log\!\left(\operatorname{diag}(R_{i,b}^l)\right)
    \end{bmatrix}
    \in\mathbb{R}^{m(m+1)/2},
    \label{eq:appendix_target_coordinates}
\end{equation}
where $\operatorname{diag}$ extracts the diagonal as a column vector and the logarithm acts
elementwise; the positive diagonal of $R_{i,b}^l$ makes $\widehat\ell_{i,b}^l$
well-defined. With the coordinate ordering of Eq.~\eqref{eq:block_geometry},
\begin{equation}
    \operatorname{mat}_{\mathrm{sl}}(\widehat a_{i,b}^l)
    +
    \operatorname{Diag}\!\left(\exp(\widehat\ell_{i,b}^l)\right)
    =
    R_{i,b}^l ,
\end{equation}
so $\widehat z_{i,b}^l$ encodes $R_{i,b}^l$ in frame coordinates. Since Cholesky
factorization is a bijection between $\operatorname{SPD}(m)$ and lower-triangular matrices
with positive diagonals, the encoding is lossless:
$R_{i,b}^l(R_{i,b}^l)^\top=C_{i,b}^l+\epsilon_{t^*}I_m$.

The target frame has a direct geometric meaning. For $h=R_{i,b}^l x$,
\begin{equation}
    \|x\|_2^2
    =
    \big\|(R_{i,b}^l)^{-1}h\big\|_2^2
    =
    h^\top\big(C_{i,b}^l+\epsilon_{t^*}I_m\big)^{-1}h,
\end{equation}
so $R_{i,b}^l$ maps the local unit ball onto the residual second-moment ellipsoid in the
shared environment,
\begin{equation}
    \left\{R_{i,b}^l x:\|x\|_2\leq1\right\}
    =
    \left\{h:\,h^\top\big(C_{i,b}^l+\epsilon_{t^*}I_m\big)^{-1}h\leq1\right\}.
    \label{eq:appendix_target_ellipsoid}
\end{equation}
The singular values of $R_{i,b}^l$ are the square roots of the eigenvalues of
$C_{i,b}^l+\epsilon_{t^*}I_m$. Local directions mapped onto high-energy residual axes are
therefore stretched more strongly in the environment, where neighborhood weights are
computed.

\section{Datasets}\label{sec:dataset}

\begin{table*}[t]
    \centering
    \caption{Statistics of the experimental datasets.}
    \vspace{0.2cm}
    \resizebox{0.65\textwidth}{!}{
    \begin{tabular}{lcccc}
        \toprule
        Datasets      & Graphs & Avg. Nodes & Avg. Edges & Classes \\
        \midrule
        CiteSeer & - & 3,327 & 9,104 & 6 \\
        PubMed & - & 19,717 & 88,648 & 3 \\
        CS & - & 18,333 & 163,788 & 15 \\
        Physics & - & 34,493 & 495,924 & 5 \\
        Photo & - & 7,650 & 238,162 & 8\\
        Computers & - & 13,752 & 491,722 & 10\\
        \midrule
        PROTEINS  & 1,113   & 39.10      & 72.80      & 2       \\
        NCI1  & 4,110   & 29.87     & 32.30     & 2       \\
        Mutagenicity  & 4,337   & 30.32      & 30.77      & 2       \\
        FRANKENSTEIN  & 4,337   & 16.90       & 17.88      & 2       \\
        BBBP  & 2,050   & 23.90       & 51.60      & 2       \\
        ogbg-molhiv & 41,127 & 25.50 & 27.50 & 2 \\
        \bottomrule
    \end{tabular}
    }
    \label{tab:dataset}
\end{table*}

\subsection{Dataset Description}

We conduct node classification, link prediction, and graph classification on a variety of datasets. The statistics of the datasets are summarized in Table \ref{tab:dataset}. The detailed descriptions of these dataset are provided as follows:

(1) For node classification and link prediction:

\begin{itemize}
    \item \textbf{CiteSeer}: The CiteSeer dataset~\cite{sen2008collective} is a widely used benchmark citation network comprising 3,327 nodes and 9,104 edges, designed for multi-class node classification with 6 categories. In this graph, nodes represent scientific publications, while edges correspond to citation relationships between documents, capturing the underlying structure of the citation network. Each node is assigned a class label that indicates the research topic of the corresponding paper. 
    \item \textbf{PubMed}: The PubMed dataset~\cite{sen2008collective} is a large-scale citation network benchmark comprising 19,717 nodes and 88,648 edges, designed for multi-class node classification with 3 categories. In this graph, nodes represent scientific publications related to diabetes research, while edges denote citation relationships between documents, capturing the structural dependencies within the citation network. Each node is assigned a class label corresponding to the type of diabetes discussed in the paper.
    \item \textbf{CS}: The CS dataset~\cite{shchur2018pitfalls} is a co-authorship network derived from the Microsoft Academic Graph, comprising 18,333 nodes and 163,788 edges for multi-class node classification with 15 categories. In this graph, nodes represent authors in the field of computer science, while edges indicate co-authorship relationships between them, capturing the collaboration structure of the research community. Each node is assigned a class label corresponding to the primary research field of the author.
    \item \textbf{Physics}: The Physics dataset~\cite{shchur2018pitfalls} is a co-authorship network also derived from the Microsoft Academic Graph, consisting of 34,493 nodes and 495,924 edges for multi-class node classification with 5 categories. In this graph, nodes denote authors in the field of physics, and edges represent co-authorship relationships, reflecting the collaboration patterns within the physics community. Each node is labeled according to the primary research area of the author.
    \item \textbf{Photo}: The Photo dataset~\cite{shchur2018pitfalls,mcauley2015image} is an Amazon co-purchase network comprising 7,650 nodes and 238,162 edges for multi-class node classification with 8 categories. In this graph, nodes represent products from the Amazon Photo category, while edges indicate that two products are frequently purchased together, capturing the co-purchasing patterns among items. Each node is assigned a class label corresponding to the product category.
    \item \textbf{Computers}: The Computers dataset~\cite{shchur2018pitfalls,mcauley2015image} is another Amazon co-purchase network consisting of 13,752 nodes and 491,722 edges for multi-class node classification with 10 categories. In this graph, nodes denote products in the Amazon Computers category, and edges represent co-purchase relationships between items, reflecting user purchasing behavior. Each node is labeled according to its product category.
\end{itemize}

(2) For graph classification:

\begin{itemize}
    \item \textbf{PROTEINS}: The PROTEINS dataset~\cite{dobson2003distinguishing} is a benchmark graph classification dataset consisting of 1,113 protein graphs. In this dataset, each graph represents a protein, where nodes correspond to secondary structure elements (e.g., helices and sheets), and edges indicate spatial or sequential adjacency between these elements, capturing the structural organization of proteins. Each graph is assigned a class label indicating whether the protein belongs to a specific functional class.
    \item \textbf{NCI1}: The NCI1 dataset~\cite{wale2008comparison} is a widely used benchmark for graph classification, consisting of 4,110 molecular graphs. In this dataset, each graph represents a chemical compound, where nodes correspond to atoms and edges denote chemical bonds, capturing the molecular structure. Each graph is labeled according to its activity against non-small cell lung cancer, indicating whether the compound is active or inactive.
    \item \textbf{Mutagenicity}: The Mutagenicity dataset~\cite{kazius2005derivation} is a benchmark graph classification dataset comprising 4,337 molecular graphs. In this dataset, each graph represents a chemical compound, where nodes correspond to atoms and edges denote chemical bonds, capturing the molecular structure. Each graph is labeled according to its mutagenic effect on a biological system, indicating whether the compound is mutagenic or non-mutagenic.
    \item \textbf{FRANKENSTEIN}: The FRANKENSTEIN dataset~\cite{orsini2015graph} is a molecular graph classification benchmark derived from Mutagenicity, consisting of 4,337 graphs. Nodes correspond to atoms whose symbols are replaced by 780-dimensional MNIST digit images as node attributes, and edges denote chemical bonds. Each graph is labeled according to whether the compound is mutagenic.
    \item \textbf{BBBP}: The BBBP dataset~\cite{wu2018moleculenet} is a molecular graph classification benchmark consisting of 2,050 compounds. In this dataset, each graph represents a molecule, where nodes correspond to atoms and edges denote chemical bonds, capturing the molecular structure. Each graph is labeled according to its ability to penetrate the blood–brain barrier, indicating whether the compound is permeable or non-permeable.
    
    \item \textbf{ogbg-molhiv}: The ogbg-molhiv dataset~\cite{hu2020open} is a large-scale molecular graph benchmark from the Open Graph Benchmark (OGB), consisting of 41,127 molecules. In this dataset, each graph represents a molecule, where nodes correspond to atoms and edges denote chemical bonds, capturing the molecular structure. Each graph is labeled according to its ability to inhibit HIV replication, indicating whether the compound is HIV active or inactive.
\end{itemize}

\subsection{Data Processing} For node classification and link prediction, we evaluate on standard benchmarks including citation networks (CiteSeer and PubMed), Amazon co-purchasing graphs (Computers and Photo), and Coauthor networks (CS and Physics). All datasets are preprocessed following PyTorch Geometric~\footnote{https://pyg.org/}
, where graphs are converted to undirected forms and node features are row-normalized. To enhance structural awareness, we further construct structural signatures from normalized degree information and Random Walk Structural Encodings (RWSE), which enter the prototype gate and the geometric controller. Specially, we adopt the standard random edge split protocol to construct positive and negative evaluation edges for link prediction. For graph classification, we consider bioinformatics datasets from TUDataset (e.g., PROTEINS and NCI1) and molecular graphs (e.g., BBBP and ogbg-molhiv).

\section{Baselines}\label{sec:baselines}

\subsection{Baseline Description}

In this part, we introduce the details of the compared baselines as follows:

(1) \textbf{General Graph Neural Networks (GNNs).} We compare \method{} with seven general GNNs:

\begin{itemize}
    \item \textbf{GCN}~\cite{kipf2017semi}: GCN is a graph neural network that propagates and transforms node features through normalized neighborhood aggregation, enabling the learning of expressive node representations via layer-wise message passing and smoothing.
    \item \textbf{GIN}~\cite{xu2019how}: GIN is a graph neural network that employs injective aggregation functions to maximally preserve structural information, enabling the learning of highly expressive node representations with discriminative power comparable to the Weisfeiler–Lehman test.
    \item \textbf{ML$^2$-GCL}~\cite{liang2025ml}: ML$^2$-GCL is a graph contrastive learning framework that leverages manifold learning principles to construct lightweight augmentations and objectives, enabling effective representation learning by preserving intrinsic geometric structures in graph data.
    \item \textbf{AMPs}~\cite{errica2025adaptive}: AMPs is a graph neural network framework that adaptively controls message passing to balance information flow, enabling effective mitigation of oversmoothing, oversquashing, and underreaching in deep graph models.
    \item \textbf{WaveGC}~\cite{liu2025general}: WaveGC is a graph neural network that employs spectral wavelet convolutions via Chebyshev order decomposition, enabling multi-scale feature extraction and effective representation learning over graph structures.
    \item \textbf{SPARROW}~\cite{lin2025simplified}: SPARROW is a graph contrastive learning model that eliminates explicit data augmentation by leveraging intrinsic structural signals, enabling effective representation learning through a simplified and efficient contrastive objective.
    \item \textbf{G$^2$Former}~\cite{zhangrestricted}: G$^2$Former is a graph neural network that integrates restricted global-aware graph filters to bridge GNNs and Transformers, enabling expressive node representation learning by capturing both local and global dependencies.
\end{itemize}

(2) \textbf{Manifold-based GNNs.} We compare \method{} with three Manifold-based GNNs: 

\begin{itemize}
    \item  \textbf{HGCN}~\cite{chami2019hyperbolic}: HGCN is a graph neural network that operates in hyperbolic space, enabling the learning of hierarchical node representations by performing message passing under non-Euclidean geometry.
    \item \textbf{D-GCN}~\cite{sun2024motif}: D-GCN is a graph neural network that integrates motif-aware Riemannian representations with generative-contrastive learning, enabling expressive node embeddings by capturing higher-order structures and non-Euclidean geometry.
    \item \textbf{SPDGNN}~\cite{wang2025enhancing}: SPDGNN is a graph neural network that operates on symmetric positive definite (SPD) manifolds using Cholesky decomposition, enabling stable and expressive representation learning by preserving the geometric structure of covariance features.
\end{itemize}

(3) \textbf{Adaptive GNNs.} We compare \method{} with four adaptive GNNs:

\begin{itemize}
    \item \textbf{ACE-HGNN}~\cite{fu2021ace}:ACE-HGNN is a geometric graph neural network that adaptively explores manifold curvature, enabling flexible representation learning by capturing heterogeneous geometric structures across graph data.
    \item \textbf{BEC-GNN}~\cite{hevapathige2025depth}: BEC-GNN is a graph neural network that leverages learnable Bakry–Émery curvature to adapt message passing depth, enabling flexible representation learning by dynamically controlling information propagation across graph structures.
    \item  \textbf{GNRF}~\cite{chen2025graph}: GNRF is a graph neural network that leverages Ricci flow to evolve node features from a curvature perspective, enabling adaptive information propagation by dynamically reshaping the underlying graph geometry.
    \item \textbf{ARGNN}~\cite{wang2026adaptive}: ARGNN is a graph neural network that learns anisotropic node-wise Riemannian metric tensors from node features and neighborhood means, enabling geometry-adaptive message passing by tailoring the local metric to each node.
\end{itemize}

\subsection{Implementation Details}

We implement \method{} and all baselines in PyTorch\footnote{https://pytorch.org/} and conduct all experiments on NVIDIA A100 GPUs. For the baselines, we use the official implementations released by the authors when available, and otherwise implement them following the original papers; their hyperparameters follow the settings reported in the corresponding papers. \method{} is trained with Adam using a learning rate of $1\times10^{-3}$ and a weight decay of $5\times10^{-4}$, with model selection on the validation set. We use $B=8$ geometric blocks, $K=4$ geometric prototypes, and $L=2$ recurrent steps with shared parameters. Structural signatures $u_i$ consist of the normalized degree and random-walk return probabilities. The controller $\Gamma_\theta$ has a single hidden layer and outputs factor matrices $U_{i,b}^l,V_{i,b}^l\in\mathbb{R}^{m\times r}$ of rank $r=4$. The log-triangular coordinates are clipped to the admissible domain $\mathcal Z$ with $a_{\max}=2.0$, $\ell_{\min}=-5.0$, and $\ell_{\max}=5.0$. For the geometric update, we set the residual-target mixing weight $\lambda=0.5$ and the correction strength $\gamma=0.1$; the neighborhood weighting temperature is $\tau=1.0$. We report accuracy (ACC) for node classification and ROC-AUC for link prediction; for graph classification, we report ACC on PROTEINS, Mutagenicity, NCI1, and FRANKENSTEIN, and ROC-AUC on BBBP and ogbg-molhiv. Results for node classification and link prediction are averaged over $10$ random splits, and graph classification uses $10$-fold cross-validation with the mean performance across folds.

\section{Algorithm}

The overall training and inference process of the proposed \method{} is shown in Algorithm~\ref{alg:overall}.

\section{Complexity Analysis}

In this section, we analyze the computational complexity of the proposed \method{}. Let $n$ and $|E|$ denote the numbers of nodes and edges, respectively, $K$ the number of prototypes, $L$ the number of recurrent steps, and $d$ the hidden dimension. We partition the feature space into $B$ blocks of size $m$ such that $d = Bm$. The prototype atlas initialization incurs a complexity of $\mathcal{O}(n \cdot K \cdot d \cdot m)$. During the $L$ recurrent steps, edge-level message passing performs batched triangular solves and residual outer products on block-diagonal frames with a cost of $\mathcal{O}(L \cdot |E| \cdot d \cdot m)$, while node-level geometric updates require stabilized Cholesky factorizations of $m \times m$ matrices, resulting in $\mathcal{O}(L \cdot n \cdot d \cdot m^2)$. Together with the $\mathcal{O}(L \cdot n \cdot d^2)$ cost of the shared feature transformations, the overall computational complexity simplifies to $\mathcal{O}(L \cdot d \cdot m \cdot |E| + L \cdot d \cdot (d + m^2) \cdot n)$. In particular, the block-diagonal parameterization reduces the node-wise geometric operations from $\mathcal{O}(d^3)$ to $\mathcal{O}(d \cdot m^2)$, keeping the cost of \method{} linear in the numbers of nodes and edges.

\begin{table*}[t]
\centering
\small
\caption{
Cross-task aggregate comparison across all 18 reported dataset--task settings. \textbf{Bold} indicates the best result.
}
\label{tab:cross_task_aggregate}
\vspace{0.1cm}

\setlength{\tabcolsep}{3.0pt}
\renewcommand{\arraystretch}{1.12}
\setlength{\extrarowheight}{0.2pt}

\resizebox{0.85\textwidth}{!}{
\begin{tabular}{
c|
>{\centering\arraybackslash}m{1.70cm}|
ccccccc
}
\toprule
Type
& Model
& NC Avg.
& LP Avg.
& GC Avg.
& Overall Avg.
& \shortstack{Avg. Rank}
& \shortstack{Avg. Gain}
& \shortstack{$p_{\mathrm{Holm}}$} \\
\midrule

\multirow{7}{*}{
    \centering
    \rotatebox{90}{
        \shortstack{General \\ GNNs}
    }
}
& GCN
& 87.7 & 90.6 & 76.3 & 84.9
& 13.2 & $+5.7$
& $2.91\times10^{-15}$ \\

& GIN
& 87.1 & 90.0 & 78.4 & 85.2
& 13.3 & $+5.4$
& $1.82\times10^{-15}$ \\

& ML$^2$-GCL
& 87.9 & 96.1 & 80.6 & 88.2
& 9.4 & $+2.4$
& $1.83\times10^{-7}$ \\

& AMPs
& 89.8 & 96.1 & 81.9 & 89.3
& 5.0 & $+1.3$
& $2.07\times10^{-2}$ \\

& WaveGC
& 89.7 & 96.6 & 81.7 & 89.4
& 4.4 & $+1.2$
& $4.17\times10^{-2}$ \\

& SPARROW
& 87.8 & 96.5 & 81.5 & 88.6
& 7.6 & $+1.9$
& $6.45\times10^{-5}$ \\

& G$^2$Former
& 89.2 & 96.5 & 81.6 & 89.1
& 5.3 & $+1.5$
& $1.55\times10^{-2}$ \\

\midrule

\multirow{3}{*}{
    \centering
    \rotatebox{90}{
        \shortstack{Manifold \\ GNNs}
    }
}
& HGCN
& 86.5 & 92.2 & 78.4 & 85.7
& 13.3 & $+4.9$
& $1.82\times10^{-15}$ \\

& D-GCN
& 88.3 & 94.9 & 79.9 & 87.7
& 10.3 & $+2.8$
& $6.01\times10^{-9}$ \\

& SPDGNN
& 88.4 & 94.1 & 80.7 & 87.7
& 10.0 & $+2.8$
& $1.76\times10^{-8}$ \\

\midrule

\multirow{4}{*}{
    \centering
    \rotatebox{90}{
        \shortstack{Adaptive \\ GNNs}
    }
}
& ACE-HGNN
& 88.5 & 96.4 & 80.4 & 88.4
& 8.5 & $+2.1$
& $3.90\times10^{-6}$ \\

& BEC-GNN
& 89.1 & 96.6 & 80.0 & 88.6
& 7.4 & $+2.0$
& $1.09\times10^{-4}$ \\

& GNRF
& 88.9 & 96.1 & 81.4 & 88.8
& 6.9 & $+1.7$
& $3.90\times10^{-4}$ \\

& ARGNN
& 90.1 & 96.6 & 81.5 & 89.4
& 4.4 & $+1.2$
& $4.17\times10^{-2}$ \\

\midrule

&
\method{}
& \textbf{91.3}
& \textbf{97.6}
& \textbf{82.8}
& \textbf{90.5}
& \textbf{1.0}
& --
& -- \\

\bottomrule
\end{tabular}
}

\vspace{-0.1cm}
\end{table*}

\begin{figure*}[t]
    \centering

    \begin{subfigure}[t]{0.245\textwidth}
        \centering
        \includegraphics[width=\linewidth]
        {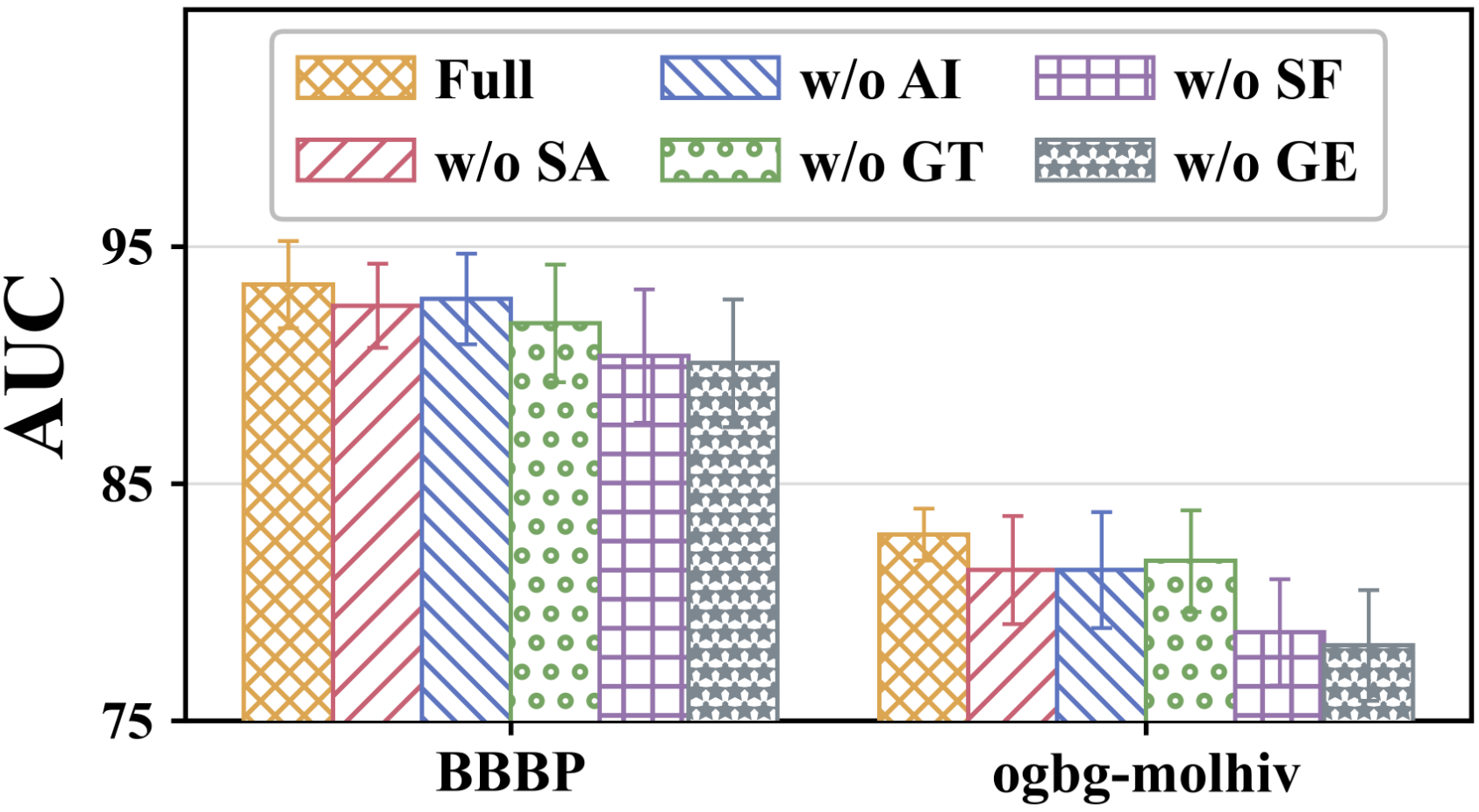}
        \caption{BBBP and ogbg-molhiv}
        \label{fig:extended_ablation_node}
    \end{subfigure}
    \hfill
    \begin{subfigure}[t]{0.245\textwidth}
        \centering
        \raisebox{-0.018cm}{\includegraphics[width=\linewidth,height=0.551\linewidth]
        {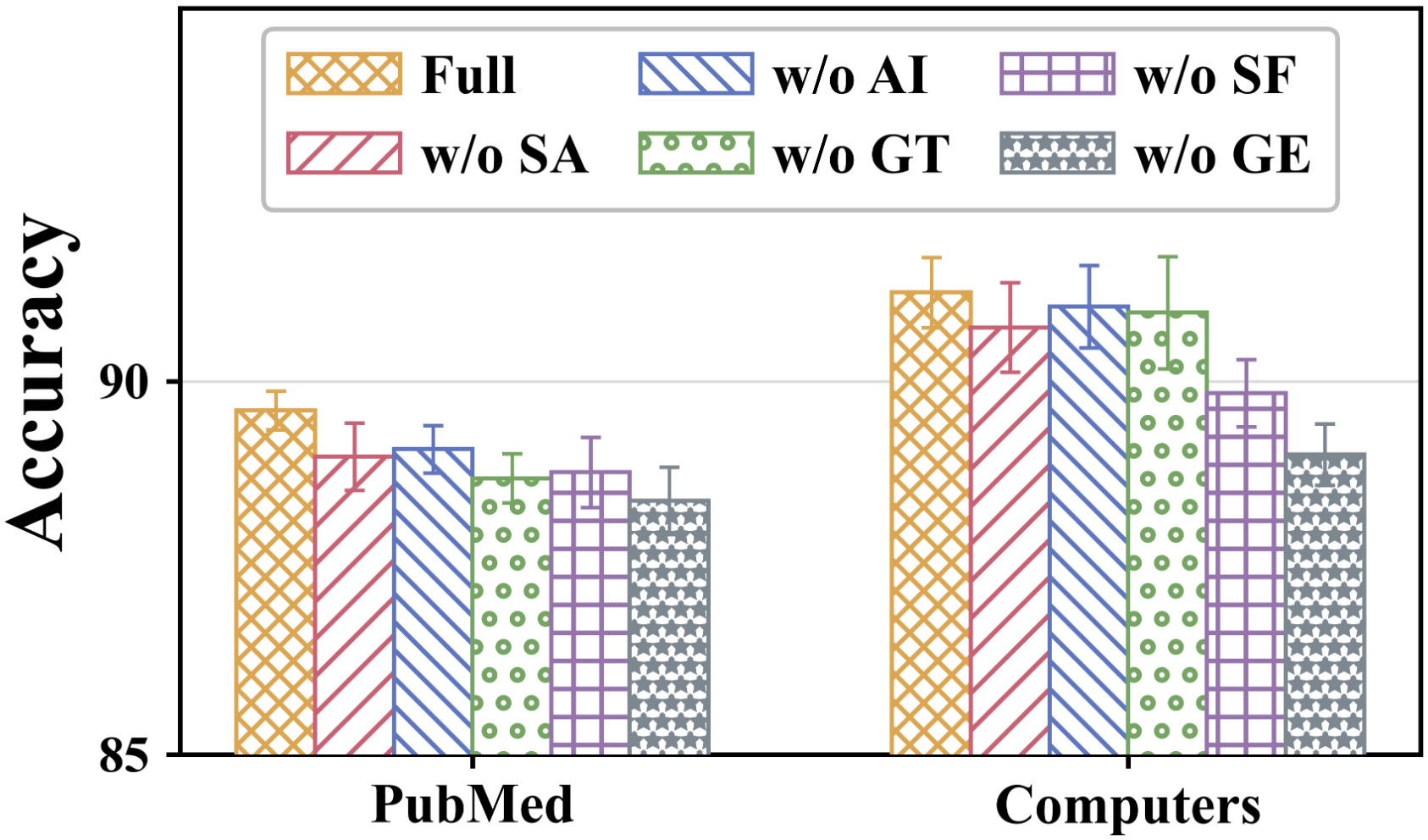}}
        \caption{PubMed and Computers}
        \label{fig:extended_ablation_graph}
    \end{subfigure}
    \hfill
    \begin{subfigure}[t]{0.245\textwidth}
        \centering
        \includegraphics[width=\linewidth,height=0.56\linewidth]
        {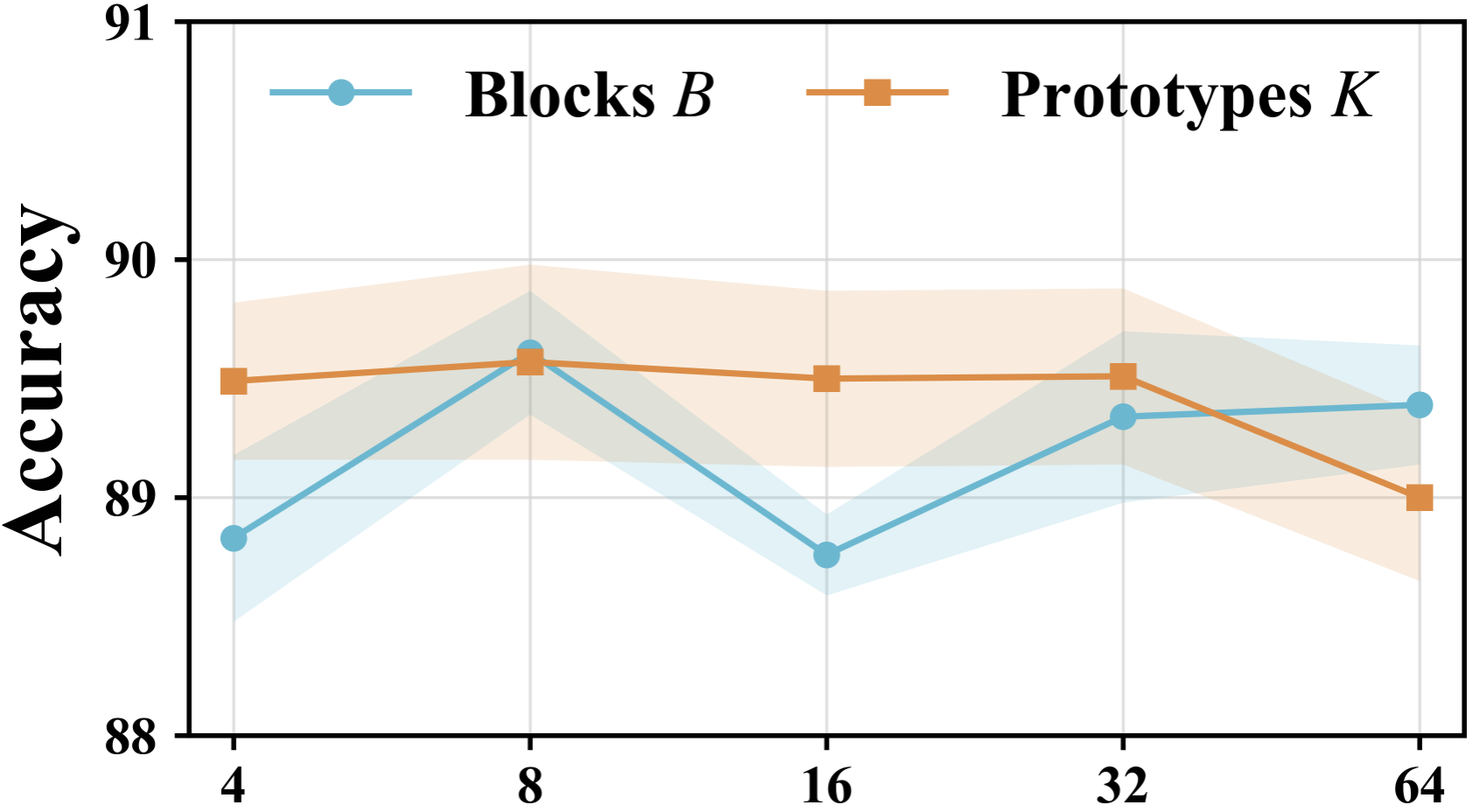}
        \caption{PubMed}
        \label{fig:extended_sensitivity_pubmed}
    \end{subfigure}
    \hfill
    \begin{subfigure}[t]{0.245\textwidth}
        \centering
        \includegraphics[width=\linewidth,height=0.565\linewidth]
        {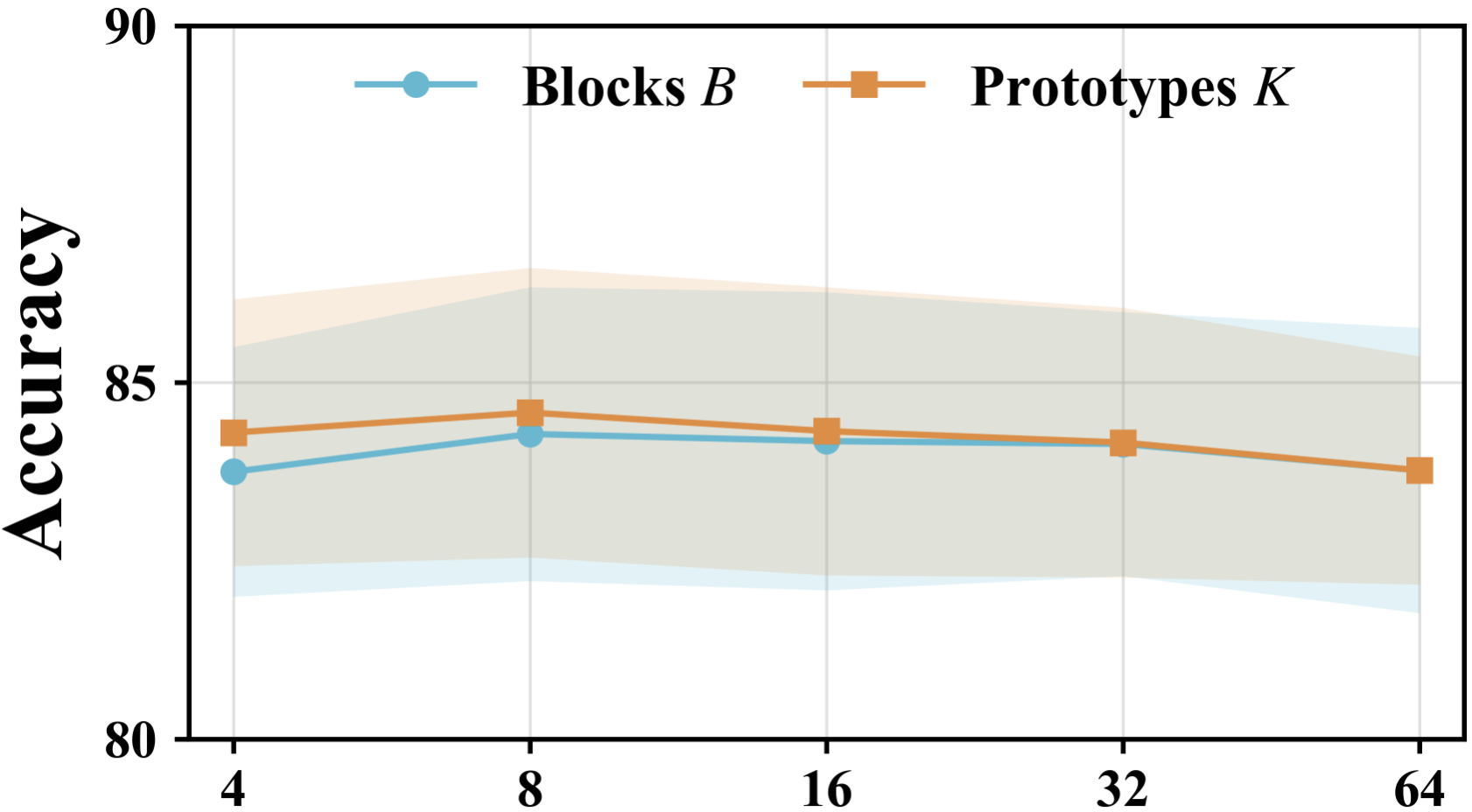}
        \caption{Mutagenicity}
        \label{fig:extended_sensitivity_mutag}
    \end{subfigure}

    \vspace{-0.15cm}
    \caption{Ablation studies on BBBP and ogbg-molhiv in (a), and PubMed and Computers in (b); sensitivity to geometric blocks $B$ and atlas prototypes $K$ on PubMed in (c) and Mutagenicity in (d).}
    \label{fig:extended_ablation_sensitivity}
    \vspace{-0.3cm}
\end{figure*}

\section{More experimental results}

\subsection{Cross-Task Aggregate Analysis}

We further summarize the cross-task aggregate results in Table~\ref{tab:cross_task_aggregate}. Each dataset--task pair is treated as one evaluation setting. NC Avg., LP Avg., and GC Avg. report the average performance within the corresponding task category, while Overall Avg. aggregates all settings. Avg. Rank is computed by ranking the methods within each setting and then averaging their ranks, whereas Avg. Gain measures the mean percentage-point improvement of \method{} over each baseline. The final column reports the Holm-adjusted post-hoc $p$-value derived from the average-rank differences following the global Friedman test, which yields $\chi^2=169.2$ with $14$ degrees of freedom and rejects the hypothesis of equal performance across methods.

\method{} achieves the strongest average within every task category and overall, together with an average rank of $1.0$, since it ranks first in all $18$ settings. Its gains over the baselines are positive throughout and grow from the strongest general and adaptive GNNs to the manifold-based and classical GNNs, and all $14$ pairwise comparisons are significant at the $0.05$ level after Holm correction. The closest baselines, WaveGC and ARGNN, share the smallest rank difference, and their separation from \method{} remains significant. Among the baseline families, adaptive GNNs and the strongest general GNNs form the leading group, whereas manifold-based GNNs trail on graph classification in particular. Overall, the gains are distributed across tasks and datasets rather than concentrated in a few favorable settings.

\begin{table}[t]
\centering
\caption{Time consumption of different methods in the training stage for each epoch (in seconds).}
\label{tab:time}
\resizebox{0.7\linewidth}{!}{
\begin{tabular}{ccccccc}
\toprule
Methods & PubMed & CS & Computers & NCI1 & Mutagenicity & ogbg-molhiv \\
\midrule
GCN
& 0.0086 & 0.0167 & 0.0153 & 0.2094 & 0.2312 & 2.1633 \\

AMPs
& 0.0393 & 0.0478 & 0.0761 & 0.3170 & 0.3173 & 2.8410 \\

G$^2$Former
& 0.0637 & 0.0739 & 0.0640 & 0.2011 & 0.2173 & 2.2996 \\

SPDGNN
& 0.0277 & 0.0321 & 0.0283 & 0.2073 & 0.2287 & 2.3973 \\

ARGNN
& 0.0902 & 0.1573 & 0.4328 & 0.2437 & 0.3383 & 2.5373 \\

\method{}
& 0.0683 & 0.1207 & 0.1810 & 0.4391 & 0.5357 & 6.0456 \\
\bottomrule
\end{tabular}
}
\end{table}

\begin{table}[t]
\centering
\caption{GPU memory consumption of different methods in the training stage (in GB).}
\label{tab:memory}
\resizebox{0.7\linewidth}{!}{
\begin{tabular}{ccccccc}
\toprule
Methods & PubMed & CS & Computers & NCI1 & Mutagenicity & ogbg-molhiv \\
\midrule
GCN
& 0.7 & 1.3 & 1.3 & 0.8 & 0.8 & 0.9 \\

AMPs
& 2.7 & 4.9 & 5.0 & 1.3 & 1.6 & 2.0 \\

G$^2$Former
& 3.3 & 3.2 & 3.0 & 1.5 & 1.5 & 1.4 \\

SPDGNN
& 1.3 & 2.9 & 1.9 & 0.9 & 1.0 & 0.9 \\

ARGNN
& 4.8 & 8.3 & 21.3 & 2.7 & 2.6 & 3.8 \\

\method{}
& 6.4 & 10.9 & 15.2 & 3.5 & 3.0 & 4.8 \\
\bottomrule
\end{tabular}
}
\end{table}

\subsection{More ablation study}~\label{sec:more_ablation}

We further extend the ablation study to PubMed and Computers for node classification and to BBBP and ogbg-molhiv for graph classification, covering different tasks and graph scales. As shown in Fig.~\ref{fig:extended_ablation_sensitivity}(a) and (b), the full model achieves the best performance on all four datasets. Freezing the geometric state (w/o GE) causes the largest degradation on every dataset, with more pronounced drops on the molecular benchmarks. Removing second-order feedback (w/o SF) causes the next largest drop on BBBP, ogbg-molhiv, and Computers, whereas on PubMed it is comparable to disabling frame transport (w/o GT). This ordering matches the main ablation and highlights the benefit of adapting propagation geometry during message passing and the contribution of residual statistics beyond task-supervised corrections. Disabling frame transport, removing structural signatures (w/o SA), and removing atlas initialization (w/o AI) yield smaller, dataset-dependent decreases, indicating that these components provide complementary improvements.

\subsection{More sensitivity study}~\label{sec:more_sensitivity}

To assess whether \method{} relies on narrowly tuned geometric capacity, we further vary the number of geometric blocks $B$ and atlas prototypes $K$ on PubMed and Mutagenicity. As shown in Fig.~\ref{fig:extended_ablation_sensitivity}(c) and (d), performance remains within a relatively narrow range across moderate values of both hyperparameters. On PubMed, varying $B$ produces modest non-monotonic fluctuations, while performance remains stable over a broad range of $K$ before decreasing at the largest setting. Mutagenicity exhibits an even flatter profile for both $B$ and $K$, followed by a mild decline under excessive capacity. Overall, increasing geometric granularity or atlas size does not yield systematic improvements, indicating that moderate block partitioning and a compact prototype atlas are sufficient for stable performance across different task settings.

\subsection{Efficiency and Resource Consumption Analysis}

We further evaluate the training efficiency of different methods in terms of per-epoch training time and GPU memory consumption. As shown in Tables~\ref{tab:time} and~\ref{tab:memory}, on the three node-classification datasets, \method{} trains faster than ARGNN, and the advantage grows with graph size, reaching its largest margin on Computers, while lightweight GNNs such as GCN and SPDGNN remain the cheapest. On the graph-classification datasets, \method{} requires more time than ARGNN and GCN, and the gap widens with the number of graphs, peaking on ogbg-molhiv. GPU memory consumption follows the same pattern: \method{} uses more memory than the compared baselines on most datasets, while remaining below ARGNN on Computers, and its memory on the graph-classification datasets stays close to that of ARGNN. The additional cost comes from the edge-level transport and residual outer products and the node-level factorizations, which are computed at every recurrent step. Across all datasets, the overhead relative to ARGNN remains a moderate constant factor rather than growing with graph size.

\section{Visualization}

\begin{figure}[h]
    \vspace{-0.3cm}
    \centering
    
    \begin{subfigure}{0.24\linewidth}
        \centering
        \includegraphics[width=0.9\linewidth]{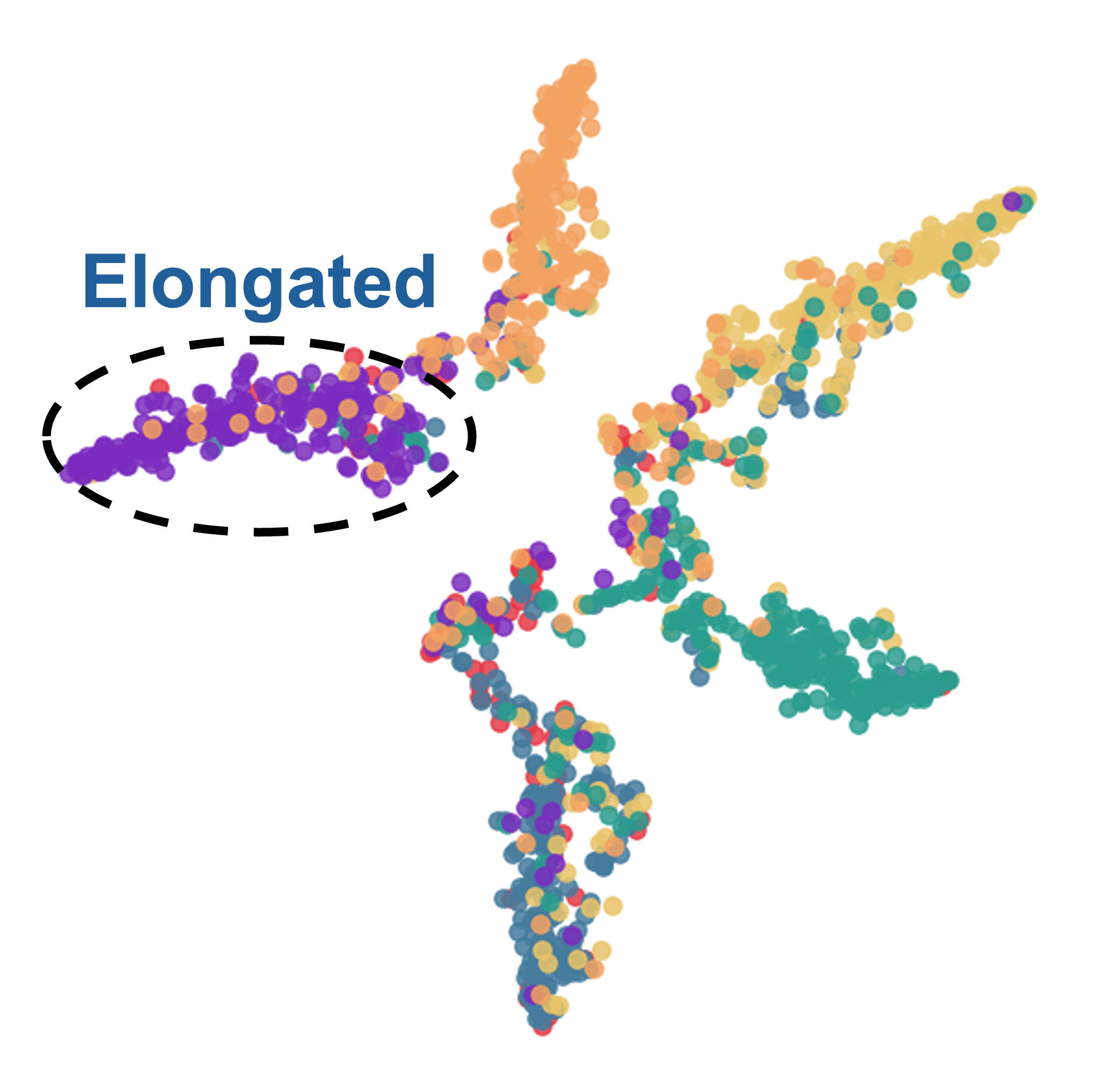}
        \caption{ARGNN}
    \end{subfigure}
    \hfill
    \begin{subfigure}{0.24\linewidth}
        \centering
        \includegraphics[width=0.9\linewidth]{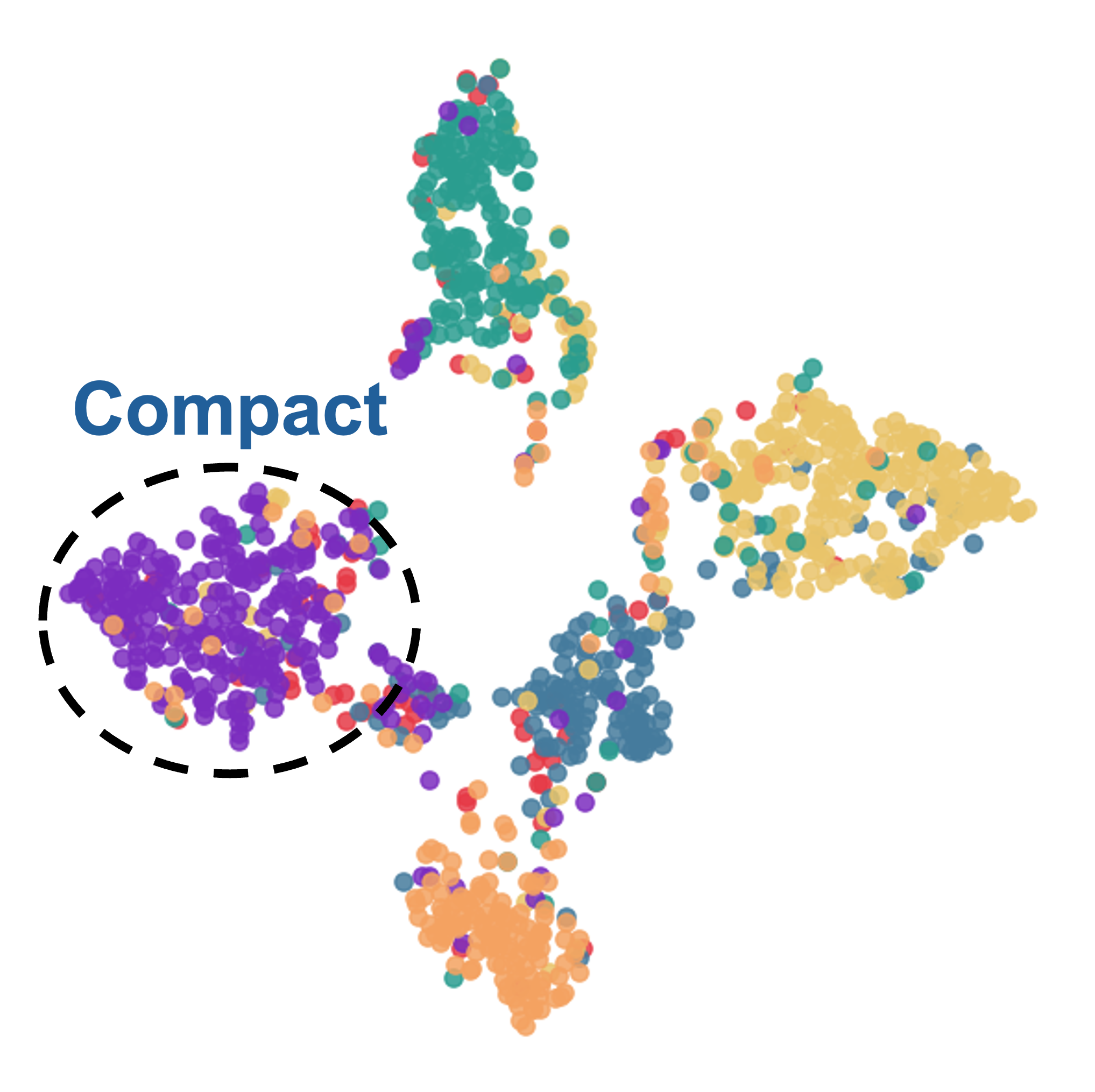}
        \caption{\method{}}
    \end{subfigure}
    \hfill
    \begin{subfigure}{0.24\linewidth}
        \centering
        \includegraphics[width=\linewidth]{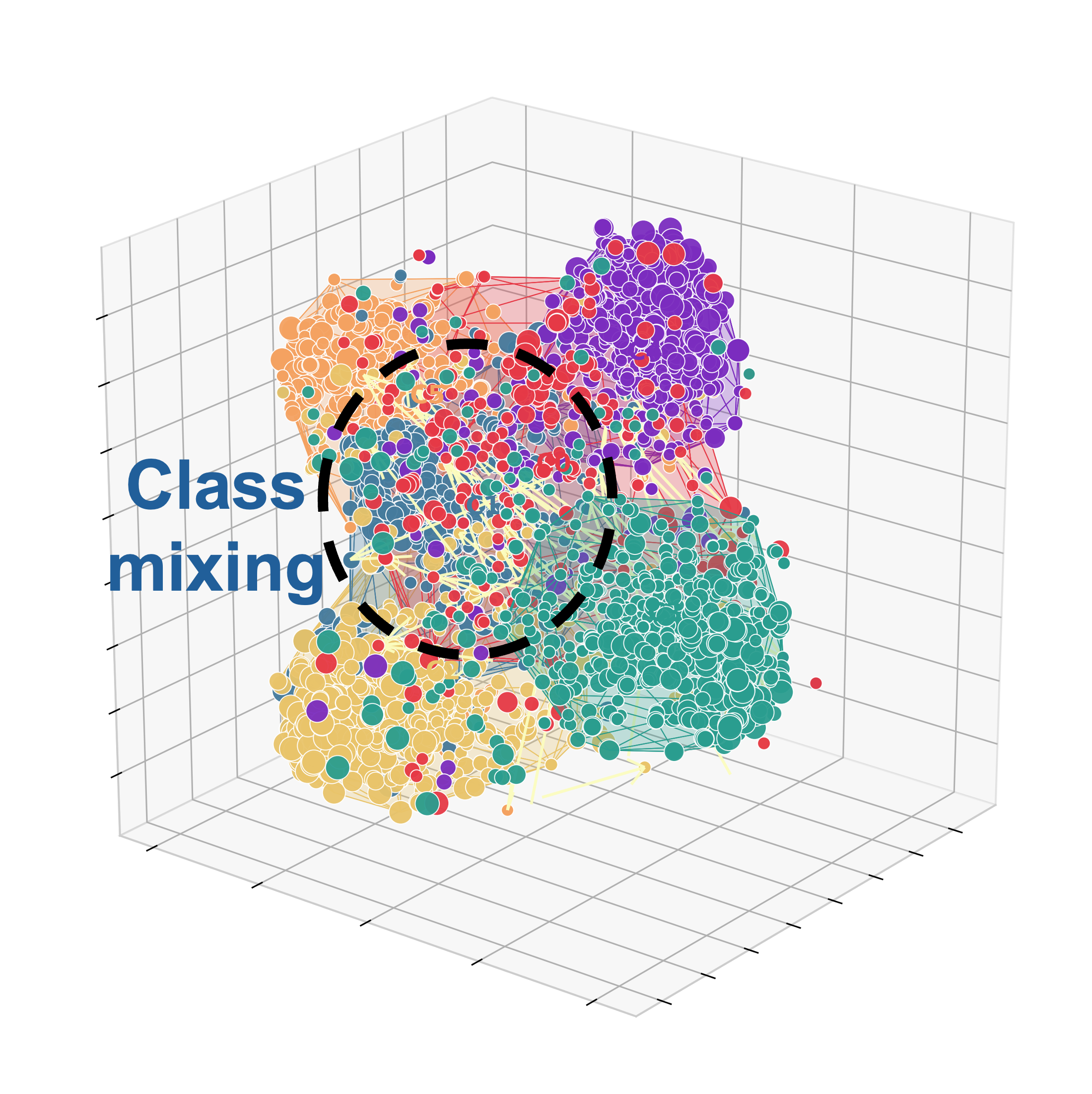}
        \caption{Correction Off}
    \end{subfigure}
    \hfill
    \begin{subfigure}{0.24\linewidth}
        \centering
        \includegraphics[width=\linewidth]{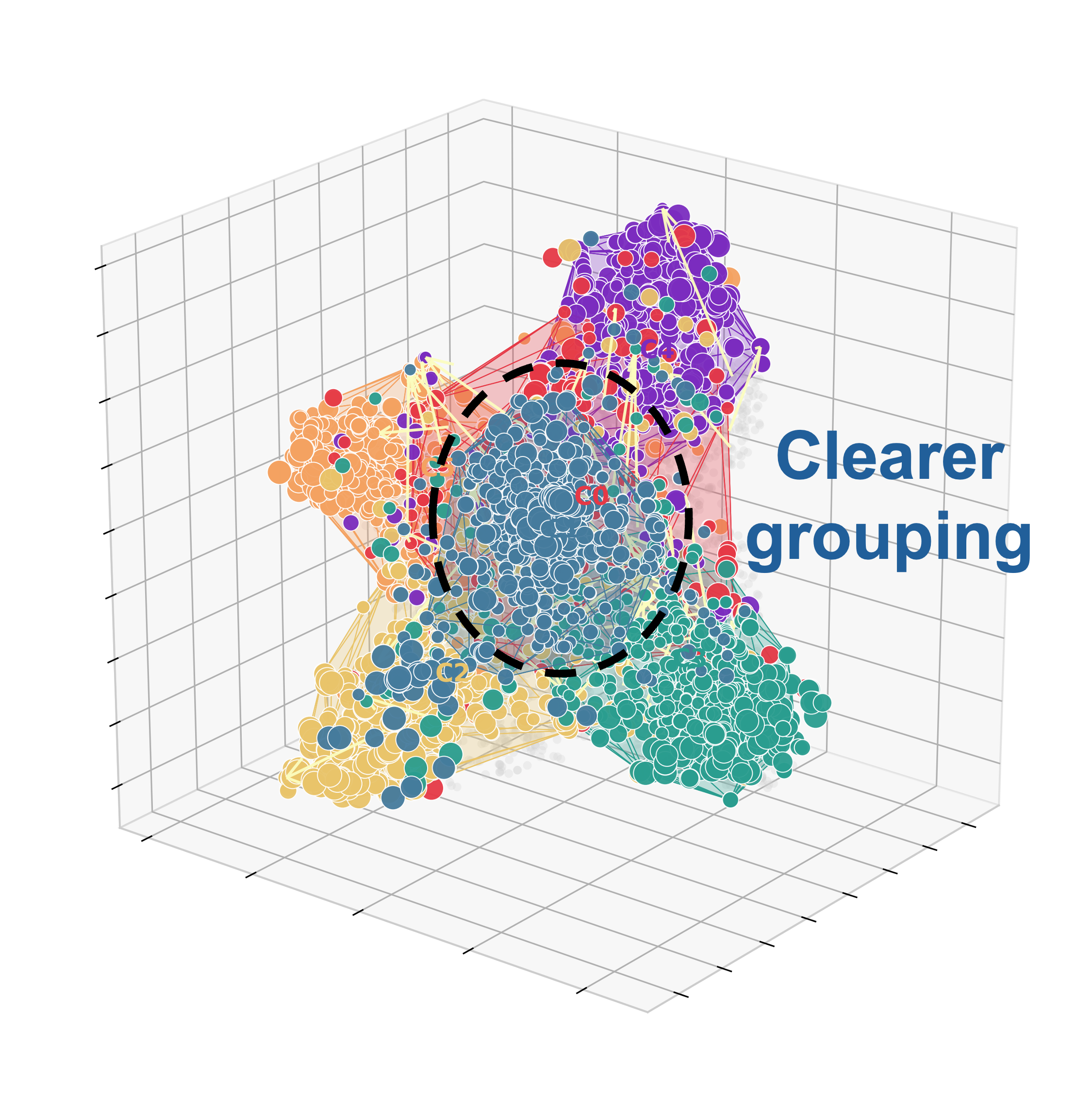}
        \caption{Correction On}
    \end{subfigure}
    
    \vspace{-0.1cm}
    \caption{(a), (b) show t-SNE visualizations of node representations learned by ARGNN and \method{}. (c), (d) compare representations with the learned geometric correction disabled and enabled.}
    \label{fig:tsne_citeseer}
\end{figure}

Figure~\ref{fig:tsne_citeseer}(a), (b) visualize node representations from ARGNN and \method{} on CiteSeer using 2D t-SNE. The highlighted class cluster is elongated under ARGNN and more compact under \method{}, as indicated by the dashed ellipses. To examine the effect of geometric correction, Fig.~\ref{fig:tsne_citeseer}(c), (d) compare representations from the same checkpoint with the correction disabled ($\gamma=0$) and enabled ($\gamma=0.1$). Both sets are visualized in a shared 3D t-SNE projection, with marker size encoding $\log\det(L_i)$. Enabling the correction yields clearer class grouping in the highlighted region, providing qualitative support for its role in shaping task-relevant representations.

\begin{algorithm}[t]
\caption{Training and inference of \method{}}
\label{alg:overall}
\small
\begin{algorithmic}[1]
\REQUIRE Training graph(s) with task-specific supervision,
test inputs, hidden dimension $d=Bm$,
number of prototypes $K$, recurrent steps $L$,
temperature $\tau$, update weights $(\lambda,\gamma)$,
spectral regularization $\epsilon_s$,
and admissible coordinate domain $\mathcal Z$.
\ENSURE Test predictions $\{p_i\}$, $\{p_G\}$, or $\{p_{uv}\}$
for node classification, graph classification,
or link prediction, respectively.

\STATE \textbf{Stage 1: Structural Encoding and Model Initialization}
\STATE Construct fixed structural signatures $U$ from
normalized degree and random-walk return probabilities
for each input graph.
\STATE Initialize the prototype atlas, feature projection,
structural gate, message transformation, feature gate,
controller $\Gamma_\theta$, and task-specific readout.
\STATE Share all propagation and controller parameters across the $L$ recurrent steps.

\STATE \textbf{Stage 2: End-to-End Training}
\WHILE{not converged}
    \STATE Select training input(s) and supervision
    for the current task.
    \FOR{each input graph $G=(V,E,X)$}
        \STATE Initialize $\xi_i^0$, $\alpha_i$, and
        $z_{i,b}^0$ for all nodes $i$ and blocks $b$
        (Eq.~\eqref{eq:prototype_initialization}).
        \label{alg:geof_state_init}
        \STATE Reconstruct $L_{i,b}^0$ and assemble $L_i^0$
        (Eqs.~\eqref{eq:block_geometry}
        and~\eqref{eq:node_geometry}).

        \FOR{$l=0,\ldots,L-1$}
            \STATE \textit{i. Triangular Frame Transport}
            \STATE Compute $h_i^l$, $e_{ij}^l$, and
            $\omega_{ij}^l$ for
            $j\in\widetilde{\mathcal N}(i)$
            (Eq.~\eqref{eq:weights}).
            \STATE Transform features with
            $\phi(\xi_i^l)=W_\phi\xi_i^l$
            and obtain aligned messages
            $m_{j\rightarrow i,b}^l$ by triangular solves
            (Eq.~\eqref{eq:cholesky_frame_transport}).
            \STATE Concatenate blockwise messages and
            compute the aligned aggregate $\bar{\xi}_i^l$
            (Eq.~\eqref{eq:aligned_aggregation}).
            \STATE Compute the feature gate $r_i^l$
            and update $\xi_i^{l+1}$
            (Eq.~\eqref{eq:feature_gate}).

            \STATE \textit{ii. Second-Order Residual Feedback}
            \STATE Compute the residuals
            $\delta_{ij,b}^l
            =m_{j\rightarrow i,b}^l
            -(\phi(\xi_i^l))_b$
            and accumulate the regularized
            second-moment matrix $C_{i,b}^l$
            with the same weights $\omega_{ij}^l$
            (Eq.~\eqref{eq:second_moment}).
            \STATE Compute
            $R_{i,b}^l=\operatorname{SChol}(C_{i,b}^l)$
            and construct the log-triangular target
            $\widehat z_{i,b}^l$
            (Eq.~\eqref{eq:coordinate_target}).

            \STATE \textit{iii. Task-Guided Geometry Evolution}
            \STATE Form
            $q_{i,b}^l
            =[(\xi_i^l)_b\|(\bar{\xi}_i^l)_b\|u_i]$,
            obtain $U_{i,b}^l$, $V_{i,b}^l$,
            and $d_{i,b}^l$ from $\Gamma_\theta$,
            and construct the geometric correction
            $\Delta z_{i,b}^l$
            (Eq.~\eqref{eq:correction}).
            \STATE Update
            $z_{i,b}^{l+1}
            =\Pi_{\mathcal Z}\!\left(
            (1-\lambda)z_{i,b}^l
            +\lambda\widehat z_{i,b}^l
            +\gamma\Delta z_{i,b}^l\right)$
            (Eq.~\eqref{eq:task_guided_geometry_update}).
            \STATE Reconstruct $L_{i,b}^{l+1}$
            and assemble $L_i^{l+1}$
            (Eqs.~\eqref{eq:block_geometry}
            and~\eqref{eq:node_geometry}).
        \ENDFOR

        \STATE Obtain the final environment representations
        $h_i^L=L_i^L\xi_i^L$.
        \STATE Compute $p_i$, $p_G$, or $p_{uv}$
        using the corresponding task-specific readout
        (Eqs.~\eqref{eq:node_objective},
        \eqref{eq:graph_objective},
        and~\eqref{eq:link_objective}).
        \label{alg:geof_prediction}
    \ENDFOR
    \STATE Evaluate $\mathcal L_{\mathrm{node}}$,
    $\mathcal L_{\mathrm{graph}}$, or
    $\mathcal L_{\mathrm{link}}$ for the selected task
    using its training supervision.
    \STATE Backpropagate through the $L$ recurrent steps
    and jointly update all trainable parameters.
\ENDWHILE

\STATE \textbf{Stage 3: Inference}
\STATE Fix the trained model parameters.
\FOR{each test input graph $G$}
    \STATE Apply the initialization, recurrent evolution,
    and readout in
    lines~\ref{alg:geof_state_init}--\ref{alg:geof_prediction}
    using the graph's structural signatures $U$.
    \STATE Retain $p_i$ for test nodes, $p_G$ for test graphs,
or $p_{uv}$ for test candidate pairs, according to the task.
\ENDFOR
\RETURN $\{p_i\}$, $\{p_G\}$, or $\{p_{uv}\}$
for the corresponding test inputs.
\end{algorithmic}
\end{algorithm}

\end{document}